%% file: main.tex
\documentclass{article}

\PassOptionsToPackage{numbers, compress}{natbib}
\usepackage[preprint]{neurips_2026}

\usepackage[utf8]{inputenc}
\usepackage[T1]{fontenc}
\usepackage{hyperref}
\usepackage{url}
\usepackage{booktabs}
\usepackage{amsfonts}
\usepackage{amsmath}
\usepackage{amssymb}
\usepackage{amsthm}
\usepackage{mathtools}
\usepackage{nicefrac}
\usepackage{microtype}
\usepackage{xcolor}
\usepackage{graphicx}
\usepackage{subcaption}
\usepackage{dsfont}  
\usepackage{multirow}
\usepackage{enumitem}
\usepackage{xspace}
\usepackage[capitalize,noabbrev]{cleveref}
\usepackage[textsize=tiny,disable]{todonotes}
\usepackage{etoc}
\usepackage{placeins}

\graphicspath{{images/}}

\theoremstyle{plain}
\newtheorem{theorem}{Theorem}[section]

\theoremstyle{definition}
\newtheorem{definition}[theorem]{Definition}

\theoremstyle{remark}
\newtheorem{remark}[theorem]{Remark}

\newcommand{\PP}{\mathbb{P}}
\newcommand{\EE}{\mathbb{E}}
\newcommand{\xb}{\mathbf{x}}
\newcommand{\Pb}{\mathbf{P}}
\newcommand{\Fb}{\mathbf{F}}
\newcommand{\Rb}{\mathbf{R}}
\newcommand{\Vb}{\mathbf{V}}
\newcommand{\Ub}{\mathbf{U}}
\newcommand{\Sb}{\mathbf{S}}
\newcommand{\ub}{\mathbf{u}}
\newcommand{\gb}{\mathbf{g}}
\newcommand{\tb}{\mathbf{t}}
\newcommand{\Deltab}{{\bf \Delta}}
\newcommand{\im}{\mathrm{i}}
\newcommand{\prule}{(P_n)_{n\ge0}}
\newcommand{\nest}{\texttt{n\_estimators\xspace}}
\newcommand{\tabreg}{\texttt{TabPFNRegressor\xspace}}
\newcommand{\tabclf}{\texttt{TabPFNClassifier\xspace}}

\title{Uncertainty Decomposition for Bayes-Filtered Transformers via Bayesian Predictive Inference}

\author{%
  Sandra Fortini\thanks{Equal contribution.} \\
  Department of Decision Sciences\\
  Bocconi University, Milan, Italy\\
  \And
  Kenyon Ng\footnotemark[1] \\
  Department of Econometrics and Business Statistics\\
  Monash University, Australia\\
  \And
  Sonia Petrone \\
  Department of Decision Sciences\\
  Bocconi University, Milan, Italy\\
  \And
  Judith Rousseau \\
  CEREMADE, Universit\'e Paris Dauphine--PSL\\
  Paris, France and University of Oxford, UK\\
  \And
  Susan Wei\thanks{Corresponding author.} \\
  Department of Econometrics and Business Statistics\\
  Monash University, Australia\\
  \texttt{susan.wei@monash.edu}\\
}

\begin{document}

\maketitle

\etocdepthtag.toc{body}
\input{body}

\begin{ack}
We thank Bin Yu, Jakob Heiss, Yan Shuo Tan, Ryan Giordano, Yingzhen Li, Isuru Shavindra Jayasekera, and Dino Sejdinovic for helpful discussions.
SF was supported by the European Union -- Next Generation EU Funds, PRIN 2022 (2022CLTYP4). KN was supported by the Australian Government Research
Training Program and the Statistical Society of Australia PhD Top-up
Scholarship.
\end{ack}

\bibliography{references}
\bibliographystyle{plainnat}

\newpage
\appendix
\etocdepthtag.toc{appendix}
\etocsettagdepth{body}{none}
\etocsettagdepth{appendix}{subsection}
\setcounter{tocdepth}{2}
\renewcommand{\contentsname}{Appendix contents}
\tableofcontents
\input{appendix}

\end{document}

%% file: body.tex
\begin{abstract}
  Bayes-filtered transformers are transformers meta-learned on sequences
  from a prior predictive distribution to approximate the corresponding
  posterior predictive distribution. They output total predictive uncertainty in
  a single forward pass but never explicitly represent a posterior distribution,
  making the standard route to separating aleatoric from epistemic uncertainty
  unavailable. We address this challenge through the lens of Bayesian predictive
  inference (BPI). Our main result is a predictive Central Limit Theorem (CLT)
  for supervised settings under conditions that are among the weakest known in
  the BPI literature. The CLT characterises the posterior of the limiting
  predictive distribution given an observed context as asymptotically Gaussian;
  the variance of this Gaussian quantifies epistemic uncertainty. We apply the
  framework to TabPFN, a Bayes-filtered transformer that is a
  state-of-the-art foundation model for tabular prediction. The resulting credible bands achieve near-nominal frequentist coverage as context length grows,
  and the decomposition largely matches standard desiderata: epistemic uncertainty
  shrinks with context length and is highest in sparsely observed regions within the span of the context data, while
  aleatoric uncertainty dominates near decision boundaries where classes
  overlap.
\end{abstract}


\section{Introduction}\label{sec:intro}
A \emph{Bayes-filtered transformer} (BFT), with weights $\phi$, is a transformer meta-learned on sequences
drawn from a prior predictive distribution. Through this meta-learning, BFTs learn to approximate the
corresponding Bayesian posterior predictive distribution (PPD). In the idealised limit the approximation is exact: a BFT realises the PPD and is therefore Bayesian by construction.
The theoretical basis for this is well established: minimising autoregressive log-loss on sequences from the prior predictive distribution
recovers the PPD \citep{ortega19metalearning}.
BFTs serve as a popular sandbox for studying in-context learning
\citep{garg22what,raventos23pretraining,zhang24trained,wu24pretraining,wurgaft25strategies,park25competition}. BFTs also underlie
recent advances in tabular prediction, most notably TabPFN
\citep{hollmann23tabpfn,hollmann25accurate}.

BFTs turn Bayesian inference into a simple regression problem: a single forward pass returns the
PPD, sidestepping the intractable posterior distribution that conventional Bayesian
workflows have to reckon with. Prior-data fitted network (PFN) documentation routinely advertises this convenience, claiming that uncertainty
estimates are readily available from the predictive output
\citep{muller22pfns,hollmann23tabpfn,hollmann25accurate}. The claim holds only for \emph{total} predictive
uncertainty. Yet practitioners typically want to separate that total into its \textbf{aleatoric} and
\textbf{epistemic} components, and the standard route to that decomposition runs through the posterior distribution. BFTs internalise this posterior implicitly during meta-learning but never expose
it, leaving the decomposition seemingly out of reach. Bayesian predictive inference \citep[BPI;][]{fortini25exchangeability} offers a route to this decomposition: latent posterior quantities are inferred from the BFT's observable sequence of predictive updates. This recovers the implicit posterior (and with it the aleatoric/epistemic split) without touching the BFT's weights, the meta-learning data, or the prior predictive distribution used during meta-learning. In supervised settings, the predictive CLT we develop delivers this recovery.

To make the problem and our solution concrete,
consider the \textbf{Beta-Bernoulli} model:
$\tilde\theta\sim\mathrm{Beta}(\alpha,\beta)$, then
$Y_1,Y_2,\ldots\mid\tilde\theta$ are i.i.d. $\mathrm{Bernoulli}(\tilde\theta)$.
This model admits two distinct sources of uncertainty.
\textbf{Aleatoric uncertainty} is the irreducible randomness of flipping a coin.
Even knowing $\tilde\theta$ exactly, the outcomes are still random and the randomness
does not vanish with more data. \textbf{Epistemic uncertainty} is the
uncertainty about $\tilde\theta$ that arises from having observed only finitely many flips. After observing $n$ flips, the posterior
$\tilde\theta\mid y_{1:n} \sim \mathrm{Beta}\bigl(\alpha+\sum_{i=1}^n y_i,\; \beta+n-\sum_{i=1}^n y_i\bigr)$
has variance $\mathrm{Var}(\tilde\theta\mid y_{1:n}) = O(1/n)$, vanishing as $n\to\infty$.

Now consider a BFT meta-learned on the Beta-Bernoulli prior predictive defined above. It
learns its own predictive rule
$g_k := \PP_\phi(Y_{k+1}=1\mid y_{1:k})$, the probability the BFT assigns to $Y_{k+1}=1$ given context $y_{1:k}$ in a single forward pass. In the idealised limit
$g_k$ coincides with the analytic PPD
$g_k^{\mathrm{oracle}} := (\alpha+\sum_{i=1}^k y_i)/(\alpha+\beta+k)$, so the
BFT realises Bayes' rule exactly. Our predictive CLT
(Theorem~\ref{th:ascondmult}) recovers the posterior of $\tilde\theta$ from the
\emph{volatility} of the predictive updates $\Delta_k := g_k - g_{k-1}$:
$\tilde\theta\mid y_{1:n} \approx \mathcal N(g_n,\,V_n/n)$ with
$V_n = \tfrac{1}{n}\sum_{k=1}^{n} k^2\Delta_k^2$; the weight $k^2$ corresponds to setting the rate parameter $\gamma=1$ in the general statement (Section~\ref{sec:PCLT}). When uncertainty is quantified by variance, $V_n/n$
\emph{is} the epistemic component.

A real trained BFT, of course, only approximates $g_k^{\mathrm{oracle}}$.
Our predictive CLT accommodates this: under our \emph{signed conditions}
(Section~\ref{sec:PCLT}), pathwise regularity conditions on the predictive
updates and on signed sums of their conditional drifts, the Gaussian approximation $\tilde\theta\mid y_{1:n}\approx\mathcal N(g_n,\,V_n/n)$
continues to hold for any approximate $g_k$ that satisfies them.
Appendix~\ref{app:theorycheck} provides empirical diagnostics for these conditions on a Beta-Bernoulli BFT (architecture and training in
Appendix~\ref{app:theorycheck-setup}; figure details in Appendix~\ref{app:headline-figure}). Figure~\ref{fig:beta_bernoulli_intro} shows the resulting fit on this BFT:
panel~(a) shows the Gaussian from our predictive CLT
closely matching the analytic Beta posterior, and panel~(b) shows
$V_n$ accumulating from the predictive updates.

\begin{figure}[tbp]
  \centering
  \includegraphics[width=0.8\linewidth]{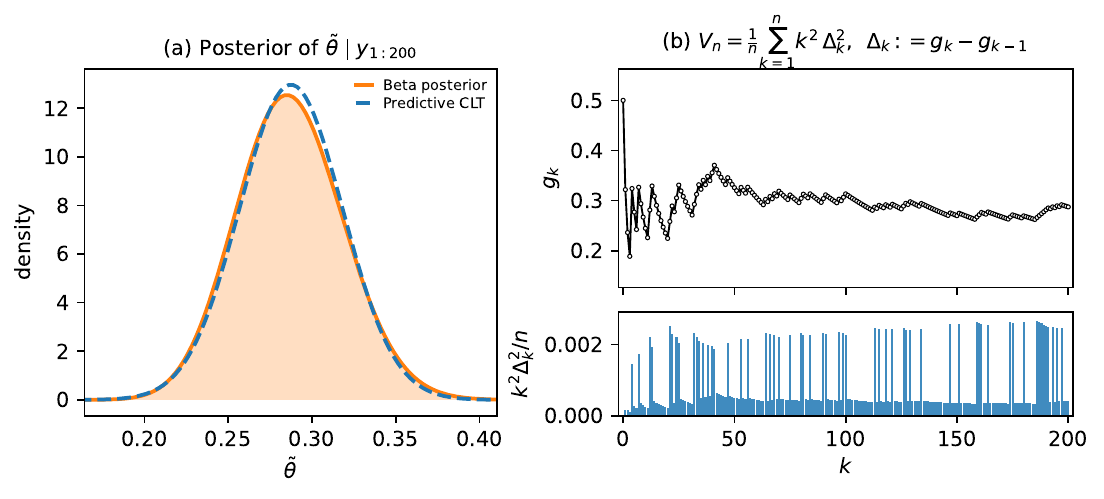}
  \caption{Beta-Bernoulli BFT meta-learned with $\tilde\theta\sim\mathrm{Beta}(1,1)$ and i.i.d.\ $\mathrm{Bernoulli}(\tilde\theta)$ sequences.
    \textbf{(a)}~On a sequence of $n=200$ observations: the analytic Beta posterior $\tilde\theta\mid y_{1:n}$ (orange) and the Gaussian approximation
    $\mathcal N(g_n, V_n/n)$ from our predictive CLT (blue dashed).
    \textbf{(b)}~Top: the predictive trajectory $(g_k)_{k\ge 0}$, where $g_k := \PP_\phi(Y_{k+1}=1\mid y_{1:k})$. Bottom: the per-step contributions
    $k^{2}\Delta_k^{2}/n$ to $V_n$, as a stem plot.
    Full experimental details are in Appendix~\ref{app:headline-figure}.}
  \label{fig:beta_bernoulli_intro}
\end{figure}

\textbf{The supervised setting.} Given a context $z_{1:n}$ with $z_i := (x_i,y_i)$ and a test covariate $x^\ast$, a BFT meta-learned in a supervised setting outputs the conditional predictive distribution $\PP_\phi(Y^\ast\in\cdot\mid x^\ast, z_{1:n})$. The role of the Beta-Bernoulli parameter $\tilde\theta$ is now played by a random predictive distribution $\tilde P$; in the Beta-Bernoulli example, $\tilde P$ is the $\mathrm{Bernoulli}(\tilde\theta)$ law, so $\tilde P(\{1\})=\tilde\theta$. By the law of total variance,
\begin{equation*}
  \underbrace{\mathrm{Var}(Y^\ast\mid x^\ast, z_{1:n})}_{\text{total}}
  \;=\;
  \underbrace{\EE[\mathrm{Var}(Y^\ast\mid x^\ast,\tilde P)\mid z_{1:n}]}_{\text{aleatoric}}
  \;+\;
  \underbrace{\mathrm{Var}(\EE[Y^\ast\mid x^\ast,\tilde P]\mid z_{1:n})}_{\text{epistemic}}.
\end{equation*}
The left-hand side, equal to the variance of the BFT's output, is the total predictive uncertainty. Our predictive CLT supplies an asymptotic approximation to the epistemic variance; the aleatoric term then follows by subtraction.

\textbf{Why uncertainty decomposition (UD) matters.} BFTs are already deployed in downstream pipelines such as Bayesian optimisation \citep{muller23pfns4bo,yu26gitbo}, which relies on the \emph{epistemic} component of uncertainty. Acquisition functions like expected improvement target posterior uncertainty; irreducible noise carries no signal about where to query next, so conflating the two causes these pipelines to over-explore high-noise regions. Existing TabPFN-BO methods \citep{muller23pfns4bo,yu26gitbo} feed the BFT's \emph{total} predictive variance into the acquisition function; our decomposition supplies the principled epistemic component.

More broadly, the proposed framework offers Bayesian practitioners an automatic
alternative to standard Bayesian computation, including MCMC and variational inference. Figure~\ref{fig:lfp_intro} shows $95\%$ credible bands
for a real-data binary classification task (labour-force participation as a
function of family income; $n=753$), obtained by applying TabPFN
together with our predictive CLT. The resulting credible bands are qualitatively similar to those from a standard
textbook Bayesian logistic regression analysis \citep{albert20probability},
which requires careful prior specification and
MCMC. The route here requires neither. A second example in
Appendix~\ref{app:real_data_fibre} applies the same approach to a textbook
Bayesian reliability analysis on silicon-carbide fibre strengths
\citep{hamada08bayesian}.

\begin{figure}[tbp]
  \centering
  \includegraphics[width=\linewidth]{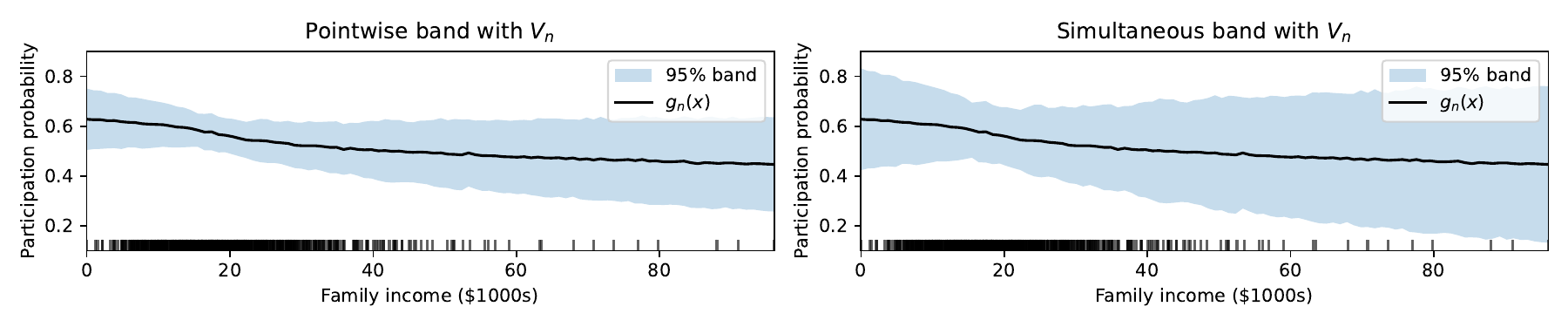}
  \caption{PSID labour-force participation ($n=753$): TabPFN's predicted
    probability $g_n(x^\ast)$ versus family income (\$1000s), with 95\% credible
    bands for the posterior of $\tilde P(x^\ast,\{1\})$ given $z_{1:n}$
    derived from the predictive CLT (Theorem~\ref{th:ascondmult}). This
    posterior represents epistemic uncertainty. \emph{Left}: pointwise credible
    band, $g_n(x^\ast)\pm 1.96\sqrt{V_n(x^\ast)/n}$. \emph{Right}: simultaneous
    credible band.}
  \label{fig:lfp_intro}
\end{figure}

\textbf{Contributions.}
\begin{itemize}[nosep,leftmargin=*]
  \item \textbf{A predictive CLT for supervised settings} under some of the
        weakest known conditions in the BPI literature. BPI's two tools for
        recovering the latent posterior, predictive Monte Carlo
        \citep{fortini20quasibayes,fong23martingale} and predictive CLTs
        \citep{fortini23predictionbased,fong26asymptotics,fortini26principled},
        both leave gaps in supervised settings: predictive Monte Carlo requires
        simulating future covariates that BFTs do not model, and existing
        predictive CLTs target unsupervised settings under stronger
        sufficient conditions (e.g., (quasi-)martingale predictive rules).
        Our predictive CLT (Theorem~\ref{th:ascondmult}) closes both gaps: it applies in supervised settings under our \emph{signed conditions} (Section~\ref{sec:PCLT}), which accommodate random-sign cancellation in the predictive drift and are strictly weaker than the (quasi-)martingale conditions of existing predictive CLTs.

  \item \textbf{First meta-learned BFT consistent with BPI sufficient conditions.} Prior empirical work has shown TabPFN to violate the martingale \citep{nagler25uncertainty} and almost conditionally identically distributed (a.c.i.d.)\ \citep{ng25tabmgp} conditions, while diagnostics of the quasi-martingale conditions on TabPFN were inconclusive \citep{fortini26principled}. We show, through empirical diagnostics, that a meta-learned Beta-Bernoulli BFT is consistent with our weaker sufficient conditions, the first positive empirical evidence for any BPI sufficient conditions on a real meta-learned model.

  \item \textbf{UD for BFTs.} The decomposition separates
        total predictive uncertainty into aleatoric and epistemic components
        (via variance or entropy), using only the observable sequence of
        predictive updates, with no access to model internals, meta-learning data, or
        prior --- the first such UD for BFTs, and the first without predictive Monte Carlo simulation. More broadly, the framework applies to any  sequence model trained to be approximately
        Bayesian.

  \item \textbf{`Automatic' Bayesian inference.} Our predictive CLT operates on the
        predictive rule directly
        without requiring a likelihood, prior elicitation, or sampling.
\end{itemize}

\textbf{Outline.}
Section~\ref{sec:background} reviews BPI and BFTs.
Section~\ref{sec:related} discusses related work.
Section~\ref{sec:methodology} presents the predictive CLT informally
and instantiates it for BFTs.
Section~\ref{sec:PCLT} states the formal theorems.
Section~\ref{sec:experiments} reports three experiments: empirical diagnostics of our sufficient conditions on a Beta-Bernoulli BFT, frequentist coverage of TabPFN's credible bands, and entropic UD on TabPFN classification.

\section{Background}
\label{sec:background}

\textbf{Bayes-filtered transformers.} \citet{ortega19metalearning} showed that
memory-based architectures trained to minimise cumulative log-loss on sequences
from a prior predictive distribution learn to implement the posterior
predictive distribution. Subsequent work confirmed this for increasingly
rich prior predictive structures, including non-exchangeable sequences such
as Markov chains and non-stationary processes
\citep{mikulik20metatrained,genewein23memorybased,grau-moya24learning}. This
broad class of meta-learned transformers is what we call a BFT. A PFN \citep{muller22pfns} is the
\emph{exchangeable} special case of a BFT. TabPFN \citep{hollmann23tabpfn,hollmann25accurate} is a prominent foundation-model PFN for tabular data.

\textbf{Meta-learning a PFN.} The meta-learning corpus consists of a dataset-of-datasets
$\{ z_{1:T}^m \}_{m=1}^M$, where each meta-learning sequence $ z_{1:T}^m = (z_1^m,\ldots,z_T^m)$ has length $T$ and contains
pairs $z_i^m=(x_i^m,y_i^m)$. Each dataset is obtained by first sampling from the
prior $p(\theta)$, then drawing samples i.i.d.\ from the $\theta$-conditional
distribution $p_\theta(x,y)$. In TabPFN, the distributions $p_\theta(x,y)$ are a rich mixture of
synthetic data-generating processes, primarily structural causal models (SCMs) and Bayesian neural networks (BNNs).

During meta-learning, each dataset $ z_{1:T}$ is presented to the transformer as
a sequence. At position $k$, the prefix $z_{1:k-1}$ plays the role of
\emph{context} and the PFN is trained in a next-token fashion to predict the
label $y_k$ at covariate $x_k$. Under the generative specification, the
conditional distribution of $y_k$ given $x_k$ and $z_{1:k-1}$ coincides with the
\textbf{Bayesian PPD}:
$p(y_k \mid x_k, z_{1:k-1}) = \int p_\theta(y_k \mid x_k)\,p(\theta \mid z_{1:k-1})\,\,d\theta.$
Following \citet{ortega19metalearning}, we call a sequence \textbf{Bayes-filtered}
when it has this property: at every position, the conditional distribution of
the next observation given the prefix is the Bayesian posterior predictive,
i.e., the result of applying Bayes' rule to all observations so far. Minimising log-loss on these
Bayes-filtered samples trains the PFN to approximate, \emph{in context}, the
PPD.

The PFN learns to approximate this PPD by minimising
$\tfrac1M \sum_{m=1}^M \sum_{i=1}^T -\log \PP_\phi(y_i^m \mid x_i^m, z_{1:i-1}^m)$.
In the idealised limit, where the transformer is sufficiently expressive, $M$
is large enough for the empirical objective to approach the population risk, and optimisation reaches the global minimum, the trained PFN realises the Bayesian PPD.

\textbf{BPI.}
We next review the BPI framework \citep{fortini25exchangeability}, which underpins our predictive CLT and UD. BPI's central object is the sequence of one-step-ahead predictive distributions, analogous to the next-token distributions a transformer defines through in-context learning. 

Let $Z_1,Z_2,\ldots\in\mathcal Z$ be a stochastic sequence governed by a
probability law $\PP$. Lower-case $z_{1:n}$ will always denote a concrete realisation of the first $n$ observations. Consider the sequence of one-step-ahead predictive distributions
\begin{equation*}
  P_n(\cdot):=\PP\bigl[Z_{n+1}\in\cdot\mid Z_{1:n}\bigr], \quad n=0,1,2,\ldots,
\end{equation*}
with $P_0(\cdot)=\PP[Z_1\in\cdot\;]$. We refer to the full sequence $\prule$ as
a \textbf{predictive rule}, which determines the entire path law
$\PP$ via the \emph{Ionescu--Tulcea extension theorem} \citep[Theorem 5.17
and Corollary 5.18]{kallenberg21foundations}. Conversely, $\PP$ determines the
predictive rule by conditioning. Thus we may speak interchangeably of the
predictive rule $\prule$ and the joint law $\PP$ it induces.

If $(Z_n)_{n \ge 0}$ is infinitely \textbf{exchangeable}, \emph{de Finetti's
  representation theorem} guarantees the existence of a random probability
measure $\tilde P$ on $\mathcal Z$ such that, conditionally on $\tilde P = P$,
the observations $Z_1, Z_2, \ldots$ are i.i.d.\ from $P$. Equivalently,
$\tilde P$ is the almost-sure limit of the predictive rule $(P_n)_{n \ge 0}$
\citep{fortini23predictionbased}. Such a $\tilde P$ is called the
\emph{directing random measure}: it plays the role of the statistical model, and
its probability law $\pi$ is the prior distribution.

Since $\tilde P$ is simultaneously a probability measure and a random object,
its notation warrants a word. Let $\Omega = \mathcal Z^\infty$ be the space of
infinite sample paths, with $\omega = (z_1, z_2, \ldots) \in \Omega$ one such
path. Then
\begin{enumerate}[label=(\alph*),itemsep=0pt,topsep=2pt,parsep=0pt]
    \item $\tilde P: \Omega \to \mathcal P(\mathcal Z)$ maps each path $\omega$ to a probability measure $\tilde P(\cdot)(\omega)$ on $\mathcal Z$;
    \item $\tilde P(A): \Omega \to [0,1]$ is, for a fixed measurable $A \subseteq \mathcal Z$, a real-valued random variable.
\end{enumerate}

In the Beta-Bernoulli setting of Section~\ref{sec:intro},
$\mathcal Z = \mathcal Y = \{0,1\}$ and the directing random measure is 
$\tilde P(\cdot)(\omega) = \mathrm{Bernoulli}(\tilde\theta(\omega))$, with
$\tilde\theta(\omega) = \lim_{n\to\infty} \tfrac{1}{n}\sum_{i=1}^n y_i = \lim_{n\to\infty} P_n(\{1\})$.
So $\tilde P(\{1\})$ and $\tilde\theta$ coincide as random variables. The conditional law of $Y_{n+1}\mid \tilde\theta$ carries the aleatoric uncertainty, and the posterior of $\tilde\theta$ given $y_{1:n}$ carries
the epistemic uncertainty. Our goal is to derive that posterior, which
enables the full aleatoric/epistemic decomposition.


BPI extends beyond strict exchangeability. We say $(Z_n)$ is
\textit{asymptotically exchangeable} with limiting directing random measure
$\tilde P$ if for $n\to\infty$, $(Z_{n+1}, Z_{n+2},\ldots) \to (Z'_{1},
Z'_{2},\ldots)$ in distribution, where $(Z'_n)$ is exchangeable with directing
random measure $\tilde P$. The relaxation matters in practice: a trained BFT
typically produces sequences that are only approximately exchangeable. 
A sufficient condition for asymptotic exchangeability is that the predictive rule is a \textbf{martingale},
$\EE [ P_{n+1}(A) \mid Z_{1:n}] = P_n(A)$ $\PP$-a.s.\ for every measurable $A\subseteq\mathcal Z$ and $n\ge0$. Under exchangeability $(P_n)$ is itself a martingale, so a trained BFT's predictive rule should ideally be one, or approximately so (e.g., a quasi-martingale \citep{fortini26principled}). Section~\ref{sec:PCLT} states the conditions we use, which are weaker still and, together with a condition on the covariate process, guarantee that future labels are, in the limit, conditionally i.i.d.\ given their covariates (Theorem~\ref{th:joint}).

\section{Related Work}
\label{sec:related}

The BPI literature is organised around conditions under which the predictive distributions $P_n$ converge to a directing random measure $\tilde P$. These convergence conditions form a hierarchy of decreasing strength: $\text{exchangeable} \subset \text{c.i.d./martingale} \subset \text{a.c.i.d.} \subset \text{quasi-martingale} \subset \text{signed (ours)}$, where c.i.d.\ denotes conditionally identically distributed.\footnote{As a companion result, Appendix~\ref{app:qm-general-gamma} states a general-$\gamma$ quasi-martingale CLT extending the $\sqrt n$-rate ($\gamma=1$) result of \citet{fortini26principled} to arbitrary $\gamma\in(0,1]$.} Classical BPI assumes exchangeability \citep{definetti37prevision} or martingale predictive rules \citep{berti04limit,fortini20quasibayes}; recent works relax this to a.c.i.d.~\citep{battiston25bayesian} and quasi-martingale conditions \citep{fortini26principled}. Empirically, TabPFN violates the martingale \citep{nagler25uncertainty} and a.c.i.d.\ \citep{ng25tabmgp} conditions, with quasi-martingale diagnostics inconclusive \citep{fortini26principled}.

A predictive CLT requires more than convergence to $\tilde P$, namely a rate of convergence and regularity on the conditional drift. Existing predictive CLTs \citep{fortini20quasibayes,fortini23predictionbased,fong25bayesian,fong26asymptotics} assume martingale or quasi-martingale predictive rules and target unsupervised settings; our predictive CLT extends to supervised settings under our \emph{signed conditions} (Section~\ref{sec:PCLT}), among the weakest known in the BPI literature for which a predictive CLT is established. Predictive Monte Carlo \citep{fortini20quasibayes,fong23martingale} hits a different obstacle for BFTs in supervised settings: simulating future observations requires simulating future covariates, which BFTs do not model. \citet{nagler25uncertainty} circumvent this by replacing the predictive rule's update dynamics at a fixed query covariate with a Gaussian copula update, so the resulting decomposition characterises a hybrid rule (TabPFN initialisation + copula dynamics) with TabPFN absent from the forward simulation (Appendix~\ref{app:coverage}; Table~\ref{tab:coverage}).

Beyond the BPI literature, several uncertainty methods are commonly invoked for deep models but do not address our decomposition target. Standard methods that operate on the network's parameters $\phi$ (deep ensembles, MC dropout, Bayesian neural networks) face practical obstacles for foundation BFTs: they require retraining or architectural surgery, which defeats the point of an off-the-shelf foundation model. Even applied to a smaller BFT one trains oneself, they characterise uncertainty about the trained model itself (its weights, dropout masks, or ensemble members), a different object than $\tilde P$. Conformal prediction produces prediction sets for $Y^*$ with no aleatoric/epistemic decomposition. Variational UD has been proposed for LLMs by bounding aleatoric uncertainty via auxiliary probes \citep{jayasekera25variational}; the setting (internal probes for LLMs vs.\ tabular BFTs) and decomposition machinery differ from ours.

\section{Methodology}
\label{sec:methodology}
We present the predictive CLT (Theorem~\ref{th:ascondmult}) informally here and defer the complete statement
to Section~\ref{sec:PCLT}. Let $Z_i=(X_i,Y_i)$, where
$X_i\in\mathcal X\subseteq\mathbb R^d$ and
$Y_i\in\mathcal Y\subset\mathbb R$, both spaces measurable. Denote by $\PP$ the law of $(Z_i)$. Under asymptotic exchangeability, the pairs $(X_i,Y_i)$ for $i>n$ are, given $\tilde P$, approximately i.i.d.\ from $\tilde P$ when $n$ is large, inducing a posterior on the conditional distribution of $Y$ given $X$. We target this posterior at a finite grid of test pairs $\{(x_j, A_j)\}_{j=1}^m$, for which our CLT provides an explicit asymptotic approximation.

For any measurable $A \subseteq \mathcal{Y}$, define the predictive probability of $A$ conditional on the history $Z_{1:k}$ and a new covariate $x$:
\begin{equation*}
  P_k(x, A) := \PP (Y_{k+1} \in A \mid X_{k+1}=x, Z_{1:k}), \quad \text{for } k \ge 0,
\end{equation*}
with $P_0(x,A) := \PP(Y_1 \in A \mid X_1=x)$. Under our \textbf{signed conditions} (Section~\ref{sec:PCLT}), for any fixed $(x,A)$ the sequence $(P_k(x, A))_{k \ge 0}$ converges to a limiting random variable $\tilde{P}(x, A)$.

For $m$ covariate-event pairs $\{(x_j, A_j)\}_{j=1}^m$, collect the predictive
probabilities into the vector
$\mathbf{P}_k = \bigl(P_k(x_1,A_1),\ldots,P_k(x_m, A_m)\bigr)^\top$, with
limit $\tilde{\mathbf{P}}$. The predictive CLT (Theorem~\ref{th:ascondmult}) states that the posterior of
$\tilde{\mathbf{P}}$ is asymptotically Gaussian:
\begin{equation}\label{eq:general_clt_vector}
  \tilde{\mathbf{P}} \mid z_{1:n} \;\approx\; \mathcal{N}_m\left(\mathbf{P}_n, \frac{\mathbf{V}_n}{n^\gamma}\right),
\end{equation}
where $\gamma\in(0,1]$ is a rate parameter and $\mathbf{V}_n$ is the $m \times m$ covariance matrix
\begin{equation}\label{eq:Vn_xA}
  \mathbf{V}_n = \frac{1}{n} \sum_{k=1}^n \frac{k^{\gamma+1}}{\gamma} \Deltab_k \Deltab_k^\top,
  \qquad
  \Deltab_k:=\Pb_k - \Pb_{k-1}.
\end{equation}
The covariance $\mathbf{V}_n$ is determined by the volatility of the predictive
rule's updates along $z_{1:n}$. Computing it requires only $n$ forward passes
of the BFT; no sampling or imputation of future values is involved.
The matrix $\mathbf{V}_n/n^\gamma$ is the posterior variance of the limiting
predictive functional $\tilde{\mathbf{P}}$, the \emph{epistemic}
uncertainty, which shrinks to zero as $n\to\infty$. The \emph{aleatoric} uncertainty, the
dispersion inherent to each realisation $\tilde{\mathbf{P}}(\omega)$,
persists. The approximation~\eqref{eq:general_clt_vector} concerns the posterior of the limiting predictive functional $\tilde{\mathbf{P}}$ given the observed context; the BFT's own predictive distributions are not themselves required to be Gaussian.

The rate $\gamma$ is a property of the predictive rule alone: condition~$(iii)$ of Theorem~\ref{th:ascondmult} pins $\gamma$ at the value at which the inflated outer product of residuals $\Rb_n=n^\gamma\sum_{k>n}\Deltab_k\Deltab_k^\top$ has a positive definite limit, while conditions~$(ia)$ and~$(ib)$ of Theorem~\ref{th:ascondmult} require the conditional drift $b_n := \EE[\Delta_n \mid Z_{1:n-1}]$ to vanish at this same rate. For complicated BFTs such as TabPFN, $\gamma$ is unknown and not reliably estimable in practice. We assume $\gamma=1$ throughout, supported by direct empirical evidence on the small Beta-Bernoulli BFT (Appendix~\ref{app:theorycheck}).

We specialise to two settings, summarised in Table~\ref{tab:specialisation}: binary classification (event $A=\{1\}$) and regression (event $A_t=(-\infty,t]$).
On a grid $\xb=(x_1,\dots,x_m)\in\mathcal X^m$ (and, for regression, $\tb=(t_1,\dots,t_m)$), we stack these into vectors of length $m$:
\begin{equation}\label{eq:multdef}
  \gb_n(\xb) := \bigl(g_n(x_j)\bigr)_{j=1}^m,
  \qquad
  \Fb_n(\xb,\tb) := \bigl(F_n(x_j,t_j)\bigr)_{j=1}^m.
\end{equation}
Then~\eqref{eq:general_clt_vector} holds with $\Pb_n=\gb_n$ or $\Pb_n=\Fb_n$.
The framework places no restriction on the dimension of $\mathcal X$, so it accommodates high-dimensional covariates.
The practical details of applying the framework to TabPFN (e.g., $\Vb_n$ estimator's sensitivity to data ordering) are described in
Appendix~\ref{app:ud4tabpfn}.

\begingroup
\setlength{\intextsep}{4pt plus 1pt minus 1pt}
\begin{table}[h]
  \centering
  \caption{Specialisation of the predictive CLT to binary classification and regression.}
  \label{tab:specialisation}
  \footnotesize
  \begin{tabular}{l c c c}
    \toprule
    \textbf{Setting} & \textbf{Event $A$} & \textbf{Finite sample} & \textbf{Limit} \\
    \midrule
    \textbf{Classification} &
    $\{1\}$ &
    $g_k(x) := P_k(x, \{1\})$ &
    $\tilde{g}(x) := \tilde{P}(x, \{1\})$ \\
    \textbf{Regression} &
    $(-\infty, t]$ &
    $F_k(x, t) := P_k\bigl(x, (-\infty, t]\bigr)$ &
    $\tilde{F}(x, t) := \tilde{P}\bigl(x, (-\infty, t]\bigr)$ \\
    \bottomrule
  \end{tabular}
\end{table}
\endgroup

\section{A Predictive CLT for Supervised Settings}
\label{sec:PCLT}

We use the notation $\mathcal L(U \mid Z)$ to denote the conditional distribution of a random vector $U$ given a random vector $Z$. Convergence of probability measures is always understood as weak convergence.
All proofs are deferred to Appendix~\ref{app:proofs}.

\subsection{Convergence of the predictive distributions}

Let $\Delta_n(x,t)=F_n(x,t)-F_{n-1}(x,t)$ be the predictive update based on $Z_n$.
Note that what follows also applies to multiclass classification with atomic events for each class.

\begin{theorem}\label{th:convergence}
Assume that the following conditions hold:
\begin{description}[itemsep=2pt,topsep=2pt,parsep=0pt]
    \item[(i)] $\mathcal Y$ is compact.
    \item[(ii)] For every $t\in \mathcal Y$, the functions $x\mapsto F_n(x,t)$ are equicontinuous.
    \item[(iii)] For every $x$ and $t$, $\sum_{k\geq n} \EE(\Delta_k(x,t)\mid Z_{1:k-1})\rightarrow 0$ $\PP$-a.s.
\end{description}
Then there exists $\tilde F(x,t)$ such that $F_n(x,\cdot)$ converges weakly to $\tilde F(x,\cdot)$ for every $x\in\mathcal X$, $\PP$-a.s.
\end{theorem}
Condition (ii) of Theorem~\ref{th:convergence} is to be understood in
an almost-sure sense under $\PP$. In practice we construct credible bands on a finite evaluation grid in $x$, where (ii) reduces to continuity of $F_n(\cdot,t)$ on a finite set and is trivially satisfied.
Note that condition~(iii) of Theorem~\ref{th:convergence} is weaker
than the quasi-martingale condition
$\sum_{n\geq 1}\EE[|\EE[\Delta_n\mid Z_{1:n-1}]|]<\infty$
\citep{fortini26principled}, as it asks
only that the signed tail sums of the conditional drifts $b_n$ vanish, without requiring
absolute summability.
Theorem~\ref{th:convergence} can be further strengthened by adding a
conditional independence assumption on the covariate process, as stated in the
next theorem.

\begin{theorem}\label{th:joint}
  Under the assumptions (i)--(iii) of Theorem~\ref{th:convergence} and \textbf{(iv)} for every $n\geq 1$, $X_{n+1}$ is conditionally independent of $Y_{1:n}$ given $X_{1:n}$, it holds $\PP$-a.s. that the conditional CDF
  $\PP[Y_{n+1}\leq t_1,\dots,Y_{n+k}\leq t_k\mid X_{n+1:n+k}=x_{1:k}, Z_{1:n}]$
  converges weakly to $\prod_{i=1}^k \tilde F(x_i,t_i)$ for every
  $x_1,\dots,x_k$.
\end{theorem}

\subsection{Predictive CLT}
Theorem~\ref{th:convergence} applied at each $x_i$ gives, $\PP$-a.s.,
componentwise weak convergence $F_n(x_i,\cdot)\to\tilde F(x_i,\cdot)$ for
$i=1,\dots,m$. Since $m$ is finite, joint weak convergence follows:
$\PP$-a.s., $\Fb_n(\xb,\cdot)\to\tilde{\Fb}(\xb,\cdot)$ weakly, with
$\tilde{\Fb}(\xb,\cdot):=[\tilde F(x_1,\cdot),\dots,\tilde F(x_m,\cdot)]^\top$.
The next theorem states that the posterior of $\tilde\Fb(\xb,\tb)$ given
$Z_{1:n}$ is, under additional conditions, asymptotically Gaussian.

\begin{theorem}\label{th:ascondmult}
Let $\Deltab_n(\xb,\tb):=\Fb_n(\xb,\tb)-\Fb_{n-1}(\xb,\tb)$ be the vector of one-step predictive updates, with $\Fb_n$ as in~\eqref{eq:multdef}. Under the
assumptions of Theorem~\ref{th:convergence}, suppose $\tb$ is a continuity
point of $\EE[\tilde\Fb(\xb,\cdot)]$, and there exists
$\gamma\in(0,1]$ for which the following hold:
    \begin{description}[itemsep=2pt,topsep=2pt,parsep=0pt]
        \item[$(ia)$] For $i=1,\dots,m$, \;
       $n^{\gamma/2}\sum_{k> n}\EE(\Delta_k(x_i,t_i)\mid Z_{1:k-1})$ and
$  n^{\gamma/2} \sup_{k\geq n}|\EE(\Delta_k(x_i,t_i)\mid Z_{1:k-1})|$ converge to zero $\PP$-a.s.\ and are dominated in $L^1$;
        \item[$(ib)$] For $i=1,\dots,m$, \;
   $ n^{\gamma}\sum_{k\geq n}\EE(\Delta_k(x_i,t_i)\mid Z_{1:k-1})^2\rightarrow 0$ \; $\PP$-a.s.;
        \item[$(ii)$] For $i=1,\dots,m$, \; $ \EE\bigl[\sup_{k\geq 1}k^{\gamma/2} |\Delta_k(x_i,t_i)|\bigr]<+\infty$;
        \item[$(iii)$] the matrix of inflated outer product of residuals
        $\Rb_n(\xb,\tb):=n^\gamma\sum_{k\geq n+1} \Deltab_k(\xb,\tb)\Deltab_k(\xb,\tb)^\top$
        converges $\PP$-a.s.\ to a positive definite random matrix
        $\Vb(\xb,\tb)$ as $n\to\infty$.
    \end{description}
Then $\mathcal L\bigl(n^{\gamma/2}(\Fb_n(\xb,\tb) - \tilde \Fb(\xb,\tb)) \mid Z_{1:n}\bigr) \to \mathcal N_m(\mathbf{0},\Vb(\xb,\tb))$, $\PP$-a.s.
\end{theorem}
The requirement that $\tb$ be a continuity point of $\EE[\tilde{\Fb}(\xb,\cdot)]$ may seem purely theoretical, since $\tilde{\Fb}$ is unknown. However, for each $i=1,\dots,m$, $\EE[\tilde F(x_i,\cdot)]$ is a cumulative distribution function and is therefore continuous at almost every $t_i$. In binary classification, $\tilde F(x,\cdot)$ is constant on $(0,1)$, so any $t\in(0,1)$ is a continuity point and Theorem~\ref{th:ascondmult} at such $t$ gives a CLT for $g_n(x)=1-F_n(x,t)$.

To apply the predictive CLT we need an estimator of the limiting variance
$\Vb(\xb,\tb)$. We use two: $\Vb_n$ in~\eqref{eq:Vn_xA}, computable directly
from the observed history; and a permutation-invariant alternative $\Ub_n$
that integrates over the next observation $Z_{n+1}$. Appendix~\ref{app:credible_bands}
formalises both: Theorem~\ref{th:residualmult} gives a consistency
relationship between them (if $\Ub_n\to\Vb$ a.s., so does $\Vb_n$), and
Theorem~\ref{th:jointnormmult} uses the resulting self-normalised predictive CLTs to
construct pointwise and simultaneous asymptotic credible bands.
Section~\ref{sec:experiments} reports results with both estimators.

\textbf{Checking the conditions.}
Applying the predictive CLT to a BFT requires verifying four pathwise
conditions (rendered as five diagnostic panels in
Appendix~\ref{app:theorycheck}, which splits condition~$(ia)$ into its
tail-sum and pointwise parts) on the predictive rule's conditional drifts $b_k$ and
increments $\Delta_k$: the signed tail-sum condition~(iii) of
Theorem~\ref{th:convergence}, which gives the rule a well-defined limit;
conditions $(ia)$ and $(ib)$ of Theorem~\ref{th:ascondmult}, requiring
$b_n$ to vanish at rate~$n^{\gamma/2}$; and condition $(iii)$ of
Theorem~\ref{th:ascondmult}, which pins $\gamma$
(Section~\ref{sec:methodology}). The domination in $L^1$ in
condition $(ia)$ and the moment condition $(ii)$ of
Theorem~\ref{th:ascondmult} are assumed, as they cannot be diagnosed
pathwise (Appendix~\ref{app:theorycheck}). We refer to these
collectively as our \textbf{signed conditions}, named for the signed sums
of $b_k$ in conditions $(iii)$ of Theorem~\ref{th:convergence} and
$(ia)$ of Theorem~\ref{th:ascondmult}. The next section describes checking these conditions on Beta-Bernoulli BFTs.

\section{Experiments}
\label{sec:experiments}

This section reports signed-condition diagnostics on Beta-Bernoulli BFTs (Appendix~\ref{app:theorycheck}), frequentist coverage of credible bands on synthetic data-generating processes (DGPs; Appendix~\ref{app:coverage}), and entropic UD on classification tasks (Appendix~\ref{app:entropic_ud}). The appendix additionally reports regression and classification with a gap in the observed covariate range on seven DGPs: credible bands respond visibly to the gap on all DGPs except Probit, whose response is muted (figures in Appendix~\ref{app:gap}). The appendix also contains illustrations on real data: PSID labour-force participation (Figure~\ref{fig:lfp_intro}) and silicon-carbide fibre strength (Appendix~\ref{app:real_data}).

\textbf{Signed-condition diagnostics.}\label{sec:bb-diagnostics}
We meta-learn two Beta-Bernoulli BFTs (600-step and 50k-step) on $\tilde\theta\sim\mathrm{Beta}(1,1)$, $Y_{1:T}\mid\tilde\theta$ i.i.d.\ $\mathrm{Bernoulli}(\tilde\theta)$. The binary outcome and absence of covariates give conditional drift $b_k$ in closed form, avoiding the Monte Carlo noise in $b_k$ that left the tests of the quasi-martingale conditions on TabPFN in \citet{fortini26principled} inconclusive. The diagnostics are consistent with the signed conditions at $\gamma=1$ on both BFTs, with residual drift bias visible on the 600-step BFT; the quasi-martingale CLT condition (Q2) at $\gamma=1$ holds on neither. This illustrates empirically the gap between the signed and quasi-martingale conditions. Full details are in Appendix~\ref{app:theorycheck}.

\textbf{Frequentist coverage.}\label{sec:frequentist}
A standard diagnostic for Bayesian uncertainty quantification is
\emph{frequentist coverage}: under repeated draws of a dataset from a
DGP, does a $(1-\alpha)$ credible set contain
the truth with probability close to $(1-\alpha)$? We evaluate coverage
and width on synthetic DGPs in $d=10$ covariate dimensions, comparing
our $\Vb_n$ and $\Ub_n$ bands against (i) the bootstrap, which
resamples $z_{1:n}$ with replacement and refits TabPFN per draw
($B=200$ samples), and (ii) the predictive Monte Carlo method of
\citet{nagler25uncertainty} (NR), which initialises a Gaussian copula
from TabPFN at a fixed query covariate and rolls the copula forward to
draw posterior samples. We restrict the comparison to these two
because deep ensembles, MC dropout, and Bayesian neural networks
require model internals or retraining, impractical for an off-the-shelf foundation model, and
conformal prediction provides no aleatoric/epistemic decomposition;
Section~\ref{sec:related} elaborates. Full setup, additional $\alpha$
values, and extra DGPs are in Appendix~\ref{app:coverage}.

\begin{table}[!ht]
  \centering
  \scriptsize
  \caption{Frequentist coverage and width of pointwise and simultaneous
    95\% credible bands for three DGPs at three sample sizes $n$. Four
    methods: our predictive CLT with $\Vb_n$ or $\Ub_n$, the bootstrap
    (Boot.), and \citet{nagler25uncertainty} (NR). The $\Vb_n$ and
    $\Ub_n$ bands approach nominal coverage as $n$ grows, with widths
    shrinking for $\Vb_n$ on the Gaussian DGPs; $\Ub_n$ is consistently more conservative than
    $\Vb_n$. The bootstrap under-covers pointwise at generally wider
    bands. NR produces wide bands on the Gaussian
    DGPs; on the Poisson DGP, its simultaneous coverage collapses to
    $0.00$ by $n=500$.}
  \label{tab:coverage}
\setlength{\tabcolsep}{3pt}
\begin{tabular}{llcccccccccccccccc}
  \toprule
  & & \multicolumn{2}{c}{$\mathbf{V}_n$ Point.} & \multicolumn{2}{c}{$\mathbf{U}_n$ Point.} & \multicolumn{2}{c}{Boot. Point.} & \multicolumn{2}{c}{NR Point.} & \multicolumn{2}{c}{$\mathbf{V}_n$ Simul.} & \multicolumn{2}{c}{$\mathbf{U}_n$ Simul.} & \multicolumn{2}{c}{Boot. Simul.} & \multicolumn{2}{c}{NR Simul.} \\
  \cmidrule(lr){3-4} \cmidrule(lr){5-6} \cmidrule(lr){7-8} \cmidrule(lr){9-10} \cmidrule(lr){11-12} \cmidrule(lr){13-14} \cmidrule(lr){15-16} \cmidrule(lr){17-18}
  DGP & $n$ & Rate & Width & Rate & Width & Rate & Width & Rate & Width & Rate & Width & Rate & Width & Rate & Width & Rate & Width \\
  \midrule
  \multirow{3}{*}{Linear}      & 200  & 0.97 & 0.39 & 0.98 & 0.47 & 0.93 & 0.42 & 0.98 & 0.85 & 0.88 & 0.67 & 0.98 & 0.81 & 1.00 & 0.82 & 1.00 & 0.98 \\
                               & 500  & 0.99 & 0.30 & 1.00 & 0.62 & 0.93 & 0.34 & 0.99 & 0.70 & 1.00 & 0.53 & 1.00 & 1.04 & 1.00 & 0.76 & 1.00 & 0.98 \\
                               & 1000 & 1.00 & 0.25 & 1.00 & 0.53 & 0.91 & 0.30 & 1.00 & 0.55 & 1.00 & 0.44 & 1.00 & 0.89 & 1.00 & 0.73 & 1.00 & 0.98 \\
  \midrule
  \multirow{3}{*}{Dependent}   & 200  & 0.90 & 0.38 & 0.96 & 0.51 & 0.94 & 0.47 & 0.97 & 0.89 & 0.68 & 0.65 & 0.90 & 0.88 & 1.00 & 0.87 & 0.98 & 0.96 \\
                               & 500  & 0.97 & 0.33 & 1.00 & 0.73 & 0.93 & 0.39 & 0.99 & 0.68 & 0.92 & 0.59 & 1.00 & 1.24 & 1.00 & 0.83 & 1.00 & 0.97 \\
                               & 1000 & 0.99 & 0.29 & 1.00 & 0.64 & 0.92 & 0.35 & 1.00 & 0.50 & 1.00 & 0.50 & 1.00 & 1.08 & 1.00 & 0.80 & 1.00 & 0.95 \\
  \midrule
  \multirow{3}{*}{Poisson}     & 200  & 0.87 & 0.07 & 0.90 & 0.08 & 0.90 & 0.13 & 0.30 & 0.10 & 0.48 & 0.11 & 0.68 & 0.14 & 1.00 & 0.45 & 0.10 & 0.10 \\
                               & 500  & 0.90 & 0.07 & 0.99 & 0.18 & 0.86 & 0.11 & 0.12 & 0.01 & 0.62 & 0.13 & 0.94 & 0.30 & 1.00 & 0.39 & 0.00 & 0.02 \\
                               & 1000 & 0.98 & 0.10 & 1.00 & 0.25 & 0.85 & 0.10 & 0.07 & 0.00 & 0.96 & 0.18 & 1.00 & 0.43 & 1.00 & 0.38 & 0.00 & 0.01 \\
  \bottomrule
\end{tabular}
\end{table}

Table~\ref{tab:coverage} reports coverage and width at $\alpha=0.05$ over $50$ Monte~Carlo replications, giving Binomial standard error $\sqrt{p(1-p)/50}\le 1/(2\sqrt{50})\approx 0.07$ on each coverage estimate (${\sim}\,0.03$ near nominal). Both $\Vb_n$ and $\Ub_n$ approach nominal coverage as $n$ grows, and meet or exceed it at $n=1000$ on every DGP. In nearly every cell, $\Vb_n$ produces the narrowest bands among the methods that retain coverage. $\Ub_n$'s bands are wider than $\Vb_n$'s across DGPs and, on the Linear, Dependent, and Poisson DGPs, can exceed the bootstrap (e.g., $0.64$ vs.\ $0.35$ pointwise on Dependent at $n=1000$).
NR \citep{nagler25uncertainty} matches $\Vb_n$/$\Ub_n$ on the Gaussian DGPs
but produces bands roughly twice as wide ($0.55$ vs.\ $0.25$ at $n=1000$ on
the Linear DGP), and on the Poisson DGP the Gaussian copula underlying NR
collapses: the simultaneous coverage drops to $0.00$ while $\Vb_n$ holds
at $0.96$. The bootstrap consistently under-covers pointwise and over-covers
simultaneously.
Widths shrink with $n$ for $\Vb_n$ on the Gaussian DGPs (Linear, Dependent), consistent with the $1/\sqrt{n}$ rate of the predictive CLT.

\textbf{Entropic UD.}\phantomsection\label{sec:entropic_ud}
For binary classification with test covariate $x^\ast$ and label $Y^\ast \in \{0,1\}$, we decompose total predictive entropy \citep[cf.][Eq.~(3)]{jayasekera25variational} as $\mathbb H(Y^\ast \mid x^\ast, z_{1:n}) = U_a + U_e$. The aleatoric term $U_a := \EE[h(\tilde g(x^\ast)) \mid z_{1:n}]$ is the expected entropy of the limiting predictive distribution $\tilde g(x^\ast)$ (with $h(p) = -p\log p - (1-p)\log(1-p)$ the binary entropy); the epistemic term $U_e$ is the remainder. The predictive CLT supplies estimators of $U_a$ and $U_e$; derivations and the extension to multiclass classification are in Appendix~\ref{app:entropy-decomp}.

Absent ground truth for
epistemic uncertainty, we evaluate any decomposition against three
broadly agreed qualitative properties: (i) epistemic uncertainty decreases as
the context length $n$ grows; (ii) epistemic uncertainty is largest far from
observed covariates; and (iii) aleatoric uncertainty dominates near decision
boundaries (classification) or in high-noise regions (regression).
Figure~\ref{fig:entropic-ud} illustrates these on synthetic experiments. The logistic-regression panel exhibits (ii) and (iii), and (i) for in-distribution $x$ (the raw uncertainty curves in Figure~\ref{fig:ud-logreg-context-length-raw} show (i) directly); the two-moons and spirals heatmaps exhibit (iii) at class overlaps, with (i) for two moons shown in Appendix~\ref{app:moons} (Figure~\ref{fig:appendix-moons-grid}). On two moons and spirals, the framework returns low epistemic uncertainty in covariate regions far from the in-context data manifold, where TabPFN reverts to a maximum-entropy predictive distribution. This may be useful for Bayesian optimisation when the optimum should remain close to this data manifold (e.g., to keep candidate molecules stable or non-toxic): standard upper confidence bound (UCB) acquisition can recommend extreme candidates. An acquisition built on our epistemic variance estimator should be less prone to pushing the search into such regions.
Experimental details and additional results are in Appendix~\ref{app:entropic_ud_experiments}.

\begin{figure}[tbp]
  \centering
  \begin{subfigure}[t]{0.28\linewidth}
    \centering
    \includegraphics[width=\linewidth]{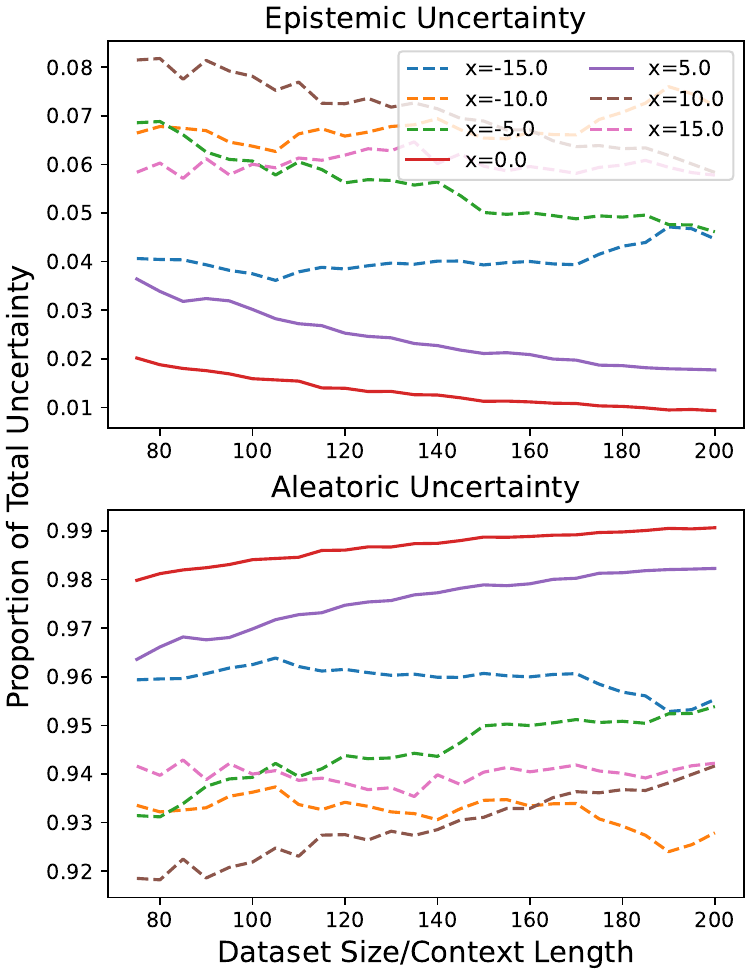}
    \caption{Logistic regression.}
    \label{fig:ud-logreg-context-length}
  \end{subfigure}\hfill
  \begin{subfigure}[t]{0.68\linewidth}
    \centering
    \includegraphics[width=\linewidth]{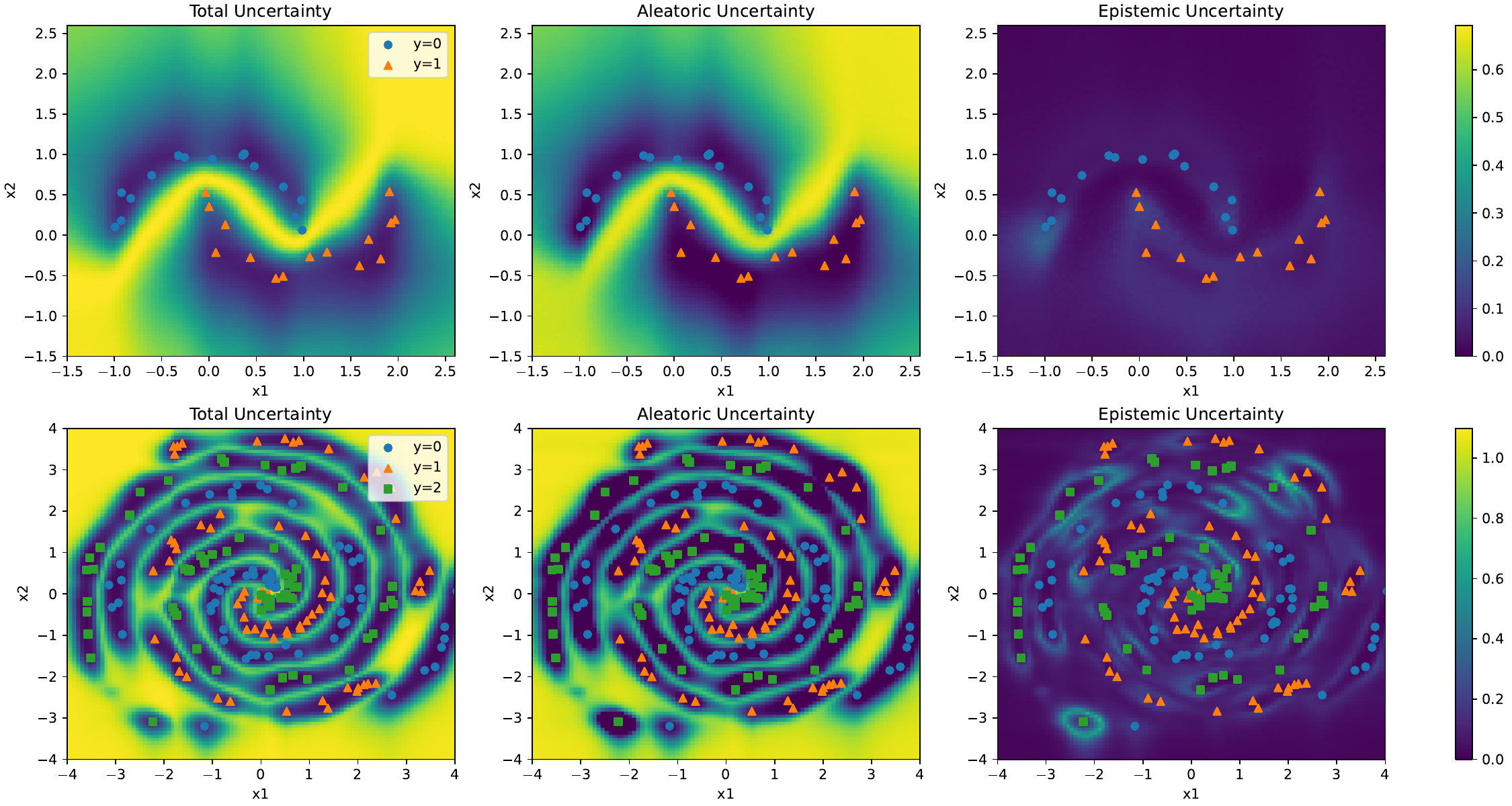}
    \caption{Two moons (top) and three-class spirals (bottom). }
    \label{fig:2moons-spirals}
  \end{subfigure}
  \caption{Entropic UD for TabPFN on synthetic tasks. \emph{Left, logistic
    regression}: aleatoric and epistemic uncertainty proportions versus context
    length for several test covariates $x$, with solid lines for in-distribution
    $x$ and dotted lines for out-of-distribution $x$. For in-distribution $x$,
    the aleatoric share grows and the epistemic share shrinks as $n$ grows;
    epistemic is highest at out-of-distribution
    $x$, and aleatoric peaks at $x=0$, the test point nearest the decision
    boundary. \emph{Right, two
    moons (top) and three-class spirals (bottom)}: total uncertainty peaks
    where classes overlap, attributed to aleatoric. Far from the data TabPFN
    reverts to a maximum-entropy predictive distribution; this region carries high total
    entropy with low epistemic, so the decomposition assigns it to aleatoric.}
  \label{fig:entropic-ud}
\end{figure}

\raggedbottom
\section{Conclusion}
\label{sec:conclusion}

\textbf{Summary.} We presented the first UD for Bayes-filtered transformers with
no access to model internals or meta-learning data, and without
computationally expensive predictive Monte Carlo simulation. The mechanism is a
predictive CLT for supervised settings under our \emph{signed conditions},
among the weakest known in the BPI literature and weaker than the
quasi-martingale conditions of \citet{fortini26principled}. A meta-learned
Beta-Bernoulli BFT is consistent with the
signed conditions while failing the quasi-martingale CLT condition (Q2) at
$\gamma=1$ (Appendix~\ref{app:theorycheck}), the first
positive empirical evidence for any BPI sufficient conditions on a
meta-learned model, narrowing the gap between rigorous BPI theory and
the practical behaviour of modern tabular foundation models. On TabPFN, the resulting credible bands achieve
near-nominal frequentist coverage as context length grows and the entropic decomposition largely matches
standard desiderata, with low epistemic uncertainty far from the data manifold
where TabPFN reverts to a maximum-entropy predictive distribution
(Section~\ref{sec:entropic_ud}); the framework also reproduces textbook Bayesian
credible bands on real data.

\textbf{Limitations.} We did not test the signed conditions empirically on
TabPFN. Diagnosing the stronger quasi-martingale conditions
\citep{fortini26principled} on TabPFN was inconclusive, as Monte~Carlo noise
from the rollouts overwhelmed the signal. The signed conditions on the
conditional drift $b_n$ and the condition on the inflated outer product of
residuals $\Rb_n$ inherit the same obstacle. Monte~Carlo noise also
frustrates empirical estimation of the convergence rate $\gamma$, which
is pinned by condition~$(iii)$ of Theorem~\ref{th:ascondmult} but unknown
for BFTs such as TabPFN; we therefore assumed $\gamma=1$ throughout. Developing diagnostics that survive Monte~Carlo noise and estimating $\gamma$ directly remain the main open problems.

\textbf{Outlook.} Two orthogonal directions can advance the applicability of BPI to BFTs. Theoretically, the sufficient conditions for BPI tools like the predictive CLT can be weakened further; our signed conditions are one step in this direction. Separately, as compute scales, BFTs will be trained closer to their idealised Bayesian limit, making the sufficient conditions easier to satisfy; the Beta-Bernoulli BFT trained for fewer steps in Appendix~\ref{app:theorycheck} already appears closer to violating them than its better-trained counterpart. Moving towards asymptotic partial exchangeability \citep{fortini18notion} would extend the predictive CLT to BFTs meta-learned on partially exchangeable data (e.g., complex random-effects models).


%% file: appendix.tex
\section{Experiment code and default settings}
\label{app:code}

The code is released at \url{https://github.com/weiyaw/ud4pfn}.

For TabPFN we use $\nest=16$ for the frequentist
coverage experiment of Appendix~\ref{app:coverage}; all other TabPFN
experiments use the default $\nest=64$. We otherwise use
\texttt{tabpfn==6.2.0} with default settings
(\texttt{softmax\_temperature=1.0}, default initialisation), except that
TabPFN is instantiated with \texttt{fit\_mode="low\_memory"}. We use
\texttt{tabpfn-v2.5-classifier-v2.5\_default.ckpt} for \tabclf{} and
\texttt{tabpfn-v2.5-regressor-v2.5\_default.ckpt} for \tabreg, both
downloaded from
\url{https://huggingface.co/Prior-Labs/tabpfn\_2\_5/tree/main}.

\textbf{Compute resources.} The frequentist coverage experiment (Section~\ref{sec:frequentist}) required approximately 700 GPU-hours on NVIDIA L40S. All other experiments (Beta-Bernoulli diagnostics, real data illustrations, and UD visualisations) can be reproduced with a few GPU-hours of compute.

\textbf{Licences for assets used.} TabPFN v2.5 is used under the Prior Labs Licence (Apache 2.0 with additional attribution requirement) \citep{priorlabs25tabpfn}. The PSID labour-force participation data are used under the University of Michigan PSID Conditions of Use (\url{https://psidonline.isr.umich.edu/}), which permit academic use with attribution. The silicon-carbide fibre strength data are reproduced from Example~7.4 of \citet{hamada08bayesian} (Springer, 2008) under standard academic fair-use provisions.

\FloatBarrier
\section{Applying the framework to TabPFN}
\label{app:ud4tabpfn}

Let $\tilde P$ represent the latent Bayesian model in TabPFN. Our goal is to
perform, given $z_{1:n}$, posterior inference for $\tilde{P}(x,A)$ via the
predictive CLT. We regard the context $z_{1:n}$ passed into
TabPFN as a table where each row is a unit of observation and each column is a feature.

\textbf{Data ordering.} By design, TabPFN is invariant to \textit{row}
permutations, meaning the prediction $\Pb_n$ depends only on the \textbf{set}
$z_{1:n}$. However, the \textit{trajectory} of updates, e.g.,
$\gb_1,\ldots, \gb_n$ and $\Fb_1,\ldots, \Fb_n$---and thus the asymptotic
variance estimator $\Vb_n$---depends on the order of the observations. In
binary classification, for example, if the rows were sorted by class, the early updates would
be deceptively small. To recover a representative trajectory that reflects the
true epistemic uncertainty of $\tilde{\Pb}(\xb,A)$, we randomly permute the rows of
$z_{1:n}$ before computing the asymptotic variance.

\textbf{Computing \eqref{eq:Vn_xA}.} We construct $\mathbf{V}_n$ by evaluating
TabPFN on expanding (non-empty) prefixes of the randomly permuted context. For
prefixes that lack diversity we substitute the predictive distribution with the
observed statistics: in classification, for single-class prefixes $g_k(x)$ is
set to $1$ if the prefix contains only class $1$ and to $0$ if it contains only
class $0$; in regression, $F_k(x, t)$ is set to the
marginal empirical CDF of $y_{1:k}$ when fewer than two unique values are
present. Otherwise, $g_k(x)$ and $F_k(x, t)$ are obtained via softmax
(temperature=1) over the output logits or the discretised bins covering
$(-\infty, t]$, respectively. Because the prior predictive $\Pb_0$ is undefined for
TabPFN, which requires at least one context token, the implementation of
$\Vb_n$ sums from $k=2$ and normalises by $n-1$; the Beta-Bernoulli
experiments use the full sum with $g_0=1/2$, the analytic prior predictive.

\textbf{Computational efficiency.} Computing $\Vb_n$ requires one forward pass per prefix $z_{1:k}$, $k=1,\ldots,n$. TabPFN's encoder-style architecture re-attends over the full context each call, so a pass on a prefix of size $k$ costs $O(k^2)$ in sequence length. Summing over prefixes,
\begin{equation*}
  \sum_{k=1}^n k^2 \;=\; \frac{n(n+1)(2n+1)}{6} \;\sim\; \frac{n^3}{3},
\end{equation*}
giving an overall $O(n^3)$ cost. With $J$ test queries riding along inside each prefix's forward pass at attention cost $O((k+J)^2)$ for prefix size $k$, the full cost is $O(n^3 + n^2 J + n J^2)$, dominated by the $n^3$ term when $J \ll n$. This cost must be contextualised within TabPFN's operational scope---typically small-to-medium tabular datasets (e.g., $n \le 10{,}000$)---where total wall-clock time remains modest. Nothing in our framework prevents a future TabPFN variant from adopting decoder-style causal masking, which would amortise the forward passes and reduce the overall cost from $O(n^3)$ to $O(n^2)$.
This cost is negligible when compared to
the methodological alternative that still uses a deep neural network: to achieve
comparable UD without a foundation model, a practitioner
would be forced to specify a bespoke Bayesian neural network and perform
computationally intensive inference (e.g., via MCMC or variational inference)
for each new dataset. Our framework allows users to extract rigorous epistemic
uncertainty estimates directly from the meta-learned TabPFN model via in-context
updates, replacing the burdensome tasks of model specification and posterior
approximation with a fast procedure.

\textbf{Asymptotic validity.} Our coverage experiments show the approximation is
accurate by $n=1000$; modern TabPFN variants handle contexts up to
$n=10{,}000$, well beyond this. While theoretical finite-sample bounds under our signed
conditions remain open problems, our extensive empirical coverage
results (Section~\ref{sec:frequentist}) demonstrate that the Gaussian
approximation holds well at moderate $n$ on most data-generating processes,
with the Poisson simultaneous bands requiring larger $n$.

\textbf{Ensembling.} TabPFN is not invariant to \textit{column} permutations; it
natively handles this by averaging predictions over an ensemble of column
shuffles (controlled by \nest, the number of column-permutation averages in TabPFN's built-in ensemble). We utilise this built-in mechanism during the
computation of the asymptotic variance to ensure the individual predictions for
each prefix are stable, thereby reducing the noise in the estimated increments.
Unless otherwise specified, we set $\nest=64$.

\newcommand{\sixhundredbeta}{1.155}
\newcommand{\sixhundredgq}{0.31}
\newcommand{\sixhundredmeanbn}{4.5\times 10^{-6}}
\newcommand{\fiftykbeta}{1.108}
\newcommand{\fiftykgq}{0.22}
\newcommand{\fiftykmeanbn}{1.7\times 10^{-6}}

\FloatBarrier
\section{Asymptotic credible intervals and bands}
\label{app:credible_bands}

Let $\Rb_n(\xb,\tb)$ be defined as in condition~$(iii)$ of Theorem~\ref{th:ascondmult} and let
\begin{equation}\label{eq:unmult}
\Ub_{n}(\xb,\tb)= \frac{n^{\gamma+1} }{\gamma}\EE\left[\Deltab_{n+1}(\xb,\tb)\Deltab_{n+1}(\xb,\tb)^\top\mid Z_{1:n}\right]
\end{equation}
and
\begin{equation}\label{eq:Vnmult}
        \Vb_n(\xb,\tb)=\frac{1}{n
        }\sum_{k=1}^n \frac{k^{\gamma+1}}{\gamma}\Deltab_k(\xb,\tb)\Deltab_k(\xb,\tb)^\top.
\end{equation}

\begin{theorem}\label{th:residualmult}
   If  the following conditions hold:
    \begin{description}
    \item[$(i)$] For $i=1,\dots,m$, $\sum_{n\geq 1} n^{2\gamma}\EE\left[\Delta_n^4(x_i,t_i)\right]<+\infty$;
        \item[$(ii)$] $\Ub_n(\xb,\tb)$ converges $\PP$-a.s.\ to a positive definite random matrix $\Vb(\xb,\tb)$.
    \end{description}
    Then $\Rb_n(\xb,\tb)$ and $\Vb_n(\xb,\tb)$ converge  to $\Vb(\xb,\tb)$  $\PP $-a.s.\ as $n\rightarrow\infty$.
    \end{theorem}

\begin{theorem}\label{th:jointnormmult}
    Under the assumptions of Theorem~\ref{th:convergence}, the assumptions $(ia)$--$(ii)$ of Theorem~\ref{th:ascondmult}, and $(i)$--$(ii)$ of Theorem~\ref{th:residualmult},  with $\tb$ a continuity point for $\EE[\tilde{\Fb}(\xb,\cdot)]$,
    \begin{equation*}
    \mathcal L(n^{\gamma/2}\Ub_n(\xb,\tb)^{-1/2}(\tilde \Fb(\xb,\tb)-\Fb_n(\xb,\tb))\mid Z_{1:n}){\rightarrow}\mathcal N_m({\bf 0},{\bf I}_m)
    \end{equation*}
    and
     \begin{equation*}
    \mathcal L(n^{\gamma/2}\Vb_n(\xb,\tb)^{-1/2}(\tilde \Fb(\xb,\tb)-\Fb_n(\xb,\tb))\mid Z_{1:n}){\rightarrow}\mathcal N_m({\bf 0},{\bf I}_m)
    \end{equation*}
    $\PP$-a.s.\ as $n\rightarrow\infty$.
\end{theorem}
While $\Vb_n$ is computed directly from the observed history, $\Ub_n$ in \eqref{eq:unmult} requires evaluating an expectation over $Z_{n+1}=(X_{n+1},Y_{n+1})$. We approximate this by Monte Carlo: $X_{n+1}$ is sampled from the empirical measure $\frac{1}{n} \sum_{i=1}^n \delta_{x_i}$, and the conditional label $Y_{n+1}\mid X_{n+1},z_{1:n}$ is sampled from the BFT's predictive distribution $P_n(\cdot\mid X_{n+1},z_{1:n})$. Beyond Monte Carlo error, sampling $X_{n+1}$ from the empirical covariate measure substitutes an approximation for the true conditional law of $X_{n+1}$ given $Z_{1:n}$ in \eqref{eq:unmult}. Under the conditions of Theorem~\ref{th:residualmult}, both $\Vb_n$ and $\Ub_n$ serve as valid asymptotic covariance estimators, up to the Monte Carlo and empirical-measure approximations of $\Ub_n$ just described.

Theorem~\ref{th:jointnormmult} and the Portmanteau theorem allow us to construct $(1-\alpha)$ asymptotic credible sets for the limiting predictive distribution. For the pointwise case ($m=1$), we obtain intervals:
\begin{equation*}
I_n^{(1)}(x,t)=F_n(x,t)\pm z_{1-\alpha/2}\left(\frac{n}{\gamma}\EE\left[\Delta_{n+1}^2(x,t)\mid Z_{1:n}\right]\right)^{1/2}
\end{equation*}
and
\begin{equation*}
I_n^{(2)}(x,t)=F_n(x,t)\pm z_{1-\alpha/2}\left(\frac{\sum_{k=1}^n \frac{k^{\gamma+1}}{\gamma} \Delta_k^2(x,t)}{n^{\gamma+1}}\right)^{1/2},
\end{equation*}
where $z_{1-\alpha/2}$ is the standard normal $(1-\alpha/2)$-quantile. For $j \in \{1,2\}$, it holds that:
\begin{equation*}
  \liminf_{n\rightarrow\infty} \PP \left[\tilde F(x,t)\in I_n^{(j)}(x,t) \mid Z_{1:n}  \right]\geq 1-\alpha.
\end{equation*}
For the simultaneous case ($m>1$), we construct studentised sup-$t$ bands \citep{olea19simultaneous}. Let $\widehat\Sigma_n \in \{\Vb_n(\xb,\tb)/n^\gamma,\,\Ub_n(\xb,\tb)/n^\gamma\}$ denote the chosen plug-in covariance on the grid, and let $s_j = \sqrt{\widehat\Sigma_{n,jj}}$ be the pointwise standard errors. Draw
\begin{equation*}
  W^{(\ell)} \sim \mathcal N_m\!\bigl(0,\widehat\Sigma_n\bigr), \qquad \ell=1,\dots,L,
\end{equation*}
and compute $T^{(\ell)} := \max_{1\le j\le m} |W^{(\ell)}_j|/s_j$. Let $c_{1-\alpha}$ be the empirical $(1-\alpha)$-quantile of $\{T^{(\ell)}\}_{\ell=1}^L$. The simultaneous sup-$t$ $(1-\alpha)$ credible band on the grid is
\begin{equation*}
  \bigl[\Pb_{n,j}\,\pm\, c_{1-\alpha}\,s_j\bigr],\qquad j=1,\dots,m,
\end{equation*}
where $\Pb_{n,j}$ is the $j$th component of $\gb_n(\xb)$ (classification) or $\Fb_n(\xb,\tb)$ (regression).

\FloatBarrier
\section{Real data illustrations}
\label{app:real_data}

\subsection{PSID labour-force participation}
\label{app:real_data_psid}

We illustrate our predictive CLT on the University of Michigan Panel Study of Income Dynamics (PSID) labour-force participation dataset studied in Chapter~12.4 of \citet{albert20probability}. The response indicates whether a married woman participates in the labour force, and the covariate is family income excluding her own earnings. The sample size is $n=753$. The textbook's canonical Bayesian analysis fits a logistic regression with a conditional-means prior and performs posterior inference via MCMC. Here, we instead apply TabPFN and construct 95\% credible bands using the predictive CLT.

Figure~\ref{fig:lfp_intro} (in the main text) plots $\gb_n(\xb)$ against income with 95\% credible bands. The estimated participation probability decreases with income; uncertainty is smallest in the central income range and widens as income grows beyond the bulk of the data, with simultaneous bands wider than pointwise bands. The resulting fit is qualitatively similar to the Bayesian logistic regression posterior summaries reported by \citet{albert20probability}, while avoiding explicit prior specification and MCMC.

\subsection{Fibre strength reliability}
\label{app:real_data_fibre}

\begin{figure}[h]
  \centering
  \includegraphics[width=\linewidth]{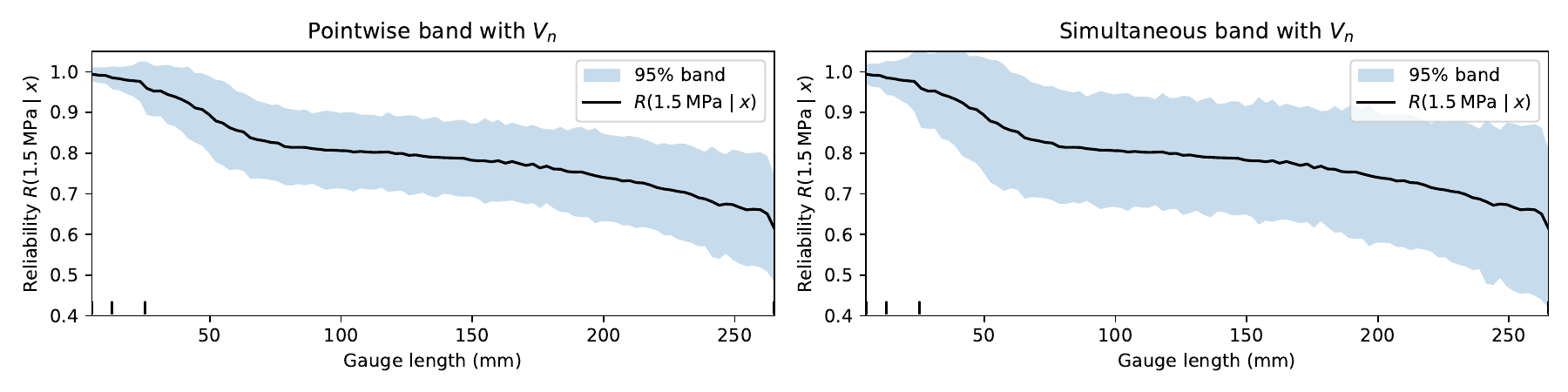}
  \caption{\emph{Left}: pointwise 95\% credible band for TabPFN predicted conditional reliability $R(1.5\text{ MPa}\mid x)$. \emph{Right}: simultaneous 95\% credible band for TabPFN predicted conditional reliability $R(1.5\text{ MPa}\mid x)$. The result is comparable to that obtained in Fig.~7.10 in \citet{hamada08bayesian} which shows length (mm) versus posterior mean and 0.05/0.95 quantiles of the conditional reliability $R(1.5\text{ MPa}\mid x)$. }
  \label{fig:fiber}
\end{figure}

Following Example~7.4 of \citet{hamada08bayesian}, we analyse tensile strengths (MPa) of silicon-carbide fibres tested at four gauge lengths
$x\in\{5.0,\,12.7,\,25.4,\,265\}\,$mm. The sample size of this dataset is $n=214$.
Let $S$ denote fibre strength and $x$ the gauge length. The textbook models the conditional distribution
$$
  S\mid x \sim \operatorname{Weibull}\!\big(\lambda(x),\,\beta(x)\big),
  \log \lambda(x)=\eta_1+\eta_2 \log x,
  \log \beta(x)=\eta_3+\eta_4 \log x,
$$
with independent diffuse Normal priors on $\eta_{1:4}$.
The target is the \emph{conditional exceedance (reliability) probability} at strength threshold $s$,
$$
  R(s\mid x)=\PP(S>s\mid x)
  = \exp\!\Big(-\,\lambda(x)\, s^{\beta(x)}\Big)
  = \exp\!\Big(-\,e^{\eta_1}x^{\eta_2}\, s^{\,e^{\eta_3}x^{\eta_4}}\Big).
$$
The traditional Bayesian workflow, which the textbook illustrates, is to first obtain posterior draws $\{\eta^{(b)}\}_{b=1}^B$ via MCMC.
For a grid of lengths $\mathcal{X}=\{x_1,\dots,x_G\}$ and fixed threshold $s=1.5$\,MPa, they evaluate
$$
  R^{(b)}(s\mid x_g)
  = \exp\!\Big(-\,e^{\eta^{(b)}_1}x_g^{\eta^{(b)}_2}\, s^{\,e^{\eta^{(b)}_3}x_g^{\eta^{(b)}_4}}\Big), \qquad g=1,\dots,G.
$$
At each $x_g$, one can summarise $\{R^{(b)}(s\mid x_g)\}_{b=1}^B$ by its posterior mean (solid curve) and the $0.05$ and $0.95$ posterior quantiles (dashed curves). Fig.~7.10  in \citet{hamada08bayesian} shows that reliability decreases with increasing length, and uncertainty widens slightly as $x$ grows.

We produce an analogous analysis to Fig.~7.10 in \citet{hamada08bayesian} using TabPFN and our predictive CLT. The result is in Figure~\ref{fig:fiber}.

\FloatBarrier
\section{Proof preliminaries}
\label{app:proof_prelim}

This appendix and the next (Appendix~\ref{app:proofs}) prove the theorems stated in the main text (Theorems~\ref{th:convergence}, \ref{th:joint}, and \ref{th:ascondmult}), together with Theorems~\ref{th:residualmult} and~\ref{th:jointnormmult} stated in Appendix~\ref{app:credible_bands}. The present appendix collects general results used across the proofs.

\subsection{A general central limit theorem}
Given a probability space $(\Omega,\mathcal F,\PP )$ and a Polish space $\mathbb S$ with its Borel sigma-algebra $\mathcal S$, a  kernel on $\mathbb S$ is a function $K:\Omega\times\mathcal S\to[0,1]$ satisfying:
\begin{description}
    \item[$(i)$] for every $\omega\in\Omega$, $K(\omega,\cdot)$ is a probability measure on $\mathcal  S$;
    \item[$(ii)$] for each $B\in\mathcal S$, the function $K(\cdot,B)$ is measurable with respect to $\mathcal F$.
\end{description}
A kernel $K$ on $\mathbb R$ is called a Gaussian kernel if $K(\omega,\cdot)$ is a Gaussian distribution with mean $\mu(\omega)$ and variance $\sigma^2(\omega)$, where $\mu$ and $\sigma^2$ are random variables defined on $(\Omega,\mathcal F,\PP )$. We denote the Gaussian kernel by $\mathcal N(\mu,\sigma^2)$ and interpret the Gaussian distribution with zero variance as the degenerate law centred on the mean.

\begin{definition}
    Let $(Z_n)$ be a sequence of random variables  adapted to a filtration $(\mathcal G_n)$ and taking values in a Polish space $\mathbb S$ with its Borel sigma-algebra $\mathcal S$. For every $n\geq 0$, let $K_n$ denote a regular version of the conditional distribution of $Z_{n+1}$, given $\mathcal G_n$. If there exists a kernel $K$ such that the sequence $(K_n(\omega,\cdot))_{n}$ converges weakly to $K(\omega,\cdot)$ for almost every $\omega\in\Omega$, then we say that the sequence $(Z_n)$ converges to $K$ in the sense of almost-sure conditional convergence.
\end{definition}

We first state a central limit theorem for martingales that will be used in the subsequent proofs.

\begin{theorem}[\cite{crimaldi2009}, Theorem A.1]
\label{th:app:almost sure conditional}
For each $n\geq 1$, let $(M_{n,j})_{j\geq 1}$ be a real valued martingale with respect to the filtration $(\mathcal F_{n,j})_{j\geq 1}$, satisfying $M_{n,0}=0$, and converging in $L^1$ to a random variable $M_{n,\infty}$. Set
$$
X_{n,j}=M_{n,j}-M_{n,j-1}\mbox{ for }j\geq 1, \quad S_n=\sum_{j\geq 1}X_{n,j}^2,\quad X_n^*=\sup_{j\geq 1}|X_{n,j}|.
$$
Assume that
\begin{itemize}
    \item[(a)] $(X_n^*)_n$ is dominated in $L^1$ and converges to zero a.s.
    \item[(b)] $(S_n)_n$ converges a.s. to a non-negative random variable $S$.
\end{itemize}
Then the sequence $(M_{n,\infty})$
                    converges to the Gaussian kernel $\mathcal N(0,S)$ in the sense of almost-sure conditional convergence with respect to the filtration $(\mathcal F_{n,1})$.
\end{theorem}

\begin{remark}
   Requesting that the almost sure limit of $(S_n)$ is non-negative guarantees the well-definedness of the kernel $\mathcal N(0,S)$ for every $\omega \in \Omega$. Theorem A.1 in {\rm \cite{crimaldi2009}}  specifies that $S$ should be a \lq\lq positive random variable''. Although not explicitly stated, this requirement should be understood as $S \geq 0$, as becomes evident upon a careful examination of the proof.
\end{remark}

We now state and prove a central limit theorem for sequences $(Z_n)$ of random variables that admit a decomposition into a martingale component and a residual term whose contribution vanishes as $n\rightarrow\infty$. To the best of our knowledge, this result is new and may be of independent interest.

\begin{theorem}\label{th:berti_new}
    Let $(Z_n)$ be a sequence of real valued random variables adapted to the filtration $(\mathcal G_n)_{n\geq 0}$ and for every $n\geq 1$ let $\Delta_n=Z_n-Z_{n-1}$.
    Assume that $(Z_n)$ is uniformly integrable and that the following conditions hold:
    \begin{description}
        \item[$(ia)$] The sequences $(n^{\gamma/2}\sum_{k> n}\EE(\Delta_k\mid \mathcal G_{k-1}))$ and $(n^{\gamma/2}\sup_{k\geq n}|\EE(\Delta_k\mid \mathcal G_{k-1})|)$  converge to zero $\PP$-a.s. and are dominated in  $L^1$;
        \item[$(ib)$]
$
    n^{\gamma}\sum_{k\geq n}\EE(\Delta_k\mid \mathcal G_{k-1})^2\rightarrow 0
$ $\PP$-a.s.;
        \item[$(ii)$] $\EE\left[   \sup_{k\geq 1} k^{\gamma/2}|\Delta_k|  \right]<+\infty$;
        \item[$(iii)$] $n^\gamma\sum_{k\geq n}\Delta_k^2{\rightarrow}V$ $\PP$-a.s.
    \end{description}
    Then there exists a random variable $Z$ such that $Z_n$ converges $\PP$-a.s. to $Z$, and
    $$
    \mathcal L( n^{\gamma/2}(Z_n-Z)\mid\mathcal G_n) {\rightarrow}\mathcal N(0,V)\quad \PP\mbox{-a.s.}
    $$
 \end{theorem}

\begin{proof} By condition $(ia)$, the sequence $(\sum_{k> n}\EE(\Delta_k\mid \mathcal G_{k-1}))$ converges to zero in $L^1$. Thus, $(\sum_{k=0}^{n-1}\EE(\Delta_{k+1}\mid \mathcal G_k))$ converges in $L^1$ and is therefore uniformly integrable.
Define $L_0=0$ and
$$
L_n=Z_n-\sum_{k=0}^{n-1}\EE(\Delta_{k+1}\mid \mathcal G_k).
$$
Then, $(L_n)$ is a martingale with respect to the filtration $(\mathcal G_n)$. Since $(Z_n)$ and $(\sum_{k=0}^{n-1}\EE(\Delta_{k+1}\mid \mathcal G_k))$ are uniformly integrable, then $L_n$ is also uniformly integrable, and therefore it converges $\PP$-a.s. and in $L^1$ to a random variable $L$. On the other hand, by $ (ia)$ $(\sum_{k=0}^{n-1}\EE(\Delta_{k+1}\mid \mathcal G_k))$ converges $\PP$-a.s. and in $L^1$. It follows that $(Z_n)$ converges $\PP$-a.s. and in $L^1$ to a limit random variable $Z$.  In particular
$$
Z_n-Z=\sum_{k\geq n}(Z_k-Z_{k+1}),\quad\quad L_n-L=\sum_{k\geq n}(L_k-L_{k+1}).
$$
By $(ia)$
\begin{align*}
 n^{\gamma/2}(Z_n-Z)-n^{\gamma/2}(L_n-L)
  &=n^{\gamma/2}\sum_{k\geq n}\left(
Z_k-\EE[Z_{k+1}\mid\mathcal G_k]  \right)\rightarrow 0\quad \PP\mbox{-a.s.}
      \end{align*}
      as $n\rightarrow\infty$.
  To conclude, it is sufficient to show that $n^{\gamma/2}(L_n-L)$ converges to $\mathcal N(0,V)$ in the sense of almost sure conditional convergence.  To this aim, we apply Theorem~\ref{th:app:almost sure conditional}.
    Define, for each $n\geq 1$, the filtration $(\mathcal F_{n,j})_{j\geq 0}$ and the process $(M_{n,j})_{j\geq 0} $ as
    $$
    \mathcal F_{n,0}=\mathcal F_{n,1}=\mathcal G_n,\quad \mathcal F_{n,j}=\mathcal G_{n+j-1}\; (j\geq 2)
    $$
    $$
    M_{n,0}=M_{n,1}=0,\quad M_{n,j}=n^{\gamma/2}(L_n-L_{n+j-1})\;(j\geq 2).
    $$
    The process $(M_{n,j})_{j\geq 0} $ is a martingale with respect to the filtration $(\mathcal F_{n,j})_{j\geq 0}$, converging $\PP$-a.s. and in $L^1$ to $n^{\gamma/2}(L_n-L)$. To conclude, we need to show that $(M_{n,j})$ satisfies conditions (a) and (b) of Theorem~\ref{th:app:almost sure conditional}.

    The increment $X_{n,j}=M_{n,j}-M_{n,j-1}$ is equal to zero for $j=1$, while, for $j\geq 2$,
    \begin{align*}
        X_{n,j}&=n^{\gamma/2}(L_{n+j-2}-L_{n+j-1})\\&=n^{\gamma/2}(\EE[Z_{n+j-1}\mid \mathcal G_{n+j-2}]-Z_{n+j-1})\\
        &=n^{\gamma/2}\EE(\Delta_{n+j-1}\mid \mathcal G_{n+j-2})
        -n^{\gamma/2}\Delta_{n+j-1}.
    \end{align*}
    Thus,
    \begin{align*}
        \sup_{j\geq 1}|X_{n,j}|\leq n^{\gamma/2}\sup_{k\geq n}|
        \EE(\Delta_{k}\mid \mathcal G_{k-1})|+
        \sup_{k\geq n}k^{\gamma/2}|\Delta_k|.
    \end{align*}
    By condition $(ia)$ $n^{\gamma/2}\sup_{k\geq n}|
        \EE(\Delta_{k}\mid \mathcal G_{k-1})|$ is dominated in $L^1$ and converges to zero $\PP$-a.s.
          Also $ \sup_{k\geq n}k^{\gamma/2}|\Delta_{k}|$ is dominated in $L^1$ by assumption $(ii)$. To prove that it converges to zero $\PP$-a.s., we can write
    \begin{align*}
       & n^\gamma \Delta_{n+1}^2=n^\gamma\sum_{k\geq n}\Delta_{k+1}^2-\frac{n^\gamma}{(n+1)^\gamma}(n+1)^\gamma \sum_{k\geq n+1}\Delta_{k+1}^2\rightarrow 0\quad \PP\mbox{-a.s.}
    \end{align*}
    It follows that $$n^{\gamma/2}|\Delta_{n+1}|\rightarrow 0,\quad\PP\mbox{-a.s.}$$
    which implies  that
    \begin{equation}\label{eq:appsup}
        \sup_{k\geq n}k^{\gamma/2}|\Delta_{k}|\rightarrow 0\quad \PP\mbox{-a.s.}
    \end{equation}
  This concludes the proof of condition (a) of Theorem~\ref{th:app:almost sure conditional}.
To prove condition (b), we can write that
\begin{align}\label{eq:appxnj}
    \sum_{j\geq 1}X_{n,j}^2&=n^{\gamma}\sum_{k\geq n+1}\EE(\Delta_{k}\mid \mathcal G_{k-1})^2+ n^{\gamma}\sum_{k\geq n+1}\Delta_{k}^2\nonumber\\
    &-2n^\gamma\sum_{k\geq n+1}\Delta_k\EE(\Delta_k\mid \mathcal G_{k-1}).
\end{align}
The first term in the right-hand side converges to zero $\PP$-a.s. by condition $(ib)$.
The last term can be bounded in absolute value  by
\begin{align*}
    &|n^{\gamma}\sum_{k\geq n+1}\Delta_k\EE(\Delta_k\mid\mathcal G_{k-1})|\leq
    (n^{\gamma}\sum_{k\geq n+1}\Delta_k^2)^{1/2}
    (n^{\gamma}\sum_{k\geq n+1}\EE(\Delta_k\mid\mathcal G_{k-1})^2)^{1/2},
\end{align*}
which converges to zero $\PP$-a.s. by conditions $(ib)$ and $(iii)$.
By equation \eqref{eq:appxnj} and  assumption (iii), $\sum_{j\geq 1}X_{n,j}^2$ converges to $V$, $\PP$-a.s., which proves that the assumption (b) of Theorem~\ref{th:app:almost sure conditional} holds.
\end{proof}

The following result gives sufficient conditions for assumption (iii) in Theorem~\ref{th:berti_new}.

\begin{theorem}[\cite{crimaldi2016}, Lemma 4.1]
\label{th:suff}
    Let $(W_n)$ be a sequence of real valued random variables adapted to the filtration $(\mathcal G_n)$ and such that $\EE(W_{n+1}\mid \mathcal G_n)\rightarrow W$ a.s. for some random variable $W$. Moreover, let $(c_n)$ and $(d_n)$ be sequences of positive real numbers such that
    $$
    d_n\uparrow +\infty,\quad \sum_{k=1}^\infty \frac{\EE(W_k^2)}{c_k^2d_k^2}<+\infty.
    $$
Then we have:
\begin{itemize}
    \item[(a)] If $\frac{1}{d_n}\sum_{k=1}^n \frac{1}{c_k}\rightarrow c$ for some constant $c$, then $\frac{1}{d_n}\sum_{k=1}^n \frac{W_k}{c_k}\rightarrow cW$ a.s.
    \item[(b)] If $d_n\sum_{k\geq n}\frac{1}{c_kd_k^2}\rightarrow c$ for some constant $c$, then ${d_n}\sum_{k\geq n} \frac{W_k}{c_kd_k^2}\rightarrow c W$ a.s.
\end{itemize}
\end{theorem}

\FloatBarrier
\section{Proofs}
\label{app:proofs}

We now prove Theorems~\ref{th:convergence}, \ref{th:joint}, \ref{th:ascondmult}, \ref{th:residualmult}, and~\ref{th:jointnormmult}.

\subsection{Proof of Theorem~\ref{th:convergence}}
For every $(x,t)$, we can write
$$
F_n(x,t)=L_n(x,t)+F_0(x,t)+\sum_{k=1}^n\EE(\Delta_k(x,t)\mid Z_{1:k-1}),
$$
where $L_0(x,t)=0$ and for $n\geq 1$
$$
L_n(x,t)=\sum_{k=1}^n(F_k(x,t)-\EE(F_k(x,t)\mid Z_{1:k-1})).
$$
The sequence $(L_n)$ is a martingale with respect to the natural filtration of $(Z_n)$, with increments $L_n(x,t)-L_{n-1}(x,t)=F_n(x,t)-\EE(F_n(x,t)\mid Z_{1:n-1})$ bounded in absolute value by $1$. By condition~(iii) of Theorem~\ref{th:convergence}, the series $\sum_{k=1}^n\EE(\Delta_k(x,t)\mid Z_{1:k-1})$ converges $\PP$-a.s.; since $L_n(x,t)=F_n(x,t)-F_0(x,t)-\sum_{k=1}^n\EE(\Delta_k(x,t)\mid Z_{1:k-1})$ and $F_n(x,t)\in[0,1]$, the paths of $(L_n)$ are therefore $\PP$-a.s. bounded. A martingale with bounded increments converges a.s. on the event where its paths are bounded \citep[see, e.g.,][]{kallenberg21foundations}, so $L_n(x,t)$ converges $\PP$-a.s. to a random limit.
Let $D$ be a countable dense subset of $\mathcal X$, which exists since $\mathcal X\subseteq\mathbb R^d$ is separable. It follows that, with probability one, $F_n(r,s)$ converges to a limit $H(r,s)$ for every $r\in D$ and every rational $s$. Fix $\omega$ such that the convergence holds. For simplicity of notation, we omit $\omega$ from here on.
Fix $x$ and let $r_n'$ and $r_n''$ be two sequences in $D$ converging to $x$ as $n\rightarrow\infty$. Fix $\epsilon>0$ and let $\delta$ be such that
$$
  \sup_m|F_m(x,s)-F_m(z,s)|<\epsilon
$$
for every $z$ satisfying $||z-x||<\delta$.
Let $n_0$ be such that for every $n\geq n_0$, $||r_n'-x||<\delta$ and $||r_n''-x||<\delta$. Then, for $n\geq n_0$,
\begin{align*}
  |H(r_n',s)-H(r_n'',s)|
   & \leq
  |H(r_n',s)-F_m(r_n',s)|+\sup_m|F_m(r_n',s)-F_m(x,s)|
  \\
   & +\sup_m|F_m(x,s)-F_m(r_n'',s)|+|H(r_n'',s)-F_m(r_n'',s)| \\
   & <4\epsilon,
\end{align*}
for $m$ large enough. Hence we can define
for every $s\in\mathbb Q$ and every $x$
$$
  H(x,s)=\lim_n H(r_n,s)
$$
where $r_n$ is any sequence of points of $D$ converging to $x$ as $n\rightarrow\infty$.

We now prove that $F_n(x,s)$ converges to $H(x,s)$, for every rational $s$ and every $x$. We can write
\begin{align*}
   & |F_n(x,s)-H(x,s)|                     \\
   & \leq \sup_n |F_n(x,s) -F_n(r_m,s)|
  +|F_n(r_m,s)-H(r_m,s)|+|H(r_m,s)-H(x,s)| \\
   & <3\epsilon,
\end{align*}
if $n$ and $m$ are large enough.
The function $H(x,s)$ is monotone non decreasing in $s\in\mathbb Q$ for every $x$. Indeed, we can write for any  $s<t$, and any sequence $(r_n)$ of points of $D$ converging to $x$,
\begin{align*}
  H(x,t)-H(x,s) & =\lim_n ( H(r_n,t)-H(r_n,s))                 \\
                & =\lim_n\lim_m (F_m(r_n,t)-F_m(r_n,s))\geq 0.
\end{align*}

We now define, for every $x$ and $t$
$$
  H(x,t)=\inf_{s>t,s\in\mathbb Q}H(x,s)
$$
The function $x\mapsto H(x,t)$ is measurable for every $t$. Moreover,
for every $x$, $H(x,\cdot)$ is monotone non decreasing, and by assumption $(i)$, it converges to $0$ as $t \to -\infty$ and to $1$ as $t \to +\infty$. We now show that it is continuous from the right. Fix $t$ and $\epsilon>0$ and let $s>t$ be such that
$$
  H(x,s)<H(x,t)+\epsilon.
$$
Then for every $u\in (t,s)$
$$
  H(x,t)\leq H(x,u)\leq H(x,s)<H(x,t)+\epsilon,
$$
which proves right-continuity at $t$. Thus, $H(x,\cdot)$ is a distribution function on $\mathbb R$ for every $x$, and it is a measurable function of $x$ for every $t$. It is therefore a kernel.

We now prove that, for every $x$, $F_n(x,\cdot)$ converges to $H(x,\cdot)$ weakly. Let $t$ be a continuity point for $H(x,\cdot)$ and let $\epsilon>0$. There exist $u<r<t<s$ with $s,r\in\mathbb Q$,
$$
  H(x,u)>H(x,t)-\epsilon,\quad H(x,s)<H(x,t)+\epsilon.
$$
Then we can write
\begin{align*}
  H(x,t)-\epsilon \leq H(x,u)\leq H(x,r)\leq H(x,t)\leq H(x,s)\leq H(x,t)+\epsilon.
\end{align*}
Since $r$ is rational and $r>u$, the limit $\lim_n F_n(x,r)$ (the value of $H(x,r)$ before the right-continuous modification) dominates $H(x,u)=\inf_{q>u,q\in\mathbb Q}\lim_n F_n(x,q)$. Then
\begin{align*}
  \liminf_n F_n(x,t)\geq \lim_n F_n(x,r)
   & \geq H(x,u)\geq H(x,t)-\epsilon,
\end{align*}
\begin{align*}
  \limsup_n F_n(x,t)\leq \limsup_n F_n(x,s)
  \leq H(x,s)\leq H(x,t)+\epsilon.
\end{align*}
The convergence is a consequence of the arbitrariness of $\epsilon$.
\\
$\square$
\subsection{Proof of Theorem~\ref{th:joint}}
Denote by $\tilde F_x$ the random probability measure with distribution function $t\mapsto \tilde F(x,t)$. We prove the theorem by induction on $k$.
By Theorem~\ref{th:convergence},
$$
  \PP [Y_{n+1}\in dy_1\mid Z_{1:n}, X_{n+1}=x_1] {\rightarrow}
  \tilde F_{x_1}(dy_1).$$
Now suppose that the claim is true for every $j\leq  k-1$. By the conditional independence assumption (iv) and the induction hypothesis, we can write
\begin{align*}
   & \PP\bigl[Y_{n+1}\!\in dy_1,\dots,Y_{n+k}\!\in dy_k
  \mid Z_{1:n}, X_{n+1:n+k}=x_{1:k}\bigr]                          \\
   & \quad\quad\quad=
  \PP\bigl[Y_{n+k}\!\in dy_k
  \mid Z_{1:n}, X_{n+1:n+k}=x_{1:k}, Y_{n+1:n+k-1}=y_{1:k-1}\bigr] \\
   & \quad\quad\quad\quad{}\quad\times
  \PP\!\bigl[Y_{n+1}\!\in dy_1,\dots,Y_{n+k-1}\!\in dy_{k-1}
  \mid Z_{1:n}, X_{n+1:n+k-1}=x_{1:k-1}\bigr] \\
   & \quad\quad\quad\quad\quad{\longrightarrow}
  \prod_{i=1}^k \tilde F_{x_i}(dy_i)\quad \PP\mbox{-a.s.}
\end{align*}
$\square$

\subsection{Proof of Theorem~\ref{th:ascondmult}}
Let $\mathbf t$ be a continuity point of $\EE[\tilde{\Fb}(\xb,\cdot)]$. Then, for each $i=1,\dots,m$, the function $\EE[\tilde F(x_i,\cdot)]$ is continuous at $t_i$. As a consequence,
\[\PP
[\tilde F(x_i,t_i)-\tilde F(x_i,t_i^-)>0]=0,
\]
where $\tilde F(x_i,t_i^-)=\lim_{t\to t_i^-}\tilde F(x_i,t)$. By Theorem~\ref{th:convergence}, for every $i=1,\dots,m$, $F_n(x_i,t_i)$ converges to $\tilde F(x_i,t_i)$ on a set of probability one. We work on this set from here on; all subsequent statements hold almost surely.
\\
Fix a vector $\ub = [u_1,\dots,u_m]^\top$ with $\|\ub\| = 1$, and define $M_n=\ub^\top\Fb_n(\xb,\tb)$ for every $n\geq 0$ and $\delta_n=M_n-M_{n-1}$ for every $n\geq 1$.
Let $\mathcal G_0$ be the trivial sigma-algebra, and for every $n\geq 1$ let $\mathcal G_n$ be the sigma-algebra generated by $Z_{1:n}$. We now prove that $(M_n)$ satisfies the assumptions of Theorem~\ref{th:berti_new}.
The sequence  $(M_n)$ is uniformly integrable. Moreover, by assumption $(ia)$, the sequences
$$(n^{\gamma/2}\sum_{k> n}\EE(\delta_k\mid \mathcal G_{k-1})),\quad (n^{\gamma/2}\sup_{k\geq n}|\EE(\delta_k\mid \mathcal G_{k-1})|)$$
are dominated in $L^1$ and converge to zero $\PP$-a.s. Furthermore, by assumption $(ib)$
     $$
    n^{\gamma}\sum_{k\geq n}\EE(\delta_k\mid \mathcal G_{k-1})^2\rightarrow 0\quad \PP\mbox{-a.s.}
    $$
    To conclude the proof we need to verify the assumptions $(ii)$ and $(iii)$ of Theorem~\ref{th:berti_new}.
    We can write
 \begin{align*}
\EE\left[   \sup_{k\geq 1}k^{\gamma/2}|M_k-M_{k-1}|  \right]&=\EE\left[   \sup_{k\geq 1}k^{\gamma/2}|\ub^\top\Deltab_k(\xb,\tb)|  \right]\\&\leq \sum_{i=1}^m\EE\left[   \sup_{k\geq 1}k^{\gamma/2}|\Delta_k(x_i,t_i)|  \right]<+\infty,
 \end{align*}
where the last inequality holds by assumption  (ii). Furthermore, by (iii)
$$n^\gamma\sum_{k\geq n}(M_k-M_{k-1})^2=n^\gamma\sum_{k\geq n}\ub^\top\Deltab_k(\xb,\tb)\Deltab_k(\xb,\tb)^\top\ub{\rightarrow}\ub^\top\Vb(\xb,\tb)\ub\quad \PP\mbox{-a.s.}$$
Here the sum starts at $k=n$, whereas the sum in (iii) starts at $k=n+1$; the two limits coincide because the $k=n$ term vanishes: by (iii), $(n-1)^\gamma\Deltab_n(\xb,\tb)\Deltab_n(\xb,\tb)^\top=\Rb_{n-1}(\xb,\tb)-\bigl(\tfrac{n-1}{n}\bigr)^\gamma\Rb_n(\xb,\tb)\rightarrow{\bf 0}$ $\PP$-a.s., so $n^\gamma(\ub^\top\Deltab_n(\xb,\tb))^2\rightarrow 0$ $\PP$-a.s.
By Theorem~\ref{th:berti_new},
   $\mathcal L(n^{\gamma/2}(\ub^\top(\Fb_n(\xb,\tb)-\tilde \Fb(\xb,\tb)))\mid Z_{1:n}) {\rightarrow }\mathcal N(0,\ub^\top\Vb(\xb,\tb)\ub)$, $\PP$-a.s.
Since, by the Cram\'er--Wold device, weak convergence in distribution of
\[
\mathcal{L}\bigl(n^{\gamma/2}(\mathbf{F}_n(\mathbf{x},\mathbf{t})
- \tilde{\mathbf{F}}(\mathbf{x},\mathbf{t})) \mid Z_{1:n}\bigr)
\]
is implied by the convergence of
\[
\mathcal{L}\bigl(n^{\gamma/2}\,\mathbf{u}^\top
(\mathbf{F}_n(\mathbf{x},\mathbf{t})-\tilde{\mathbf{F}}(\mathbf{x},\mathbf{t}))
\mid Z_{1:n}\bigr)
\]
for all $\mathbf{u}$ in a countable dense subset of the unit sphere,
the claim follows.

\subsection{Proof of Theorem~\ref{th:residualmult}}
Let $\mathcal G_0$ be the trivial sigma-algebra, and for every $n\geq 1$, let $\mathcal G_n$ be the sigma-algebra generated by $Z_{1:n}$. Fix a vector $\ub = [u_1,\dots,u_m]^\top$ with $\|\ub\| = 1$, and define
$$W_n=\frac{n^{\gamma+1}}{\gamma}\ub^\top\Deltab_n(\xb,\tb)\Deltab_n(\xb,\tb)^\top\ub,\quad d_n=n^\gamma,\quad c_n=\frac{n^{1-\gamma}}{\gamma}.$$
Then
$$
d_n\sum_{k\geq n}\frac{1}{c_kd_k^2}=\gamma n^\gamma\sum_{k\geq n}k^{\gamma-1}k^{-2\gamma}\rightarrow 1,
$$
 $$
    \sum_{k=1}^\infty \frac{\EE[W_k^2]}{c_k^2d_k^2}=\sum_{k=1}^\infty \frac{k^{2\gamma+2}\EE[(\ub^\top\Deltab_k(\xb,\tb))^4]}{k^{2-2\gamma}k^{2\gamma}}\leq m^4\sum_{i=1}^m\sum_{k\geq 1} k^{2\gamma}\EE[\Delta_k^4(x_i,t_i)]<+\infty,
    $$
and
$$\EE[W_{n+1}\mid\mathcal G_n]=\frac{(n+1)^{\gamma+1}}{\gamma}\EE[\ub^\top\Deltab_{n+1}(\xb,\tb)\Deltab_{n+1}(\xb,\tb)^\top\ub\mid Z_{1:n}]{\rightarrow} \ub^\top\Vb(\xb,\tb)\ub\quad \PP\mbox{-a.s.},
$$
by assumptions (i) and (ii), where the identification with $\ub^\top\Ub_n(\xb,\tb)\ub$ in \eqref{eq:unmult} uses $((n+1)/n)^{\gamma+1}\rightarrow 1$. Then by Theorem~\ref{th:suff} (b), it holds $\PP$-a.s. that
\begin{align*}
 \ub^\top\Vb(\xb,\tb)\ub&=\lim_n   {d_n}\sum_{k\geq n} \frac{W_k}{c_kd_k^2}\\&=\lim_nn^\gamma \sum_{k\geq n} \frac{k^{\gamma+1}\ub^\top\Deltab_k(\xb,\tb)\Deltab_k(\xb,\tb)^\top\ub}{k^{1-\gamma}k^{2\gamma}}\\
    &=\lim_n n^\gamma \sum_{k\geq n}\ub^\top\Deltab_k(\xb,\tb)\Deltab_k(\xb,\tb)^\top\ub.
\end{align*}
This limit coincides with $\lim_n \ub^\top\Rb_n(\xb,\tb)\ub$, whose sum starts at $k=n+1$: assumption~(i) gives $\sum_n n^{2\gamma}\Delta_n^4(x_i,t_i)<\infty$ $\PP$-a.s., hence $n^{2\gamma}(\ub^\top\Deltab_n(\xb,\tb))^4\rightarrow 0$ and $n^\gamma(\ub^\top\Deltab_n(\xb,\tb))^2\rightarrow 0$ $\PP$-a.s.
Now let $d_n=n$, $c_n=1$ so that
$\frac{1}{d_n}\sum_{k=1}^n \frac{1}{c_k}\rightarrow 1$. By Theorem~\ref{th:suff} (a)
$$
\frac 1 n \sum_{k=1}^nW_k=\frac 1 n \sum_{k=1}^n\frac{k^{\gamma+1}}{\gamma}\ub^\top\Deltab_k(\xb,\tb)\Deltab_k(\xb,\tb)^\top\ub{\rightarrow}\ub^\top\Vb(\xb,\tb)\ub.
$$
$\PP$-a.s.
The above convergence holds $\PP$-a.s. for $\mathbf{u}$ ranging over a
countable dense subset of the unit sphere.
By continuity, the convergence extends to all $\mathbf{u}$ on the unit
sphere, and since convergence of the quadratic forms
$\mathbf{u}^\top A_n \mathbf{u}$ for all $\mathbf{u}$ implies convergence
of the symmetric matrices $A_n$, the claim follows.

\subsection{Proof of Theorem~\ref{th:jointnormmult}}

Since ${\Ub}_n$ and ${\Vb}_n$ are functions of $Z_{1:n}$, then under the assumptions $(ia)$--$(ii)$ of Theorem~\ref{th:ascondmult} and $(i)$, $(ii)$ of Theorem~\ref{th:residualmult}, for every matrix $\Sb$ and vector $\mathbf s$
\begin{align*}
    \EE\bigl[\exp(\im\langle{\Sb},{\Ub}_n(\xb,\tb)\rangle_F+&\im\mathbf{s}^\top n^{\gamma/2}(\Fb_n(\xb,\tb)-\tilde \Fb(\xb,\tb)))\mid Z_{1:n}\bigr]\\&\rightarrow \exp\left(\im\langle\Sb , \Vb(\xb,\tb)\rangle_F-\frac 1 2 \mathbf{s}^\top\Vb(\xb,\tb) \mathbf s\right),
\end{align*}
and
\begin{align*}
    \EE\bigl[\exp(\im\langle\Sb,\Vb_n(\xb,\tb)\rangle_F+&\im\mathbf s^\top n^{\gamma/2}(\Fb_n(\xb,\tb)-\tilde \Fb(\xb,\tb))\mid Z_{1:n}\bigr]\\&\rightarrow \exp\left(\im\langle \Sb, \Vb(\xb,\tb)\rangle_F-\frac 1 2\mathbf s^\top\Vb(\xb,\tb)\mathbf s\right),
\end{align*}
$\PP$-a.s., where $\langle\cdot,\cdot\rangle_F$ denotes the Frobenius product, and $\im$ is the imaginary unit. The convergence holds for $\mathbf{S}$ and $\mathbf s$ in a countable dense
subset of the unit spheres, which allows one to work on a single
probability--one event and thus preserve almost sure convergence.
The result is then extended to all $\mathbf{S}$ and $\mathbf s$ by continuity, and the
claim follows from the Cram\'er--Wold device.
If the matrix $\Vb(\xb,\tb)$ is $\PP$-a.s. positive definite, then $\Ub_n(\xb,\tb)$ and $\Vb_n(\xb,\tb)$ are also positive definite for all sufficiently large $n$.
Thus we obtain Theorem~\ref{th:jointnormmult}.

\FloatBarrier
\section{A general-\texorpdfstring{$\gamma$}{gamma} quasi-martingale CLT}
\label{app:qm-general-gamma}

For reference, and to make explicit what our signed conditions (Section~\ref{sec:PCLT}) are designed to weaken, we state a general-$\gamma$ version of the quasi-martingale CLT, extending the $\sqrt n$-rate result ($\gamma=1$) of \citet{fortini26principled} (their Theorem~4.3) to arbitrary rates $\gamma\in(0,1]$. Here $\gamma$ is the same rate parameter as elsewhere in this paper: the predictive CLT holds at the rate $n^{\gamma/2}$. The extension is straightforward and we state it without proof.

We state the $\gamma$-extensions of \citet[Theorems 4.1 and 4.3]{fortini26principled}. Throughout, $\Delta_k(x,t):=F_k(x,t)-F_{k-1}(x,t)$.

\begin{theorem}[Convergence under quasi-martingale; {\citet[Theorem~4.1]{fortini26principled}}]
\label{th:qm-convergence}
Assume:
\begin{description}
  \item[(i)] $\mathcal Y$ is compact;
  \item[(ii)] for every $t\in\mathcal Y$, the functions $x\mapsto F_n(x,t)$ are equicontinuous;
  \item[(iii)] for every $x$ and $t$, $\sum_{k\ge 1}\EE\!\left[\,\bigl|\EE[\Delta_k(x,t)\mid Z_{1:k-1}]\bigr|\,\right]<+\infty$.
\end{description}
Then there exists a kernel $\tilde F(x,t)$ such that $F_n(x,\cdot)\to\tilde F(x,\cdot)$ weakly $\PP$-a.s.\ for every $x\in\mathcal X$.
\end{theorem}

\begin{theorem}[General-$\gamma$ quasi-martingale CLT; extending {\citet[Theorem~4.3]{fortini26principled}}]
\label{th:qm-general-gamma-clt}
Under the assumptions of Theorem~\ref{th:qm-convergence}, suppose that for some $\gamma\in(0,1]$:
\begin{description}
  \item[(i)] for $i=1,\dots,m$, $\sum_{k\ge 1} k^{\gamma/2}\,\EE\!\left[\,\bigl|\EE[\Delta_k(x_i,t_i)\mid Z_{1:k-1}]\bigr|\,\right]<+\infty$;
  \item[(ii)] for $i=1,\dots,m$, $\EE\!\left[\sup_{k\ge 1}k^{\gamma/2}|\Delta_k(x_i,t_i)|\right]<+\infty$;
  \item[(iii)] $\Rb_n(\xb,\tb)=n^\gamma\sum_{k\ge n+1}\Deltab_k(\xb,\tb)\Deltab_k(\xb,\tb)^\top$ converges $\PP$-a.s.\ to a positive definite random matrix $\Vb(\xb,\tb)$.
\end{description}
Writing $\tilde\Fb(\xb,\cdot)=[\tilde F(x_1,\cdot),\dots,\tilde F(x_m,\cdot)]^\top$, if $\tb$ is a continuity point of $\EE[\tilde\Fb(\xb,\cdot)]$ (under $\PP$), then
\[
  \mathcal L\bigl(n^{\gamma/2}(\Fb_n(\xb,\tb)-\tilde\Fb(\xb,\tb))\mid Z_{1:n}\bigr) \;\longrightarrow\; \mathcal N_m(\mathbf 0,\Vb(\xb,\tb))\qquad\PP\text{-a.s.}
\]
\end{theorem}

\paragraph{Relation to our signed conditions.}
Our signed conditions (Theorems~\ref{th:convergence} and~\ref{th:ascondmult}) differ from the corresponding quasi-martingale (QM) conditions (condition~(iii) of Theorem~\ref{th:qm-convergence} for kernel convergence; condition~(i) of Theorem~\ref{th:qm-general-gamma-clt} for the rate-$\gamma$ CLT) in two respects. First, our conditions use \emph{signed} tail sums $\sum_{k\ge n}\EE[\Delta_k(x,t)\mid Z_{1:k-1}]$, whereas the QM versions sum absolute values $|\EE[\Delta_k(x,t)\mid Z_{1:k-1}]|$; signed sums can converge through cancellation when the absolute sums diverge. Second, our conditions are $\PP$-a.s.\ \emph{pathwise} statements, whereas the QM versions are outer-expectation conditions over the law of the sequence. The pathwise description applies to the limit parts of our conditions: the $L^1$-domination part of $(ia)$ and the moment condition $(ii)$ are expectation conditions, and the diagnostics of Appendix~\ref{app:theorycheck} test the pathwise limit parts while these domination and moment parts are assumed. Both differences make the signed conditions strictly weaker. Under the power-law ansatz $\EE|\EE[\Delta_n(x,t)\mid Z_{1:n-1}]|\asymp n^{-\beta}$ with $\beta>1$, a direct calculation shows that condition~(i) of Theorem~\ref{th:qm-general-gamma-clt} holds for $\gamma< 2\beta-2$, while random-sign cancellation in the signed partial sums extends the range to $\gamma<2\beta-1$, a structural one-unit gap.

\FloatBarrier
\section{Diagnostics for the predictive CLT sufficient conditions on a Beta-Bernoulli BFT}
\label{app:theorycheck}

We empirically assess whether the theoretical conditions required for the predictive CLT (Section~\ref{sec:PCLT}) are plausibly satisfied by a trained Beta-Bernoulli BFT. The assumptions in our theorems are asymptotic and involve expectations taken over the path of the predictive rule. The transformer's predictive rule is not given in closed form, so these conditions cannot be formally verified; instead, we perform diagnostics to detect \emph{gross violations} in a controlled setting.

In this Beta-Bernoulli specialisation the observations are binary and covariate-free, so the data tokens reduce to $Z_k = Y_k \in \{0,1\}$, and the transformer's one-step predictive is the scalar success probability
\begin{equation}
  g_k \;=\; \PP_\phi(Y_{k+1}=1 \mid Y_{1:k}),
\end{equation}
where $\phi$ denotes the trained network weights. Let $\Delta_k = g_k - g_{k-1}$ denote the realised one-step increment of this predictive probability, and let
\begin{equation}
  b_k \;=\; \EE[\Delta_k \mid Y_{1:k-1}]
  \label{eq:bb-bn}
\end{equation}
denote the conditional drift of the predictive rule (computed in closed form below; the binary covariate-free setting eliminates all Monte Carlo variance in the inner conditional expectation). The diagnostic checks two sets of drift conditions on $b_k$ (our signed conditions and the quasi-martingale conditions) together with a residual condition on $R_n$ that is shared between them.

\textbf{Our signed conditions} on the conditional drift $b_k$ are
\begin{align*}
    \textbf{(C1)} &\quad \sum_{k\ge n} b_k \to 0 \;\PP\text{-a.s.} && \text{(condition (iii) of Theorem~\ref{th:convergence})}, \\
    \textbf{(C2)} &\quad n^{\gamma/2}\sum_{k> n} b_k\to 0 \;\PP\text{-a.s.} && \text{(condition $(ia)$, first part of Theorem~\ref{th:ascondmult})}, \\
    \textbf{(C3)} &\quad n^{\gamma/2}\lvert b_n\rvert\to 0 \;\PP\text{-a.s.}\footnotemark && \text{(condition $(ia)$, second part of Theorem~\ref{th:ascondmult})}, \\
    \textbf{(C4)} &\quad n^{\gamma}\sum_{k\ge n} b_k^2\to 0 \;\PP\text{-a.s.} && \text{(condition $(ib)$ of Theorem~\ref{th:ascondmult})}.
\end{align*}
\footnotetext{The original condition $(ia)$ second part requires $n^{\gamma/2}\sup_{k\ge n}\lvert b_k\rvert\to 0$; we test the pointwise $k{=}n$ form, which is a necessary implication of the sup version.}

\textbf{The quasi-martingale conditions} of \citet{fortini26principled} on the conditional drift $b_k$ (stated formally in Appendix~\ref{app:qm-general-gamma}) come in two forms:
\begin{align*}
    \textbf{(Q1)} &\quad \sum_{k\ge 1} \EE\lvert b_k\rvert<\infty && \text{(condition (iii) of Theorem~\ref{th:qm-convergence})}, \\
    \textbf{(Q2)} &\quad \sum_{k\ge 1} k^{\gamma/2}\EE\lvert b_k\rvert<\infty && \text{(condition (i) of Theorem~\ref{th:qm-general-gamma-clt})}.
\end{align*}

\textbf{The residual condition} on the inflated outer product of residuals $R_n$ is shared between the two sets:
\begin{align*}
    \textbf{(R)} &\quad n^{\gamma}\sum_{k> n}\Delta_k^2\to V \;\PP\text{-a.s., some finite } V>0 && \text{(condition $(iii)$ of Theorems~\ref{th:ascondmult} and~\ref{th:qm-general-gamma-clt})}.
\end{align*}

The sampling arm of BPI, predictive Monte Carlo (PMC), in which one draws rollouts from the predictive rule and averages functionals along them, is justified as soon as the predictive rule itself converges, which is delivered by~(C1) among the signed conditions or by~(Q1) among the quasi-martingale conditions. The predictive CLT, which is what licenses the asymptotic credible bands of Theorem~\ref{th:ascondmult}, requires substantially more: at the supported rate $\gamma$, the drift conditions (C2)--(C4) on the signed side or (Q2) on the quasi-martingale side, together with~(R) on the inflated outer product of residuals, which is common to both sets.

(C1)--(C4) require the conditional drift $b_k$ to decay to zero; (R) requires the inflated outer product of residuals $R_n$ to converge to a positive finite limit. The role of (R) is asymmetric in $\gamma$: it pins $\gamma$ at the unique value (if any) at which $n^\gamma\sum_{k>n}\Delta_k^2$ stabilises, since $\gamma$ smaller than that value drives the limit to zero and $\gamma$ larger drives it to infinity. (C2)--(C4), by contrast, give upper bounds: each fails as $\gamma$ becomes too large. The supported rate is the largest $\gamma$ on the grid at which all five conditions visibly hold on every rollout.

The drift halves of the two sets, (C1)--(C4) on the signed side and (Q2) on the quasi-martingale side, differ in two respects ((R) is common to both, so it is excluded from the comparison). First, (C1)--(C4) are \emph{pathwise}: they are stated as $\PP$-almost-sure limits along a single sampled sequence $Z_1, Z_2, \ldots$ (a \emph{rollout}), so each rollout can be inspected on its own. (Q2), by contrast, takes an outer expectation $\EE$ over the law of the sequence, so it constrains the average of $\lvert b_k\rvert$ across rollouts and says nothing about any individual rollout. Second, (C1) and (C2) use signed sums of $b_k$, whereas~(Q2) sums $\lvert b_k\rvert$; signed sums are weaker because cancellation between positive and negative drifts can produce convergence even when absolute summability fails. These differences make (C1)--(C4) strictly weaker than~(Q2).

The size of the gap can be made explicit under a single magnitude assumption: that $\lvert b_k\rvert$ has typical scale $k^{-\beta}$, $\beta>1$, both on average ($\EE\lvert b_k\rvert\asymp k^{-\beta}$) and pathwise along a typical rollout ($\lvert b_k\rvert\asymp k^{-\beta}$). Both rates follow from this one ansatz, but the two sides apply it through different summation operations. (Q2) sums absolute values, and a direct calculation gives $\gamma_Q^* := \sup\{\gamma: (Q2)\text{ holds}\} = 2\beta-2$. (C2) sums signed $b_k$ instead, and if the signs along a rollout are approximately random its partial sums enjoy a $\sqrt n$ cancellation, so (C2) holds for $\gamma<2\beta-1$ where without cancellation the threshold would be $2\beta-2$. (C3) and (C4) involve no cancellation and, under the same ansatz, hold for $\gamma<2\beta$ and $\gamma<2\beta-1$ respectively. The binding pathwise rate across (C2)--(C4) is therefore $2\beta-1$, a $1$-unit improvement over (Q2) attributable to the cancellation available when the drifts in (C2) are summed with their signs.

\subsection{Setup}
\label{app:theorycheck-setup}

\paragraph{Meta-learning.}
The model is a minimal PFN \citep{muller22pfns} consisting of a standard \texttt{nn.TransformerEncoder} with $d_{\mathrm{model}}=64$, $L=2$ layers, $H=4$ heads, and FFN width $128$. No positional encoding is used, so the network is exactly permutation-equivariant in the prefix. Each training step samples a random prefix length $K$, feeds the prefix $Y_{1:K}$ as context, and applies BCE loss on the query tokens at positions $>K$, following \citet{muller22pfns}. Meta-learning sequences are drawn from the prior predictive of a fixed Beta-Bernoulli model:
\[
    \tilde\theta\sim\mathrm{Beta}(1,1),
    \qquad
    Y_1,\dots,Y_{1024}\mid\tilde\theta\stackrel{\mathrm{iid}}{\sim}\mathrm{Bernoulli}(\tilde\theta),
\]
with the meta-learning sequence length set to $T=1024$. There are no covariates ($\mathcal X$ is a singleton), and $\mathcal Y=\{0,1\}$ is trivially compact, so conditions~(i) and~(ii) of Theorem~\ref{th:convergence} are satisfied automatically for the trained BFT regardless of how rollouts are generated; only the signed tail-sum condition~(iii) of Theorem~\ref{th:convergence}, the rate-weighted conditions on the conditional drift $b_k$ of Theorem~\ref{th:ascondmult}, and condition~(iii) of Theorem~\ref{th:ascondmult} on the inflated outer product of residuals $R_n$ are non-trivial to check. We compare two meta-learning budgets that differ only in the number of training steps: a \textbf{600-step run} ($9{,}600$ tasks seen, warmup $50$) and a \textbf{50k-step run} ($800{,}000$ tasks seen, warmup $1{,}000$); both use batch size $16$ and AdamW with a cosine schedule.

\paragraph{Predictive rules evaluated.}
We run the same diagnostic pipeline on \emph{ten} predictive rules, grouped into three families.
\begin{itemize}
  \item \textbf{Bayes oracle} (1 predictive rule): the exact posterior predictive
  \[
    g_k^{\mathrm{oracle}} \;=\; \frac{1+\sum_{i=1}^k Y_i}{2+k}.
  \]
  $g_k^{\mathrm{oracle}}$ is an exact martingale, so $b_k\equiv 0$ analytically and every condition we test holds trivially. The oracle curves in the figures trace the floating-point roundoff floor amplified by rate weighting.
  \item \textbf{Corrupted oracles} (8 predictive rules), a calibration family with known asymptotics:
  \[
    g_k^{\mathrm{corr}} \;=\; \sigma\!\Big(\mathrm{logit}\,g_k^{\mathrm{oracle}} \;+\; \eta_k\Big),
    \qquad
    \eta_k \;=\; c_k\,\xi_k,
  \]
  where $\xi_k=\xi(Y_{1:k},k)$ is a deterministic standard-normal hash of the prefix bits and the step index. Determinism in $(Y_{1:k},k)$ preserves the closed-form two-point evaluation of $b_k$ developed in Appendix~\ref{app:theorycheck-probes} (the inner $\EE$ over $Y_k$ still reduces to two evaluations of $g_k^{\mathrm{corr}}$). We probe two corruption schedules:
  \begin{itemize}
    \item \emph{Constant noise}, $c_k=\varepsilon$, at $\varepsilon\in\{10^{-3},10^{-2},10^{-1}\}$ (3 predictive rules): asymptotically a non-vanishing logit-scale perturbation, so $\lvert b_n\rvert=O(1)$.
    \item \emph{Power-law decay}, $c_k=\varepsilon\,k^{-p}$, at $\varepsilon=0.5$ and $p\in\{0.25,0.5,1.0,1.5,2.0\}$ (5 predictive rules): injects a step of magnitude $O(n^{-p})$, so $\lvert b_n\rvert=O(n^{-p})$.
  \end{itemize}
  Since the logit perturbation is of the same order as the probability-scale step it induces when $g_k^{\mathrm{oracle}}$ is bounded away from $0$ and $1$, the predicted QM decay exponent is $0$ (noise) or $p$ (decay). This family serves as a pipeline check: the fitted $\widehat\beta$ on $\widehat{\EE}\lvert b_n\rvert$ should recover the injected rate.
  \item \textbf{Meta-learned BFTs} (2 predictive rules), the 600-step and 50k-step checkpoints described above.
\end{itemize}

\subsection{Diagnostic computations}
\label{app:theorycheck-probes}

\paragraph{Computing $b_k$ in closed form.}
\label{app:theorycheck-compute}
Because the setting is binary and covariate-free, the conditional drift can be computed exactly without Monte Carlo. Write
\[
  g_k^{(y)} \;:=\; \PP(Y_{k+1}=1\mid Y_{1:k-1},\, Y_k = y),
  \qquad y\in\{0,1\},
\]
for the prediction when the $k$-th observation is \emph{set} to $y$ (i.e.\ the BFT is queried on the augmented prefix $Y_{1:k-1},y$). Because $Y_k\in\{0,1\}$ and $\PP(Y_k=1\mid Y_{1:k-1})=g_{k-1}$, the conditional drift reduces to a two-point average:
\begin{equation}
  b_k
  \;=\; g_{k-1}\,g_k^{(1)} \;+\; (1-g_{k-1})\,g_k^{(0)} \;-\; g_{k-1}.
  \label{eq:bb-two-point}
\end{equation}
A single evaluation of $b_k$ therefore costs exactly three forward passes of the predictive rule (one for $g_{k-1}$ and one each for $g_k^{(0)}$ and $g_k^{(1)}$), with no Monte Carlo average over $Y_k$. The resulting quantity is a deterministic, pathwise function of the prefix $Y_{1:k-1}$.

\paragraph{Rollouts.}
Conditions~(C1)--(C4) and (R) are pathwise $\PP$-a.s.\ statements with respect to the law $\PP$ of the sequence the predictive rule is applied to, so they must be checked against each predictive rule's \emph{own} induced law; the prior predictive used for BFT meta-learning plays no role here. We therefore run the same diagnostic for every one of the ten predictive rules in Appendix~\ref{app:theorycheck-setup}. For each predictive rule $g_k$ we generate $16$ independent rollouts, each of length $N=10{,}001$. Each rollout $Y_{1:N}$ is produced iteratively by sampling
\[
    Y_n \;\sim\; \mathrm{Bernoulli}(g_{n-1}),\qquad n=2,3,\dots,N.
\]
The first token $Y_1$ is drawn from $\mathrm{Bernoulli}(g_0)$, where $g_0$ is the predictive rule's step-$0$ output ($g_0^{\mathrm{oracle}}=g_0^{\mathrm{corr}}=\tfrac12$ for the oracle and corrupted oracles; a uniform $\mathrm{Bernoulli}(\tfrac12)$ seed for the BFTs, since the BFT's forward pass requires at least one context token). For the oracle, the induced law \emph{is} the Beta-Bernoulli prior predictive by construction; for the corrupted oracles and the BFTs, the induced law is a proper mixture that is not available in closed form. Since the BFTs are meta-learned at $T=1024$, indices $n>1024$ are out-of-distribution for the BFTs; we still take $N=10{,}001$ because the conditions are statements as $n\to\infty$.

\paragraph{Evaluating (C1)--(C4).}
These conditions are pathwise ($\PP$-a.s.) statements about \emph{individual} rollouts, so we inspect the per-rollout quantities directly. Probe positions run over $n\in\{2,\dots,N-1\}$. We evaluate~\eqref{eq:bb-two-point} at every probe position on each rollout, and the realised increment $\Delta_n^{(r)} = g_n^{(r)} - g_{n-1}^{(r)}$ is read off the same forward passes. The tail sums in (C1), (C2), and (C4) are truncated at the largest available index $N-1$:
\[
    \sum_{k\ge n} b_k^{(r)} \;\approx\; \sum_{k=n}^{N-1} b_k^{(r)},
    \qquad
    \sum_{k\ge n} (b_k^{(r)})^2 \;\approx\; \sum_{k=n}^{N-1} (b_k^{(r)})^2.
\]

\paragraph{Evaluating (Q1)--(Q2).}
Both conditions involve an outer expectation over rollouts, so we evaluate them via the cross-rollout mean $\widehat{\EE}\lvert b_n\rvert := \tfrac{1}{16}\sum_r \lvert b_n^{(r)}\rvert$ and fit a power law $\widehat{\EE}\lvert b_n\rvert\approx C n^{-\widehat\beta}$ by OLS on log--log axes over $n\in[10,2000]$. Under the ansatz $\EE\lvert b_n\rvert\asymp n^{-\beta}$, (Q1) requires $\beta>1$ and (Q2) at rate $\gamma$ requires $\beta>1+\gamma/2$, so
\[
    \gamma_Q^* \;\approx\; \min\!\bigl(1,\; \max(0,\, 2(\widehat\beta - 1))\bigr)
\]
is the largest $\gamma$ at which (Q2) holds, and (Q1) corresponds to $\widehat\beta>1$. The OLS confidence intervals are narrow ($\pm 0.002$ and $\pm 0.008$ for the 600-step and 50k-step BFTs) because the OLS standard errors ignore the autocorrelation of the log--log residuals; the point estimates should be read as accurate to $\sim\pm 0.03$, an estimate we calibrate on the corrupted-oracle family in Appendix~\ref{app:theorycheck-findings}. As an auxiliary visual check we also compute the cumulative overlays
\[
    U_\gamma(n) \;=\; \sum_{m\le n} m^{\gamma/2}\,\widehat{\EE}\lvert b_m\rvert,
    \qquad \gamma\in\{0, 0.25, 0.5, 0.75, 1.0\},
\]
which flatten iff~(Q2) holds at that $\gamma$.

\paragraph{Evaluating (R).}
For Beta-Bernoulli the oracle increment satisfies $\Delta_k = (Y_k-g_{k-1})/(k+2)$ exactly, so at the supported rate $\gamma=1$
\begin{equation*}
  n\!\!\sum_{k=n+1}^{N-1}\!\!\Delta_k^2 \;\approx\; \tilde\theta(1-\tilde\theta)\,\bigl(1 - n/(N-1)\bigr),
\end{equation*}
where $\tilde\theta$ is the de Finetti limit of $g_k$ on that rollout; $\tilde\theta(1-\tilde\theta)$ is the limit $V$ in condition~(iii) of Theorem~\ref{th:ascondmult} at the supported rate. The relation $V=\tilde\theta(1-\tilde\theta)$ is specific to $\gamma=1$: at $\gamma<1$ the truncated rate-weighted sum decays to zero and at $\gamma>1$ it diverges. We evaluate the rate-weighted truncated sum at a single probe index $n_*=2000$ (well inside the regime $1\ll n_*\ll N$) and divide by the $\gamma=1$ truncation factor $\bar C_1 = 1 - n_*/(N-1)\approx 0.8$, applied uniformly at every $\gamma$:
\begin{equation*}
  \hat V_r(\gamma) \;:=\; \bar C_1^{-1} \cdot n_*^\gamma \sum_{k=n_*+1}^{N-1}(\Delta_k^{(r)})^2.
\end{equation*}
Define $\hat\theta_r=N^{-1}\sum_{i=1}^N Y_i^{(r)}$. We display one mini-scatter of $(\hat\theta_r(1-\hat\theta_r),\,\hat V_r(\gamma))$ per $\gamma$ row, with the diagonal $y=x$ as the $\gamma=1$ oracle reference. Under the oracle, the $\gamma=1$ row sits on $y=x$; $\gamma<1$ rows show $\hat V_r(\gamma)$ collapsing towards $0$; $\gamma>1$ rows would shoot above the diagonal (not in our test grid). The supported rate is the row in which the dots align with the diagonal; systematic departures at that row (over- or under-spread) diagnose calibration failure on individual rollouts.

\paragraph{Figure layout.}
The diagnostic is run on all ten predictive rules of Appendix~\ref{app:theorycheck-setup}, with figures grouped by family of predictive rules: the two meta-learned BFTs (Figures~\ref{fig:bb-600} and~\ref{fig:bb-50k}), the Bayes-oracle precision floor (Figure~\ref{fig:bb-oracle}), and the calibration family of eight rules (Figures~\ref{fig:bb-corrupt-noise1e-3}--\ref{fig:bb-corrupt-decay-p20}). Each figure has two subpanels:
\begin{itemize}
    \item the \emph{signed-condition panel} comprises (i) a top row for (C1) (no $\gamma$, full width); and (ii) a $10\times 4$ grid sweeping $\gamma\in\{0.1, 0.2, \dots, 1.0\}$ across (C2)--(C4) on log axes and the (R) scatter with truncation correction on linear axes. The (R) cell in each row plots the $16$ per-rollout pairs $(\hat\theta_r(1-\hat\theta_r),\,\hat V_r(\gamma))$ against the diagonal $y=x$.
    \item the \emph{(Q1)+(Q2) panel} shows $\widehat{\EE}\lvert b_n\rvert$ with the OLS power-law fit over $n\in[10,2000]$ (read (Q1) off $\widehat\beta>1$, $\gamma_Q^*$ for (Q2) off $\min\!\bigl(1, \max(0, 2(\widehat\beta-1))\bigr)$), and the $U_\gamma(n)$ overlay (each $U_\gamma$ curve solid or dashed according to whether $\widehat\beta>1+\gamma/2$).
\end{itemize}
We deliberately do \emph{not} automate a pass/fail test on the signed-condition panels: any threshold-based rule interacts badly with the early-$n$ transient of (C2), which is non-monotone and whose peak can occur at $n\ll 100$ on a subset of rollouts. We instead read the drift columns by eye, as the largest row at which (C2)--(C4) visibly decay to zero on every rollout. (C2)--(C4) amplify harder as $\gamma$ grows, so failures appear first at the bottom rows. The (R) column is read row-by-row: the dots should cluster along $y=x$ in every row when the predictive rule mimics the prior predictive, with rate-invariant departures flagging calibration failures.

\textbf{Finite-$N$ caveat for the drift columns.} (C1), (C2), and~(C4) all involve tail sums truncated at index $N-1$; as $n\to N$ these truncated sums collapse to zero by construction, so those panels are only trustworthy for $n\ll N$. We therefore restrict the $x$-axis on these three panels to $n\le 5\times 10^3$. Condition~(C3) is not affected and is shown over the full probe range $n\in[2,N-1]$. The (R) scatter cells use the explicit truncation correction $1/\bar C_1$ described above.

\paragraph{Computational cost.}
Each diagnostic run costs $16\,(N-1)\approx 1.6\times 10^5$ forward passes for rollout generation, plus $3\cdot 16\,(N-2)\approx 4.8\times 10^5$ forward passes for probing: three forward passes per evaluation of~\eqref{eq:bb-two-point} (one for $g_{k-1}$, one each for $g_k^{(0)}$ and $g_k^{(1)}$), evaluated at every $n\in\{2,\dots,N-1\}$ on each of the $16$ rollouts. Both BFT runs took approximately $25$~minutes each on a single NVIDIA H100 SXM GPU; the 50k-step meta-learning took a further $\sim 5$~minutes on the same device.

\subsection{Findings}
\label{app:theorycheck-findings}

\paragraph{Meta-learned BFTs.}
We discuss each condition in turn, referring to the 600-step BFT (Figures~\ref{fig:bb-600-signed},~\ref{fig:bb-600-qm}) and the 50k-step BFT (Figures~\ref{fig:bb-50k-signed},~\ref{fig:bb-50k-qm}) side-by-side.

\emph{(C1)}, top panel, no $\gamma$. On both BFTs the tail sum $\bigl\lvert\sum_{k=n}^{N-1} b_k^{(r)}\bigr\rvert$ decays to zero on every rollout.

\emph{(C2)}. Rate-weighted tail residuals decay over the displayed window ($n\in[2000,5000]$) at every $\gamma\le 1$ on both BFTs; the per-rollout band is tighter on the 50k-step BFT, and in every row the values at the start of the displayed window lie within an order of magnitude of those at the largest displayed $n$.

\emph{(C3)}. The pointwise $n^{\gamma/2}\lvert b_n^{(r)}\rvert$ visibly decays at every $\gamma\le 1$ on both BFTs. This is the only column unaffected by the finite-$N$ truncation.

\emph{(C4)}. Rate-weighted tail sums of squares decay at every $\gamma\le 1$ on both BFTs.

\emph{(Q1) and (Q2)}. The OLS fits give $\widehat\beta = \sixhundredbeta$ on the 600-step BFT (mean $\lvert b_n\rvert \approx \sixhundredmeanbn$) and $\widehat\beta = \fiftykbeta$ on the 50k-step BFT (mean $\lvert b_n\rvert \approx \fiftykmeanbn$), giving
\[
    \gamma_Q^*(\text{600-step}) \;\approx\; \sixhundredgq,
    \qquad
    \gamma_Q^*(\text{50k-step}) \;\approx\; \fiftykgq.
\]
The fitted exponent did not improve with training ($\widehat\beta=\fiftykbeta$ at 50k steps against $\sixhundredbeta$ at 600 steps), so the improvement from training is visible in the (R) calibration below and in the $\gamma$-free diagnostic of Appendix~\ref{app:variance-only}; the QM fit alone does not show it.
Both BFTs clear $\widehat\beta>1$, so (Q1) holds on each and the predictive rule converges $\PP$-a.s., a result of independent interest, since (Q1) suffices for the sampling arm of BPI (predictive Monte Carlo). Neither reaches $\widehat\beta>3/2$, so (Q2) at $\gamma=1$ fails on each and the QM route alone does not support a $\sqrt n$-rate predictive CLT.

\emph{(R)}. The 50k-step BFT scatter sweeps cleanly across the $\gamma$ rows: dots cluster near $y=0$ at $\gamma\le 0.5$, gradually lift towards the diagonal between $\gamma=0.6$ and $\gamma=0.9$, and sit on the diagonal at $\gamma=1$, indicating $\hat V_r(1)\approx \hat\theta_r(1-\hat\theta_r)$ on every rollout and confirming that the BFT's increments mimic the Beta-Bernoulli oracle's at the supported rate $\gamma=1$. The 600-step BFT shows the same rate-sweep pattern but with the $\gamma=1$ dots sitting systematically \emph{below} the diagonal on boundary rollouts (those with $\hat\theta_r$ near $0$ or $1$): $\hat V_r(1)\approx 0.3$--$0.5$ times the theoretical $\hat\theta_r(1-\hat\theta_r)$; in the centre of $\theta$-space ($\hat\theta_r\in[0.4, 0.7]$) the points recover the diagonal. The under-trained BFT is over-confident on boundary rollouts: its $g_k$ converges faster than the oracle's, shrinking $\Delta_k$ below the PPD-predicted scale. $\hat V_r(1)$ still remains strictly positive on every rollout, so (R) holds at $\gamma=1$ even when the BFT quantitatively under-spreads.

\smallskip
Taken together, the signed conditions (C1)--(C4) and~(R) are consistent with Theorem~\ref{th:ascondmult} at $\gamma=1$ on both BFTs, though the $\gamma$-free diagnostic (Appendix~\ref{app:variance-only}) shows residual drift bias on the 600-step BFT at accessible $n$. The QM route, which would require (Q2) and~(R) at the \emph{same} $\gamma$, supports the CLT at no $\gamma$: the supported ranges for (Q2), $\gamma\le \sixhundredgq$ on the 600-step BFT and $\gamma\le \fiftykgq$ on the 50k-step BFT, are disjoint from the range $\gamma=1$ for (R). This is the structural $1$-unit gap (signed rate $2\beta-1$ vs.\ QM rate $2\beta-2$ under the magnitude ansatz, derived above) realised empirically. The signed conditions are \emph{sufficient}; our diagnostics test for \emph{gross} violations and do not constitute formal certification. A diagnostic curve that does not visibly decay over the probe range $n\in[2, N-1]$ is ambiguous: the underlying condition may fail asymptotically, or it may decay only at scales beyond our compute budget. We cannot distinguish these two without extending the rollouts indefinitely.

\paragraph{Bayes oracle.}
On Figure~\ref{fig:bb-oracle}, the analytic predictive is an exact martingale, so $b_k\equiv 0$ pathwise.

\emph{(C1)--(C4)}. Sit at the floating-point random-walk floor of order $\epsilon\sqrt n\sim 10^{-15}$ at every $\gamma$.

\emph{(Q1) and (Q2)}. Hold vacuously at every $\gamma$.

\emph{(R)}. The oracle's $\Delta_k=(Y_k-g_{k-1})/(k+2)$ identity gives $n\sum_{k>n}\Delta_k^2\to\tilde\theta(1-\tilde\theta)$ exactly, and the truncation correction makes the $\gamma=1$ scatter sit on $y=x$ to within numerical precision (per-rollout ratios $\hat V_r(1)/\hat\theta_r(1-\hat\theta_r)\in[0.97,1.03]$). The $\gamma<1$ rows show $\hat V_r(\gamma)$ collapsing towards zero as $\gamma$ decreases, in line with the rate analysis. This calibrates the diagnostic end-to-end and serves as the reference shape against which the BFT scatters are read.

\smallskip
The oracle's role is dual: it sets the noise floor below which any signal is numerical artefact, and it gives the reference (R) shape against which the BFT panels are read.

\paragraph{Corrupted oracles.}
The corruption schedule injects a known $\lvert b_k\rvert$ decay (Figures~\ref{fig:bb-corrupt-noise1e-3}--\ref{fig:bb-corrupt-decay-p20}), which calibrates both the (Q2) fit and the per-panel reading.

\emph{(C1)--(C4)}. The rate-weighted panels track the injected decay: constant-noise fails at every $\gamma>0$; $p=0.5$ is marginal at $\gamma=1$ on (C3) and fails the tail-residual columns; $p\ge 1.5$ passes every rate-weighted panel at every tested $\gamma$. The progression is monotone in the injected exponent.

\emph{(Q1) and (Q2)}. Fitted exponents on $\widehat{\EE}\lvert b_n\rvert$ are $\widehat\beta=-0.018,-0.018,-0.021$ at $\varepsilon=10^{-3},10^{-2},10^{-1}$ (predicted $0$) and $\widehat\beta=0.228, 0.473, 0.993, 1.490, 1.989$ at $p=0.25, 0.5, 1.0, 1.5, 2.0$ (predicted $p$); $\widehat\beta$ matches its prediction to within $0.03$ across all eight, so the (Q2) pipeline is trustworthy. (Fitted OLS CIs of $\pm 0.022$ are narrow because the OLS standard errors ignore the autocorrelation of the log--log residuals; the $\pm 0.03$ spread is the more honest accuracy estimate.) (Q1) requires $\widehat\beta>1$, which holds only on the decay schedule with $p>1$ ($p=1$ sits on the boundary and its fit $\widehat\beta=0.993$ fails the criterion); the three constant-noise corruptions and the $p\le 1$ decays violate (Q1).

\emph{(R)}. The $\gamma=1$ scatter cell tracks the injected corruption monotonically. Constant noise blows up: per-rollout median $\hat V_r(1)/\hat\theta_r(1-\hat\theta_r)\approx 10$ at $\varepsilon=10^{-3}$, $\approx 10^{3}$ at $\varepsilon=10^{-2}$, and $\approx 10^{5}$ at $\varepsilon=10^{-1}$, with the dots sitting far above the diagonal (the $y$-axis auto-scales to accommodate them). Power-law decay $p<1$ also overshoots ($\approx 4\times 10^{4}$ at $p=0.25$, $\approx 5\times 10^{2}$ at $p=0.5$); $p=1.0$ is borderline (median ratio $1.11$); $p\ge 1.5$ lies on the diagonal within the same numerical-precision band as the oracle. Lower-$\gamma$ rows show progressively smaller $\hat V_r(\gamma)$ in cases where the increments do decay (decay $p\ge 1$); for the constant-noise and slow-decay cases the increments do not decay, so $\hat V_r(\gamma)$ is large at every $\gamma$ row. Both failure directions and successful recovery are exhibited, so the (R) scatter is calibrated end-to-end.

\begin{figure}[p]
  \centering
  \begin{subfigure}[t]{\linewidth}
    \centering
    \includegraphics[width=\linewidth,height=0.70\textheight,keepaspectratio]{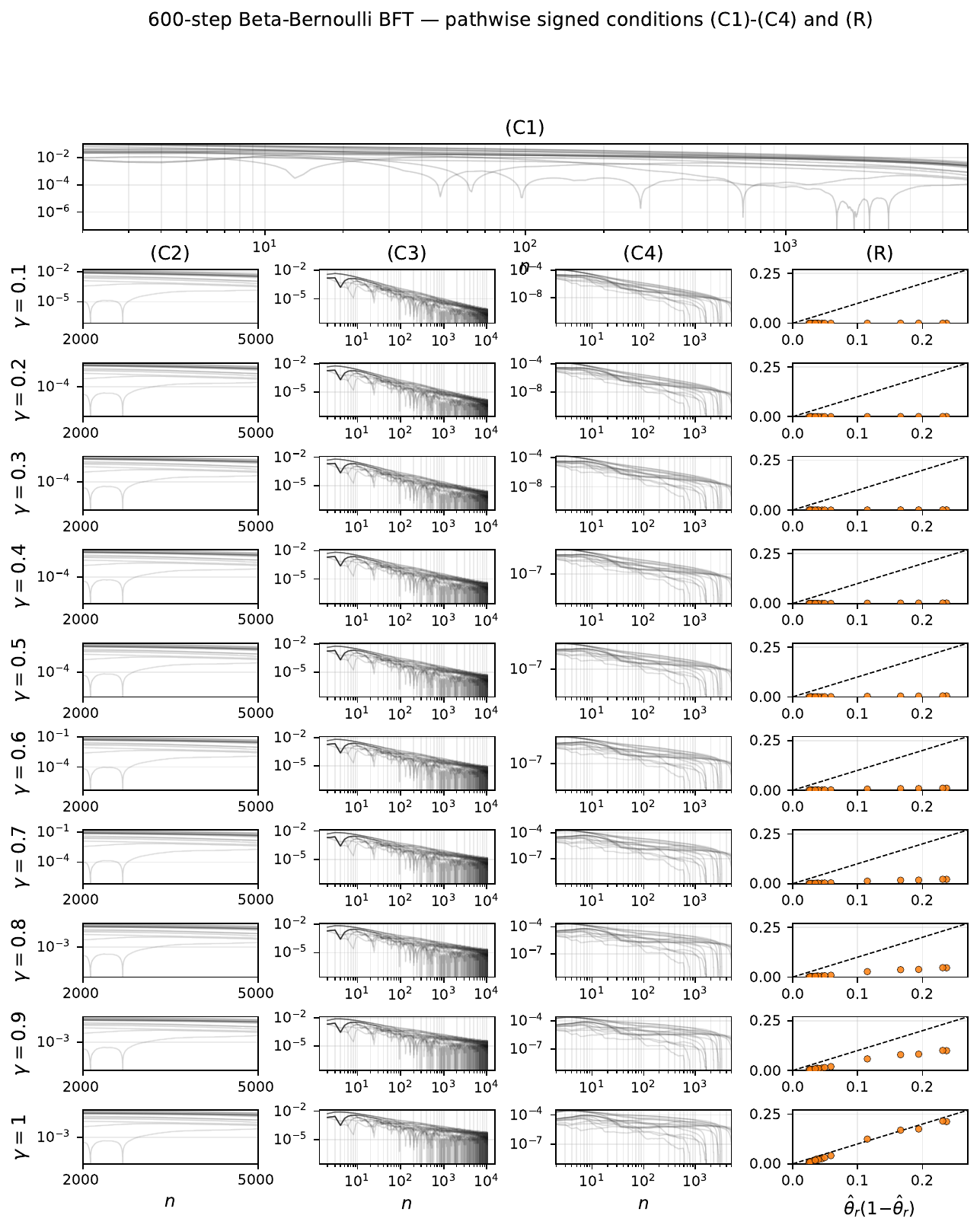}
    \caption{\textbf{Signed conditions, 600-step BFT.} Per-rollout trajectories under the BFT-induced law: (C1) on top, then a row per $\gamma\in\{0.1,\dots,1.0\}$ across (C2), (C3), (C4), and the (R) scatter (formulas in panel titles; estimator $\hat V_r(\gamma)$ defined in ``Evaluating (R)'' above). (C1), (C4) restricted to $n\le 5\times 10^3$; (C2) shown for $n\in[2\times 10^3, 5\times 10^3]$; (C3) over the full probe range. The supported rate is the largest row at which (C2)--(C4) visibly decay; the (R) column should cluster on the diagonal at $\gamma=1$ and collapse towards zero at $\gamma<1$.}
    \label{fig:bb-600-signed}
  \end{subfigure}\\[1.5ex]
  \begin{subfigure}[t]{\linewidth}
    \centering
    \includegraphics[width=\linewidth]{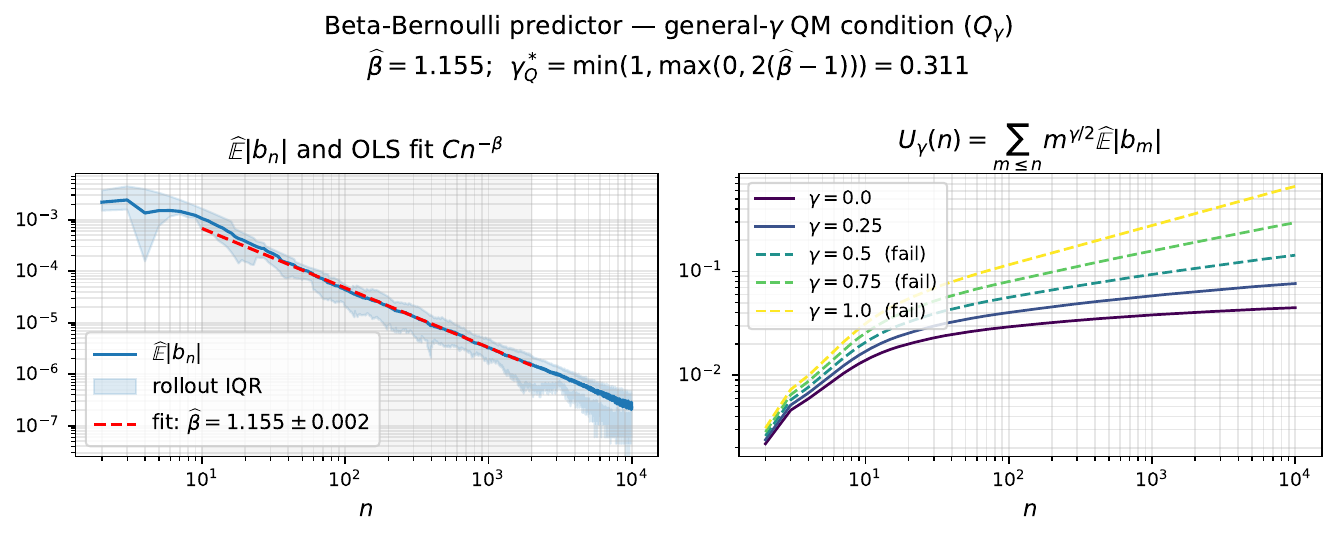}
    \caption{\textbf{$(Q1){+}(Q2)$, 600-step BFT.} Left: cross-rollout mean $\widehat{\EE}\lvert b_n\rvert$ with OLS power-law fit giving $\widehat\beta=\sixhundredbeta$, hence $\gamma_Q^*=\min(1, \max(0, 2(\widehat\beta-1)))\approx\sixhundredgq$. Right: cumulative overlay $U_\gamma(n)$ for $\gamma\in\{0, 0.25, 0.5, 0.75, 1.0\}$; solid $=$ bounded under the fit, dashed $=$ diverges.}
    \label{fig:bb-600-qm}
  \end{subfigure}
  \caption{\textbf{Diagnostics on the 600-step Beta-Bernoulli BFT; rollouts sampled from the BFT's own predictive rule.}}
  \label{fig:bb-600}
\end{figure}

\begin{figure}[p]
  \centering
  \begin{subfigure}[t]{\linewidth}
    \centering
    \includegraphics[width=\linewidth,height=0.70\textheight,keepaspectratio]{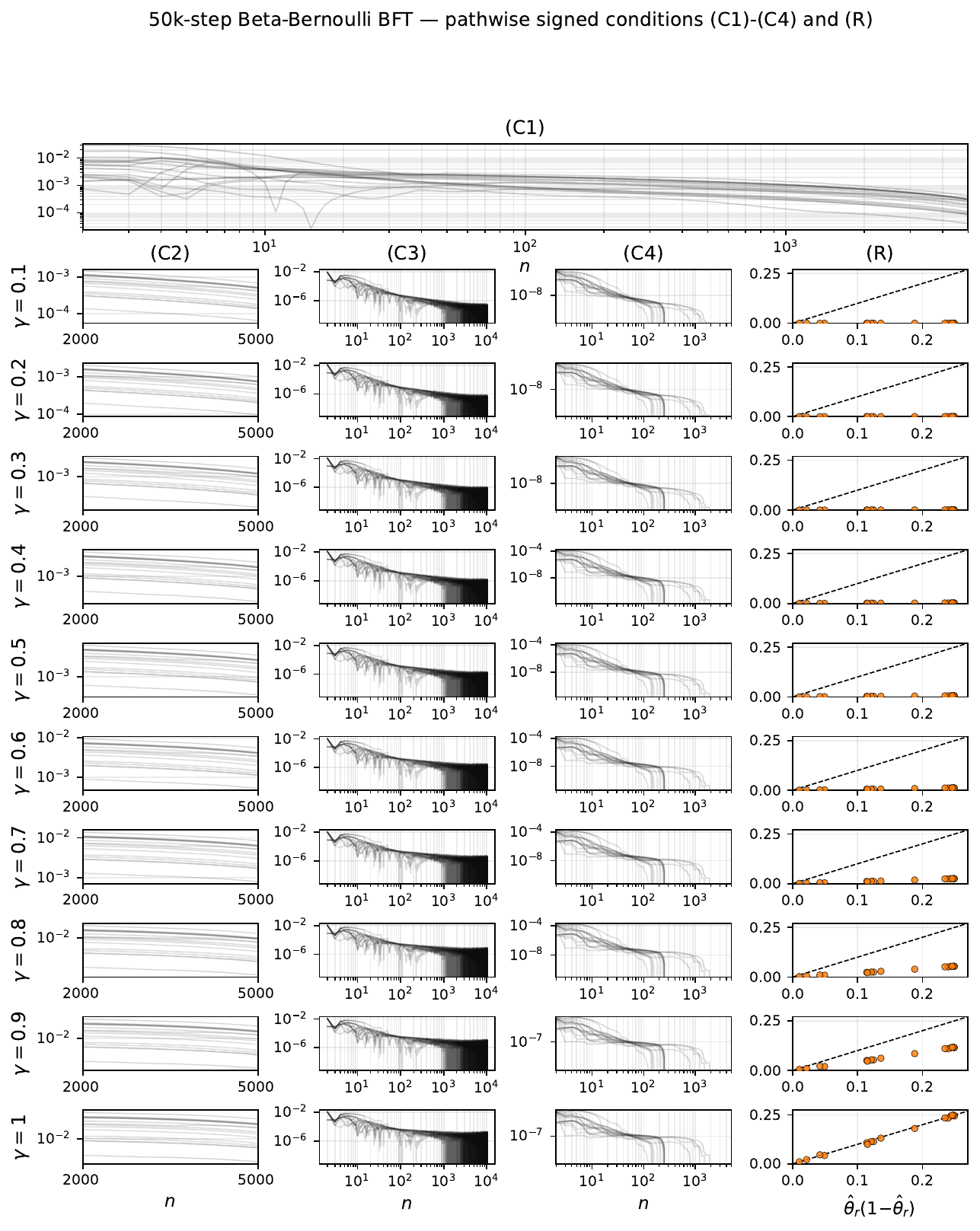}
    \caption{\textbf{Signed conditions, 50k-step BFT.} Same panel layout as Figure~\ref{fig:bb-600-signed}. Trajectories are visibly tighter than the 600-step checkpoint, consistent with the longer meta-learning.}
    \label{fig:bb-50k-signed}
  \end{subfigure}\\[1.5ex]
  \begin{subfigure}[t]{\linewidth}
    \centering
    \includegraphics[width=\linewidth]{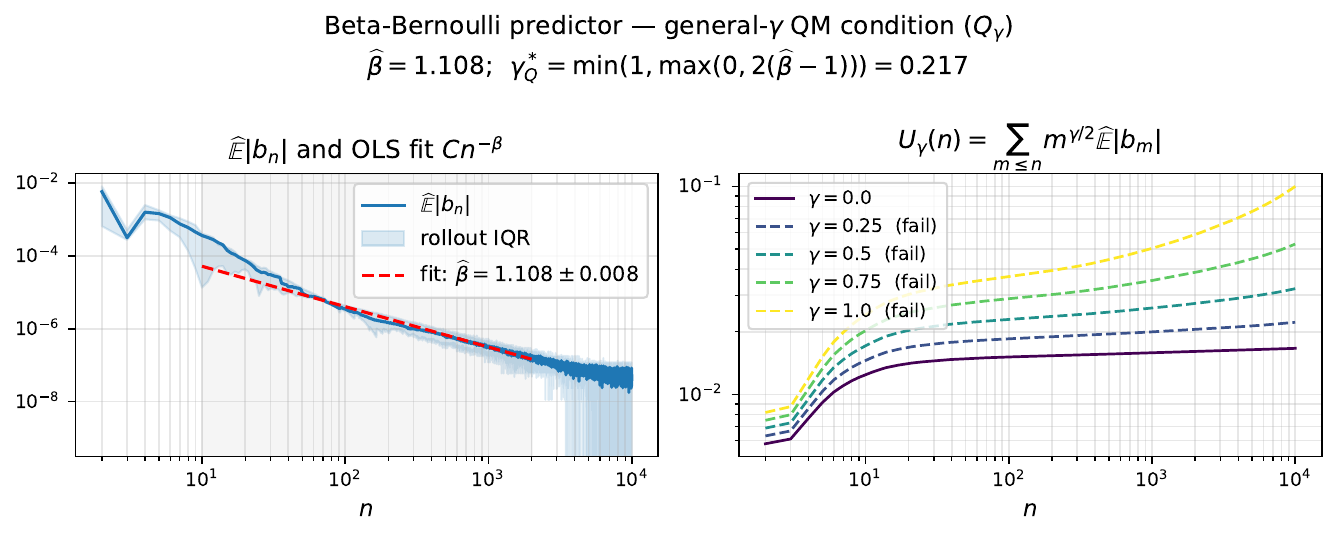}
    \caption{\textbf{$(Q1){+}(Q2)$, 50k-step BFT.} Same panel layout as Figure~\ref{fig:bb-600-qm}. The fitted decay is $\widehat\beta\approx\fiftykbeta$ (OLS on $n\in[10,2000]$), so $\gamma_Q^*\approx\fiftykgq$. $\widehat\beta>1$ clears the $\gamma=0$ special case but misses $\gamma=1$, reflected in the dashed $U_\gamma$ curves at higher $\gamma$.}
    \label{fig:bb-50k-qm}
  \end{subfigure}
  \caption{\textbf{Diagnostics on the 50k-step Beta-Bernoulli BFT; rollouts sampled from the BFT's own predictive rule.}}
  \label{fig:bb-50k}
\end{figure}

\begin{figure}[p]
  \centering
  \begin{subfigure}[t]{\linewidth}
    \centering
    \includegraphics[width=\linewidth,height=0.70\textheight,keepaspectratio]{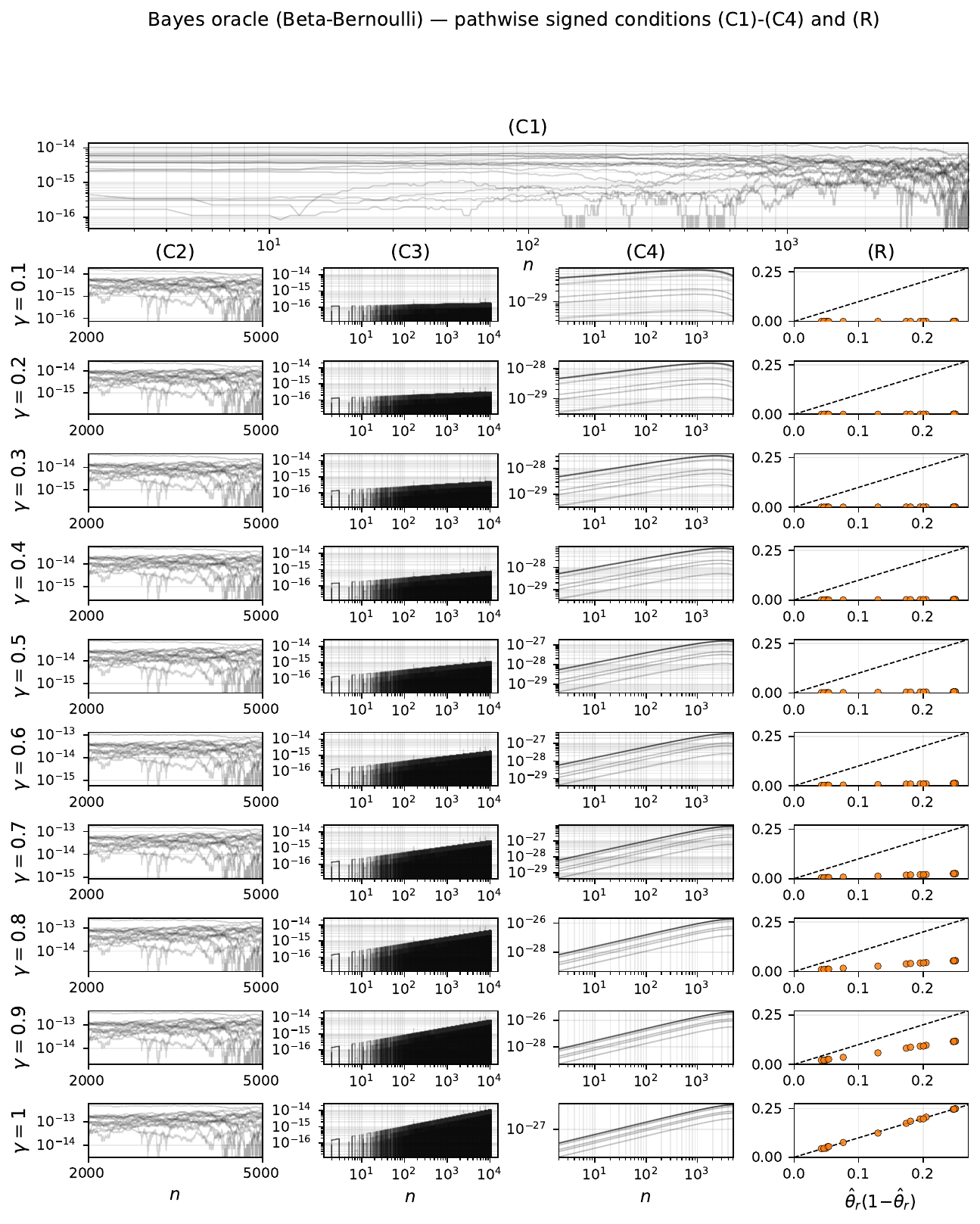}
    \caption{\textbf{Signed conditions, Bayes oracle.} Same panel layout as Figure~\ref{fig:bb-600-signed}, under the oracle-induced law. Since $b_n\equiv 0$ analytically, the (C1)--(C4) curves trace floating-point roundoff.}
    \label{fig:bb-oracle-signed}
  \end{subfigure}\\[1.5ex]
  \begin{subfigure}[t]{\linewidth}
    \centering
    \includegraphics[width=\linewidth]{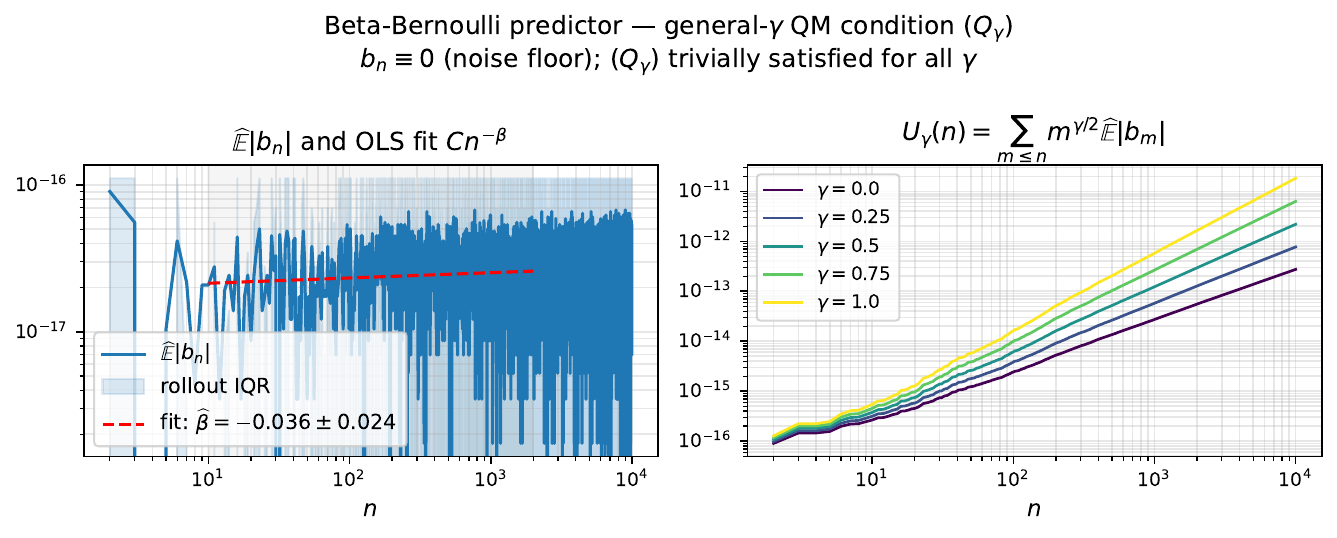}
    \caption{\textbf{$(Q1){+}(Q2)$, Bayes oracle.} Same panel layout as Figure~\ref{fig:bb-600-qm}. The oracle satisfies~(Q2) at every $\gamma$ trivially since $b_n\equiv 0$ analytically; the fitted log--log slope only reflects the power-law regression trying to fit the floating-point noise floor and is not a meaningful decay-rate estimate. Every $U_\gamma(n)$ curve sits many orders of magnitude below the BFT curves.}
    \label{fig:bb-oracle-qm}
  \end{subfigure}
  \caption{\textbf{Diagnostics on the Bayes oracle; rollouts sampled from the oracle's own predictive rule.}}
  \label{fig:bb-oracle}
\end{figure}

\begin{figure}[p]
  \centering
  \begin{subfigure}[t]{\linewidth}
    \centering
    \includegraphics[width=\linewidth,height=0.70\textheight,keepaspectratio]{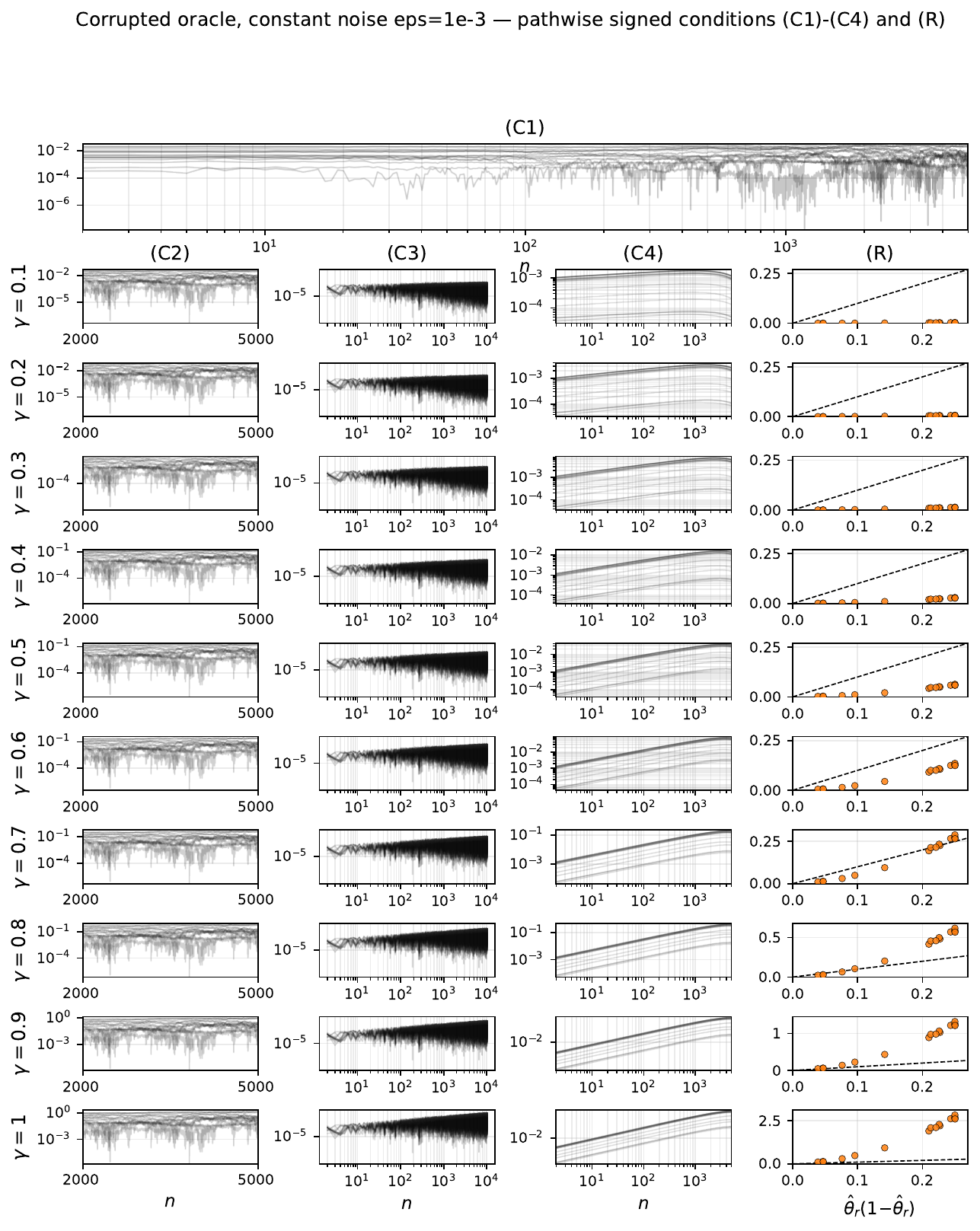}
    \caption{\textbf{Signed conditions, corrupted oracle (constant noise, $\varepsilon=10^{-3}$).} Same panel layout as Figure~\ref{fig:bb-600-signed}, under this predictive rule's own induced law. As expected for $\lvert b_n\rvert=O(1)$, no rate-weighted panel decays at any tested $\gamma$.}
    \label{fig:bb-corrupt-noise1e-3-signed}
  \end{subfigure}\\[1.5ex]
  \begin{subfigure}[t]{\linewidth}
    \centering
    \includegraphics[width=\linewidth]{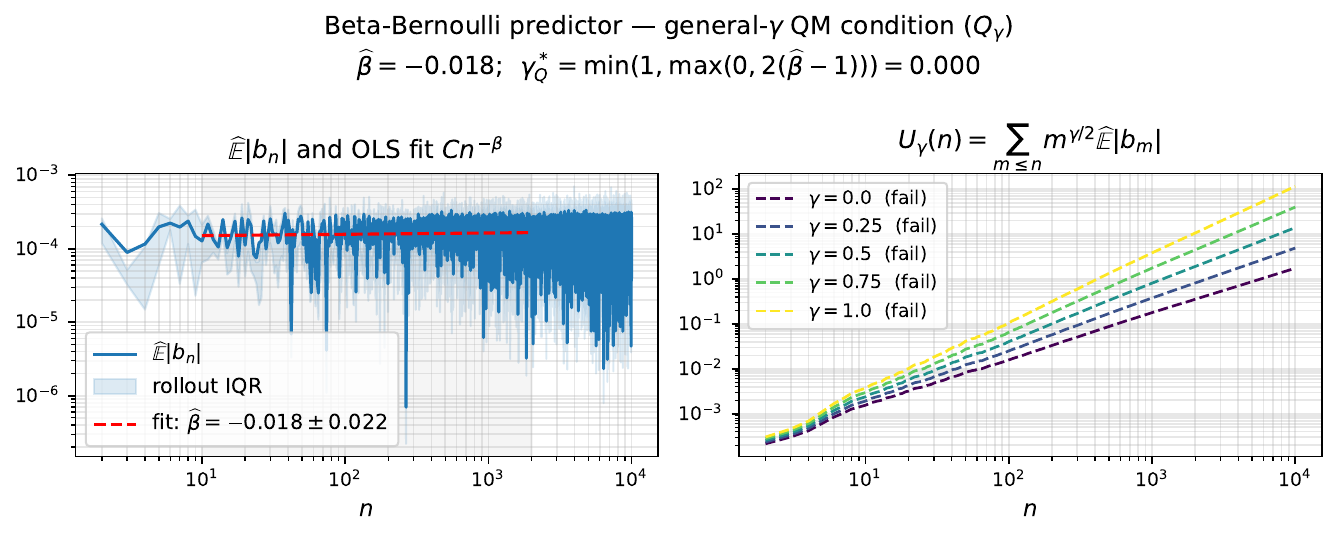}
    \caption{\textbf{$(Q1){+}(Q2)$, corrupted oracle (constant noise, $\varepsilon=10^{-3}$).} Same panel layout as Figure~\ref{fig:bb-600-qm}. Fitted $\widehat\beta=-0.018\pm 0.022$ (predicted $0$); $\gamma_Q^*=0$ (all $U_\gamma$ curves diverge).}
    \label{fig:bb-corrupt-noise1e-3-qm}
  \end{subfigure}
  \caption{\textbf{Corrupted oracle, constant noise $\varepsilon=10^{-3}$; rollouts sampled from this corrupted oracle's own predictive rule.}}
  \label{fig:bb-corrupt-noise1e-3}
\end{figure}

\begin{figure}[p]
  \centering
  \begin{subfigure}[t]{\linewidth}
    \centering
    \includegraphics[width=\linewidth,height=0.70\textheight,keepaspectratio]{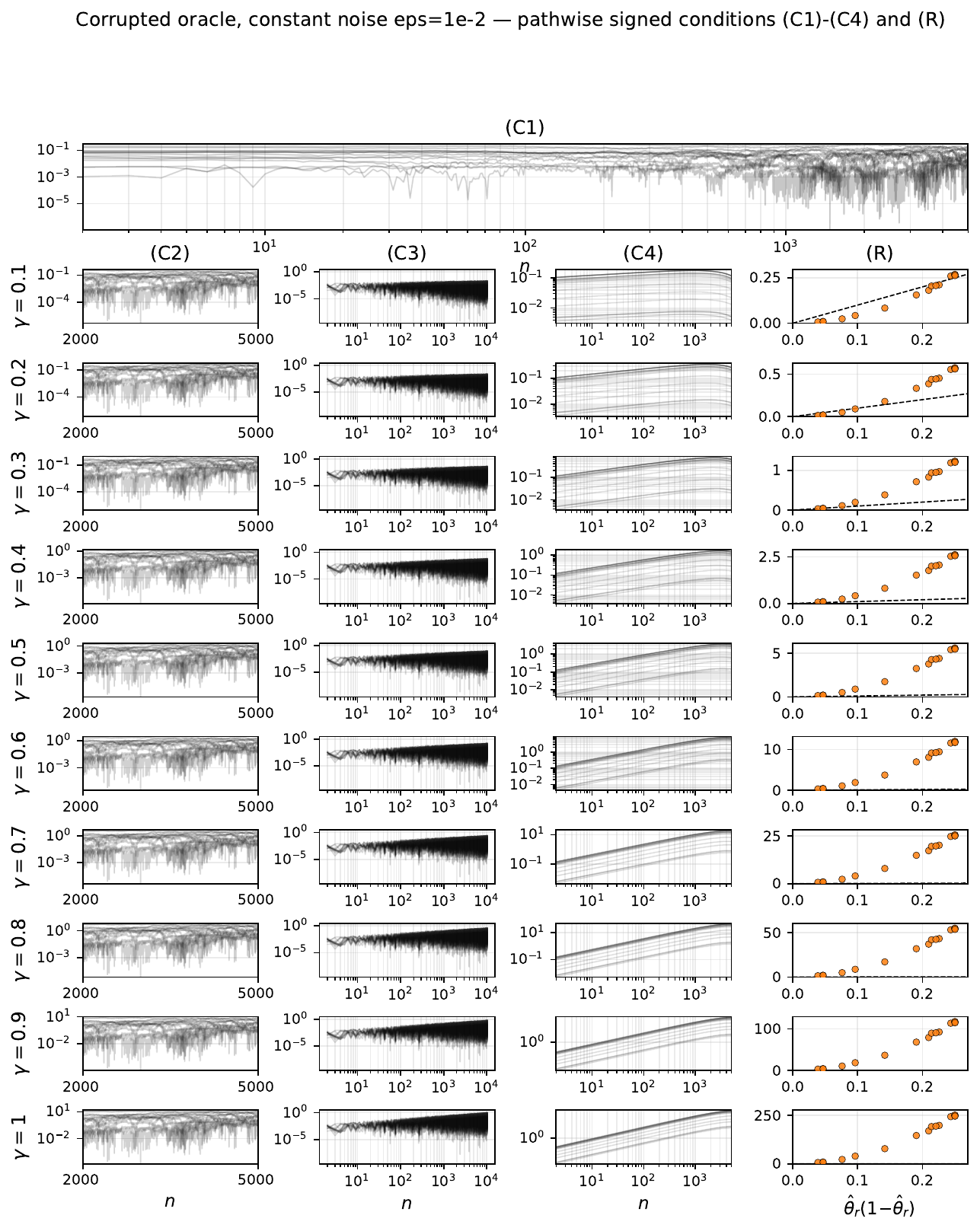}
    \caption{\textbf{Signed conditions, corrupted oracle (constant noise, $\varepsilon=10^{-2}$).} Same panel layout as Figure~\ref{fig:bb-600-signed}, under this predictive rule's own induced law. No rate-weighted panel decays at any tested $\gamma$.}
    \label{fig:bb-corrupt-noise1e-2-signed}
  \end{subfigure}\\[1.5ex]
  \begin{subfigure}[t]{\linewidth}
    \centering
    \includegraphics[width=\linewidth]{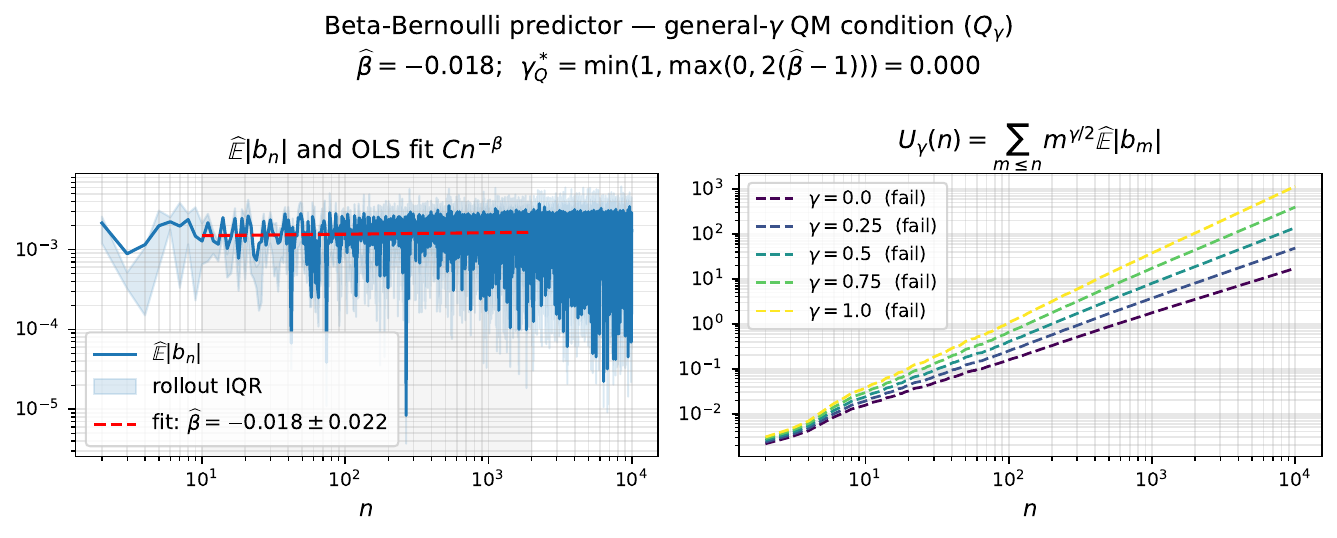}
    \caption{\textbf{$(Q1){+}(Q2)$, corrupted oracle (constant noise, $\varepsilon=10^{-2}$).} Same panel layout as Figure~\ref{fig:bb-600-qm}. Fitted $\widehat\beta=-0.018\pm 0.022$ (predicted $0$); $\gamma_Q^*=0$.}
    \label{fig:bb-corrupt-noise1e-2-qm}
  \end{subfigure}
  \caption{\textbf{Corrupted oracle, constant noise $\varepsilon=10^{-2}$; rollouts sampled from this corrupted oracle's own predictive rule.}}
  \label{fig:bb-corrupt-noise1e-2}
\end{figure}

\begin{figure}[p]
  \centering
  \begin{subfigure}[t]{\linewidth}
    \centering
    \includegraphics[width=\linewidth,height=0.70\textheight,keepaspectratio]{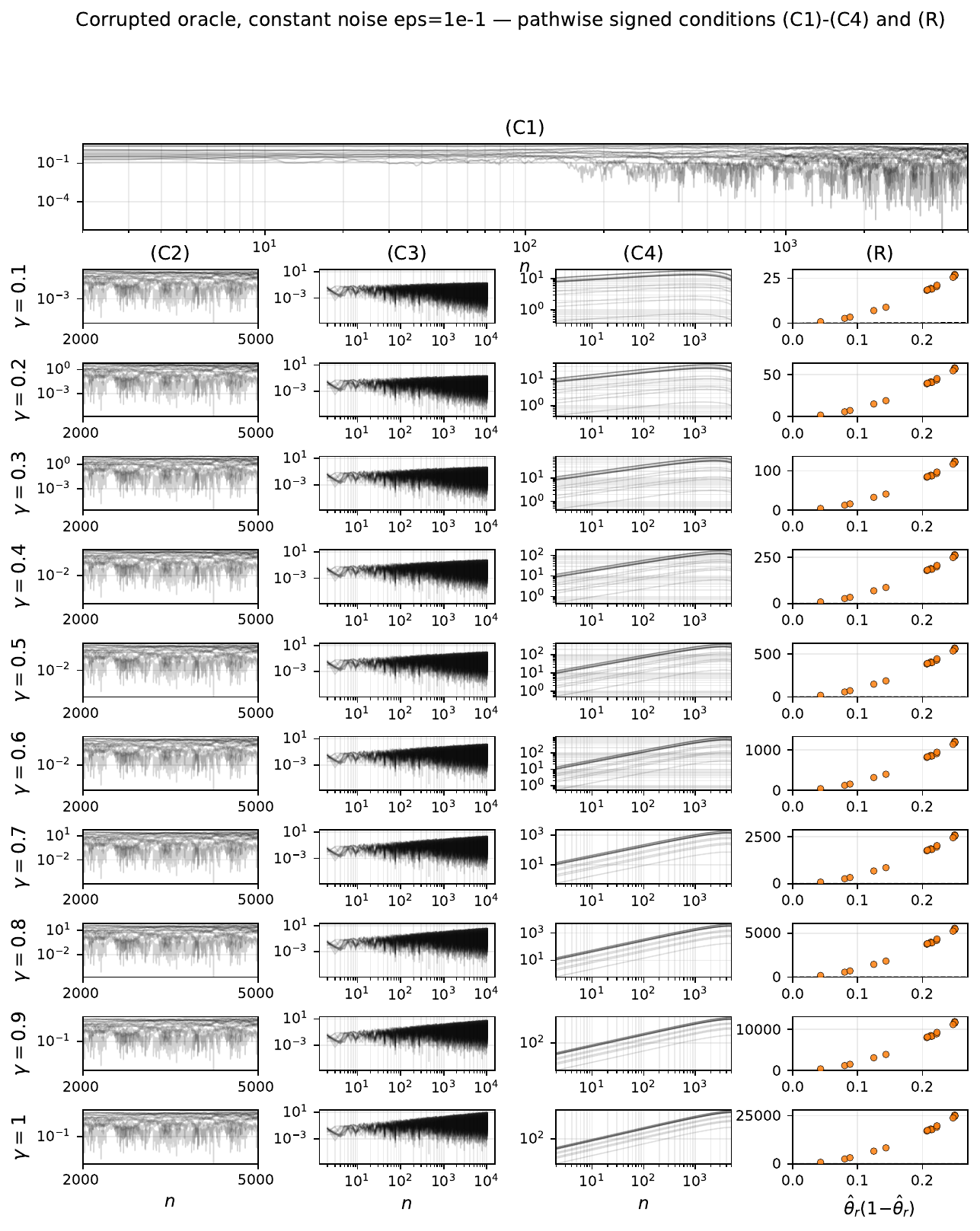}
    \caption{\textbf{Signed conditions, corrupted oracle (constant noise, $\varepsilon=10^{-1}$).} Same panel layout as Figure~\ref{fig:bb-600-signed}, under this predictive rule's own induced law. No rate-weighted panel decays at any tested $\gamma$; the larger perturbation amplitude is visible as elevated overall magnitude.}
    \label{fig:bb-corrupt-noise1e-1-signed}
  \end{subfigure}\\[1.5ex]
  \begin{subfigure}[t]{\linewidth}
    \centering
    \includegraphics[width=\linewidth]{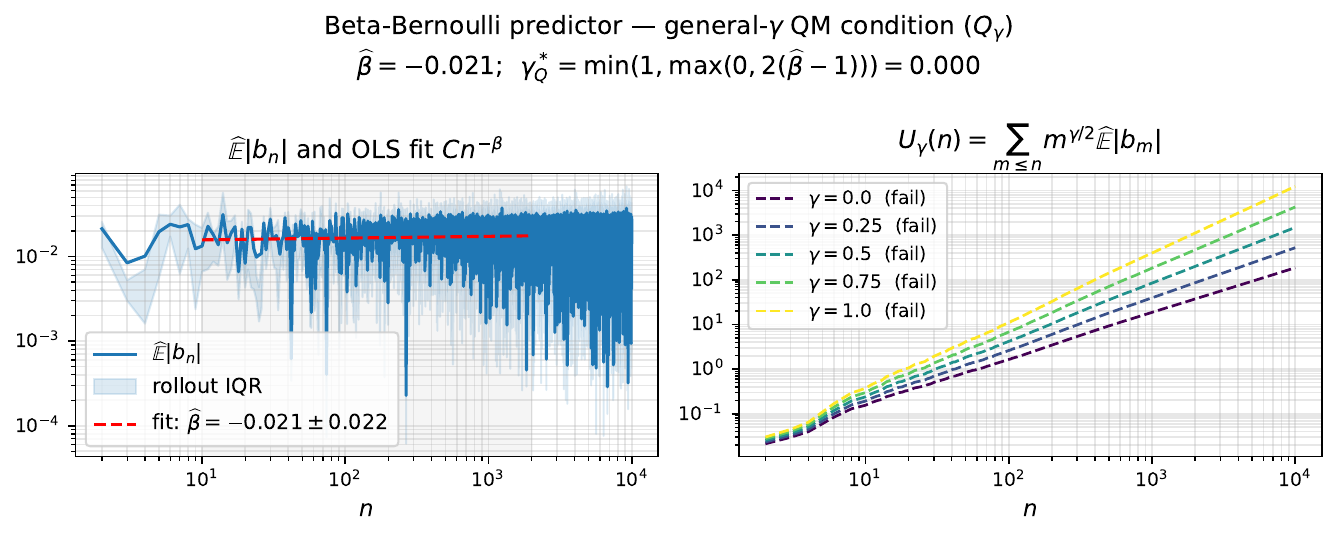}
    \caption{\textbf{$(Q1){+}(Q2)$, corrupted oracle (constant noise, $\varepsilon=10^{-1}$).} Same panel layout as Figure~\ref{fig:bb-600-qm}. Fitted $\widehat\beta=-0.021\pm 0.022$ (predicted $0$); $\gamma_Q^*=0$.}
    \label{fig:bb-corrupt-noise1e-1-qm}
  \end{subfigure}
  \caption{\textbf{Corrupted oracle, constant noise $\varepsilon=10^{-1}$; rollouts sampled from this corrupted oracle's own predictive rule.}}
  \label{fig:bb-corrupt-noise1e-1}
\end{figure}

\begin{figure}[p]
  \centering
  \begin{subfigure}[t]{\linewidth}
    \centering
    \includegraphics[width=\linewidth,height=0.70\textheight,keepaspectratio]{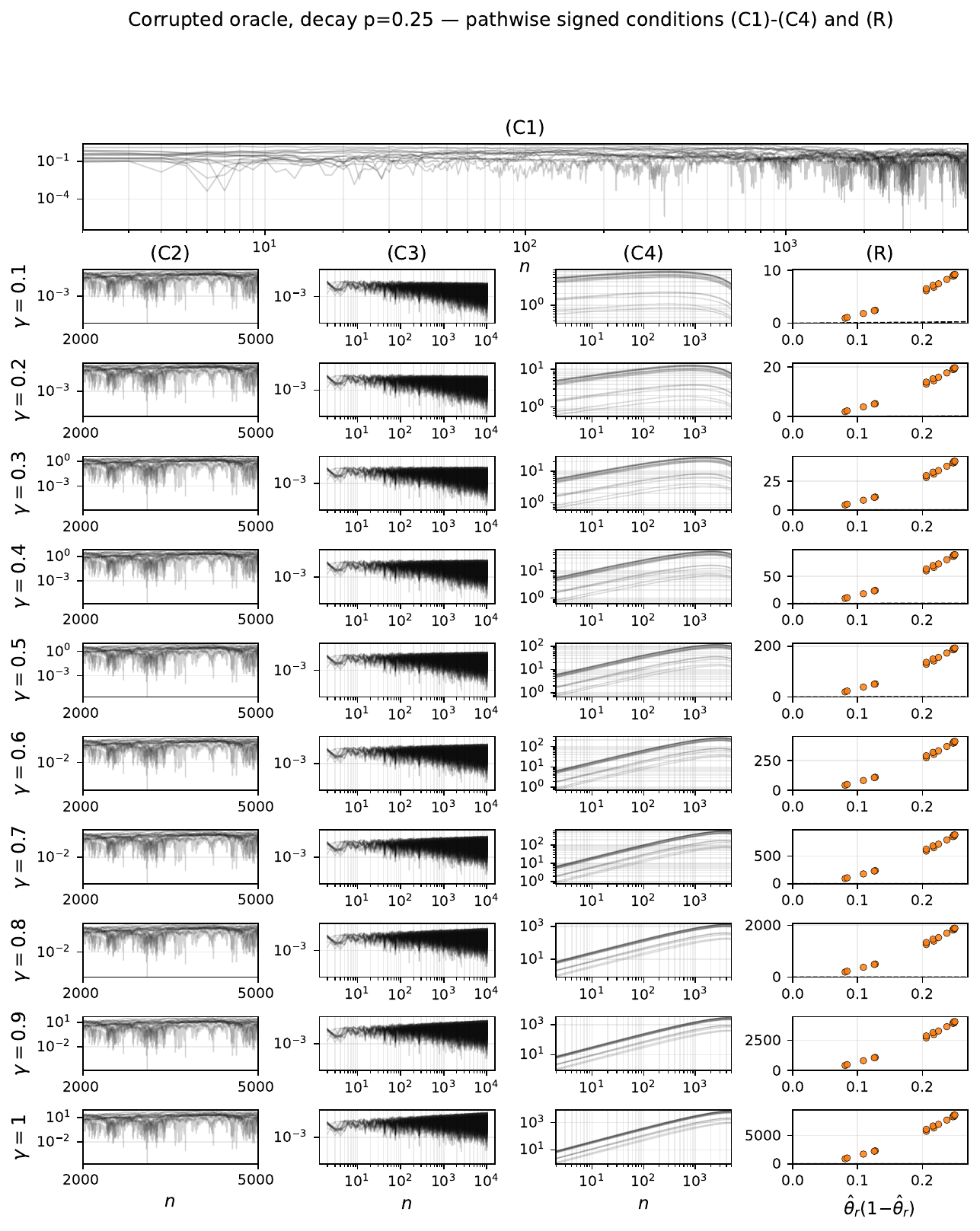}
    \caption{\textbf{Signed conditions, corrupted oracle (decay $\varepsilon n^{-p}$, $\varepsilon=0.5$, $p=0.25$).} Same panel layout as Figure~\ref{fig:bb-600-signed}, under this predictive rule's own induced law. $\lvert b_n\rvert=O(n^{-0.25})$ is too slow to survive $n^{\gamma/2}$ amplification: the pointwise panel (C3) decays only at $\gamma\le 0.5$; the tail-sum panels (C2), (C4) and the (R) scatter fail at every tested $\gamma$.}
    \label{fig:bb-corrupt-decay-p025-signed}
  \end{subfigure}\\[1.5ex]
  \begin{subfigure}[t]{\linewidth}
    \centering
    \includegraphics[width=\linewidth]{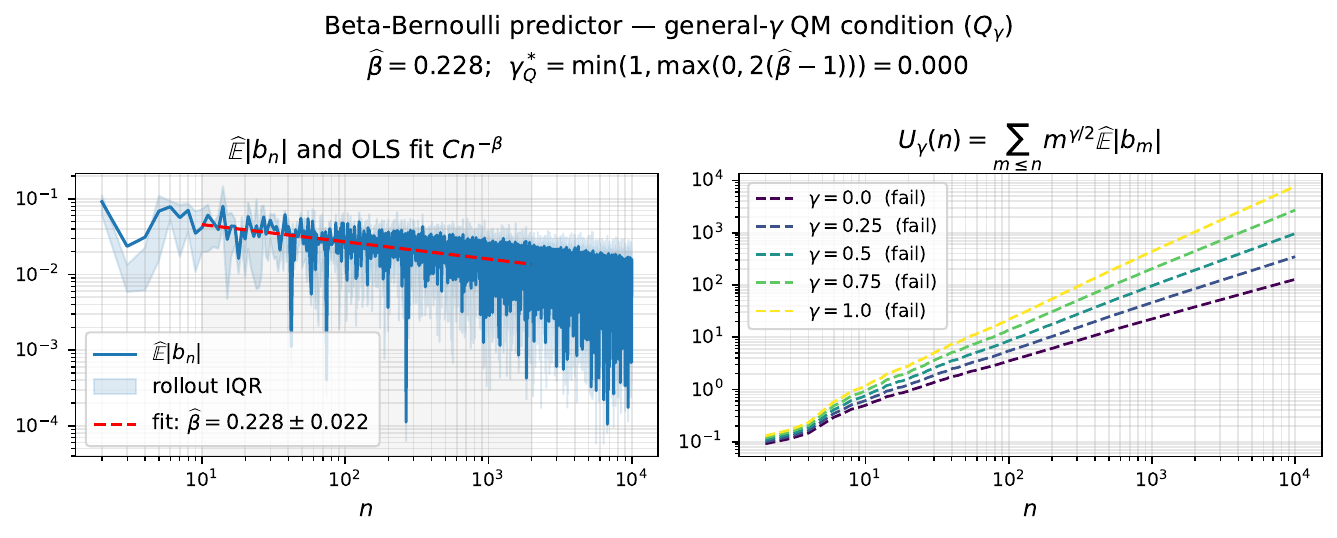}
    \caption{\textbf{$(Q1){+}(Q2)$, corrupted oracle (decay $\varepsilon n^{-p}$, $\varepsilon=0.5$, $p=0.25$).} Same panel layout as Figure~\ref{fig:bb-600-qm}. Fitted $\widehat\beta=0.228\pm 0.022$ (predicted $p=0.25$); $\gamma_Q^*=0$.}
    \label{fig:bb-corrupt-decay-p025-qm}
  \end{subfigure}
  \caption{\textbf{Corrupted oracle, decay $\varepsilon n^{-p}$ with $\varepsilon=0.5$, $p=0.25$; rollouts sampled from this corrupted oracle's own predictive rule.}}
  \label{fig:bb-corrupt-decay-p025}
\end{figure}

\begin{figure}[p]
  \centering
  \begin{subfigure}[t]{\linewidth}
    \centering
    \includegraphics[width=\linewidth,height=0.70\textheight,keepaspectratio]{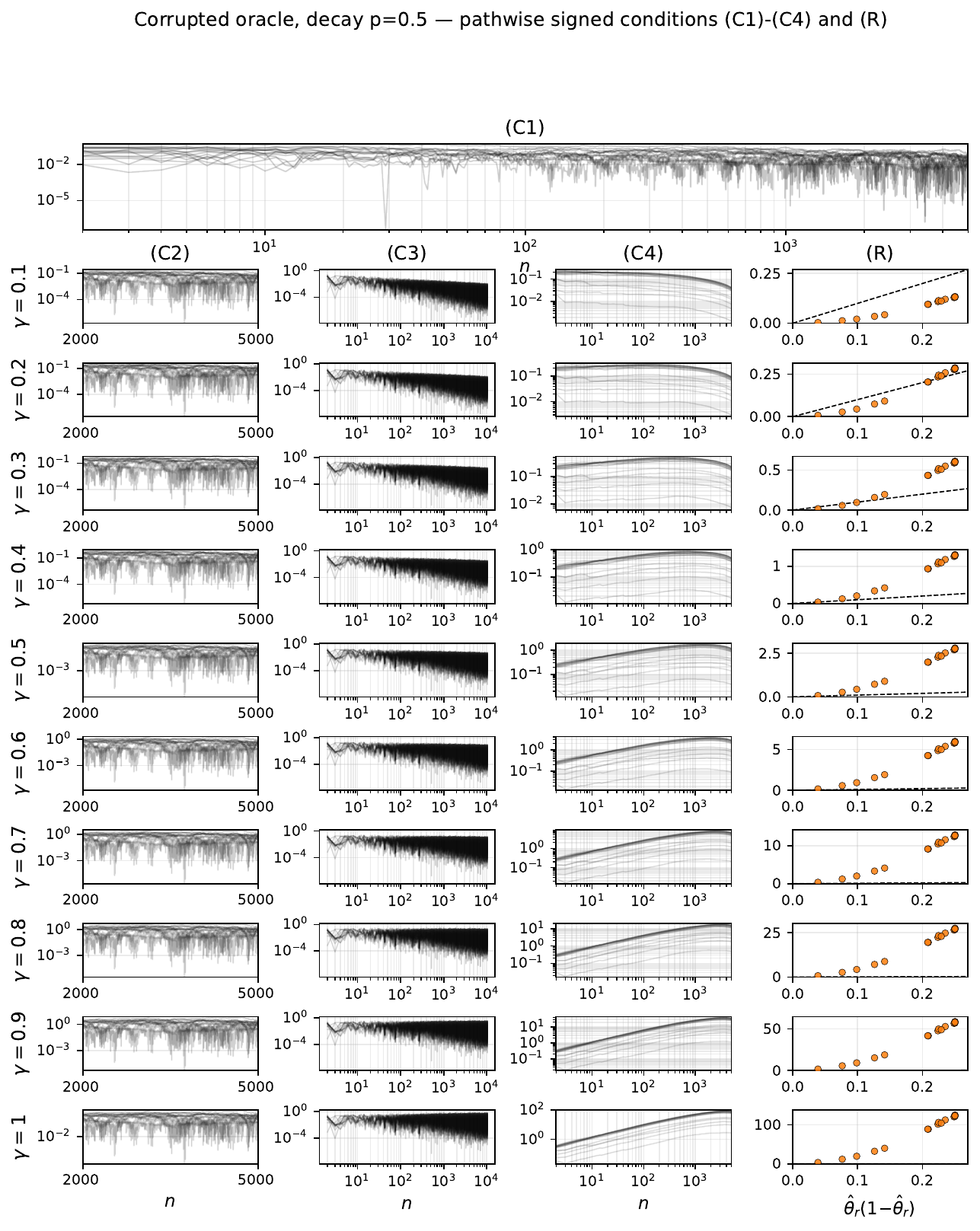}
    \caption{\textbf{Signed conditions, corrupted oracle (decay $\varepsilon n^{-p}$, $\varepsilon=0.5$, $p=0.5$).} Same panel layout as Figure~\ref{fig:bb-600-signed}, under this predictive rule's own induced law. $\lvert b_n\rvert=O(n^{-0.5})$: the pointwise panel (C3) is marginal at $\gamma=1$; the tail-residual panels (C2), (C4) fail at every tested $\gamma$ (log-divergent tail).}
    \label{fig:bb-corrupt-decay-p05-signed}
  \end{subfigure}\\[1.5ex]
  \begin{subfigure}[t]{\linewidth}
    \centering
    \includegraphics[width=\linewidth]{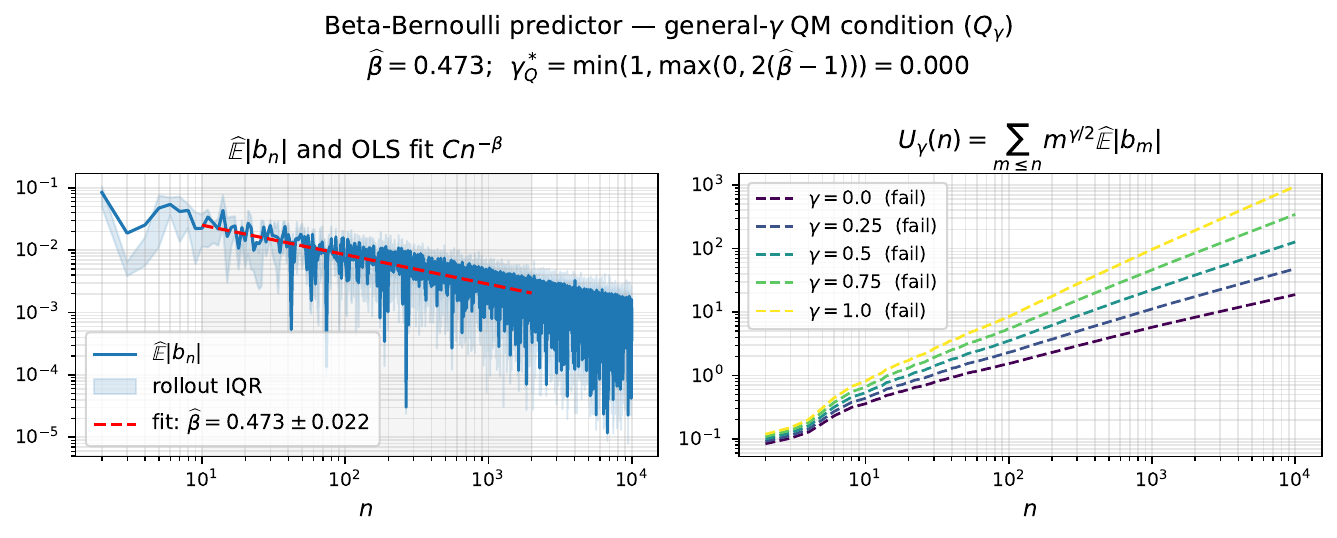}
    \caption{\textbf{$(Q1){+}(Q2)$, corrupted oracle (decay $\varepsilon n^{-p}$, $\varepsilon=0.5$, $p=0.5$).} Same panel layout as Figure~\ref{fig:bb-600-qm}. Fitted $\widehat\beta=0.473\pm 0.022$ (predicted $p=0.5$); $\gamma_Q^*=0$.}
    \label{fig:bb-corrupt-decay-p05-qm}
  \end{subfigure}
  \caption{\textbf{Corrupted oracle, decay $\varepsilon n^{-p}$ with $\varepsilon=0.5$, $p=0.5$; rollouts sampled from this corrupted oracle's own predictive rule.}}
  \label{fig:bb-corrupt-decay-p05}
\end{figure}

\begin{figure}[p]
  \centering
  \begin{subfigure}[t]{\linewidth}
    \centering
    \includegraphics[width=\linewidth,height=0.70\textheight,keepaspectratio]{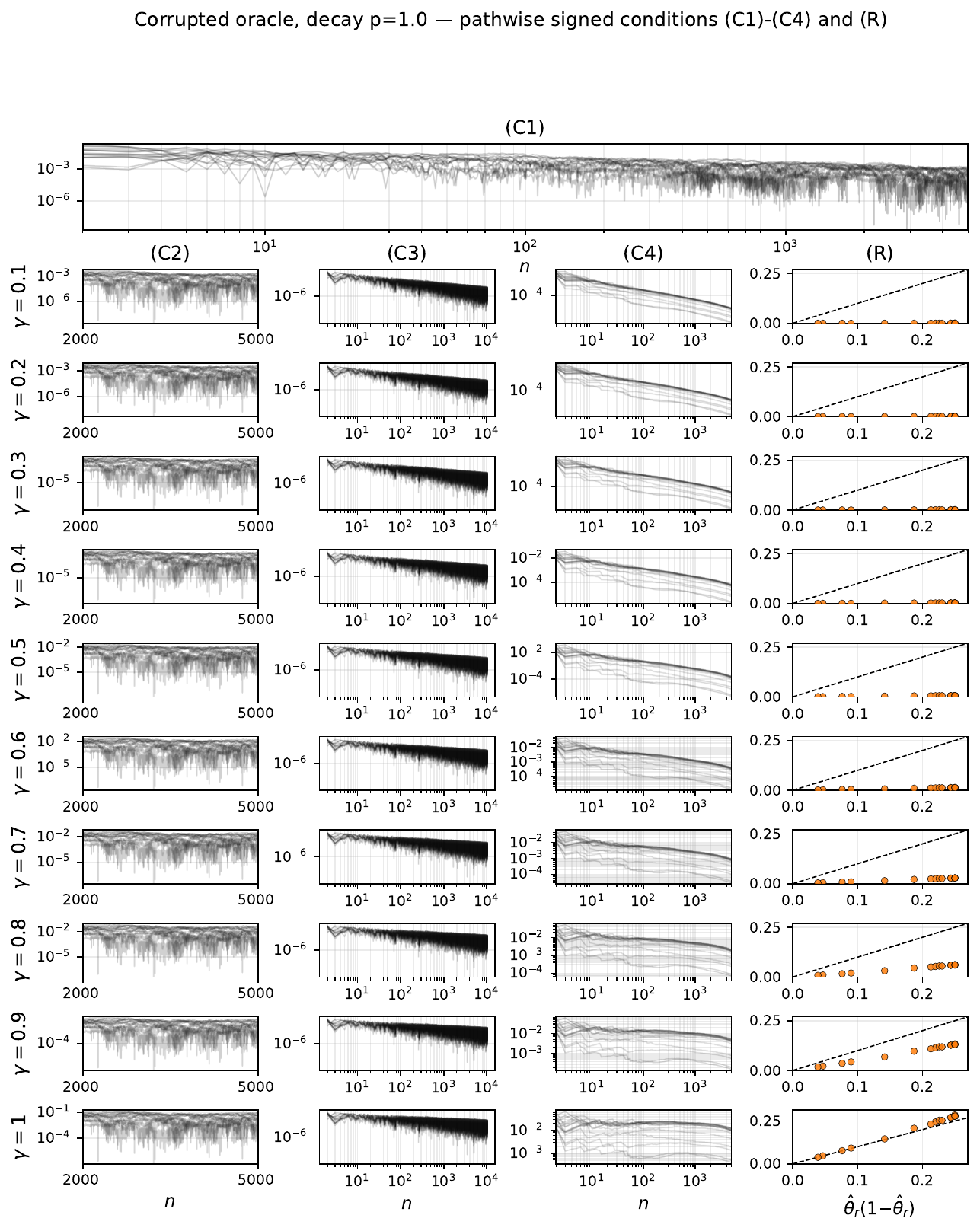}
    \caption{\textbf{Signed conditions, corrupted oracle (decay $\varepsilon n^{-p}$, $\varepsilon=0.5$, $p=1.0$).} Same panel layout as Figure~\ref{fig:bb-600-signed}, under this predictive rule's own induced law. $\lvert b_n\rvert=O(n^{-1})$: the pointwise panel (C3) decays at every tested $\gamma$; the tail-residual panels are on the boundary at $\gamma=1$.}
    \label{fig:bb-corrupt-decay-p10-signed}
  \end{subfigure}\\[1.5ex]
  \begin{subfigure}[t]{\linewidth}
    \centering
    \includegraphics[width=\linewidth]{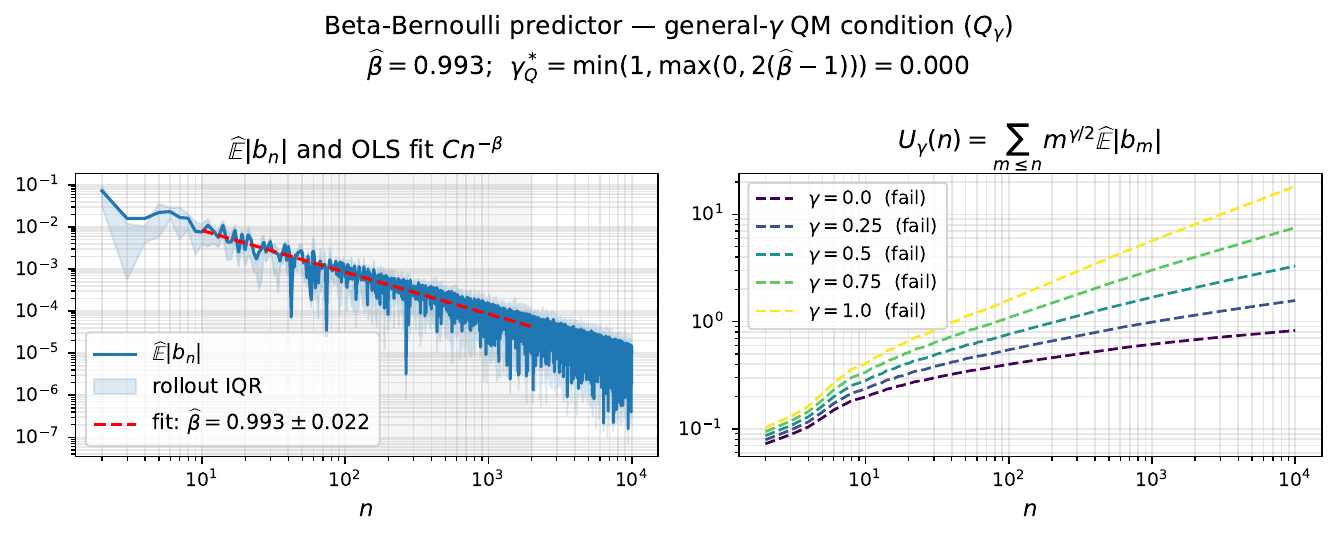}
    \caption{\textbf{$(Q1){+}(Q2)$, corrupted oracle (decay $\varepsilon n^{-p}$, $\varepsilon=0.5$, $p=1.0$).} Same panel layout as Figure~\ref{fig:bb-600-qm}. Fitted $\widehat\beta=0.993\pm 0.022$ (predicted $p=1.0$); $\gamma_Q^*=0$ (on the boundary of the $\gamma=0$ special case).}
    \label{fig:bb-corrupt-decay-p10-qm}
  \end{subfigure}
  \caption{\textbf{Corrupted oracle, decay $\varepsilon n^{-p}$ with $\varepsilon=0.5$, $p=1.0$; rollouts sampled from this corrupted oracle's own predictive rule.}}
  \label{fig:bb-corrupt-decay-p10}
\end{figure}

\begin{figure}[p]
  \centering
  \begin{subfigure}[t]{\linewidth}
    \centering
    \includegraphics[width=\linewidth,height=0.70\textheight,keepaspectratio]{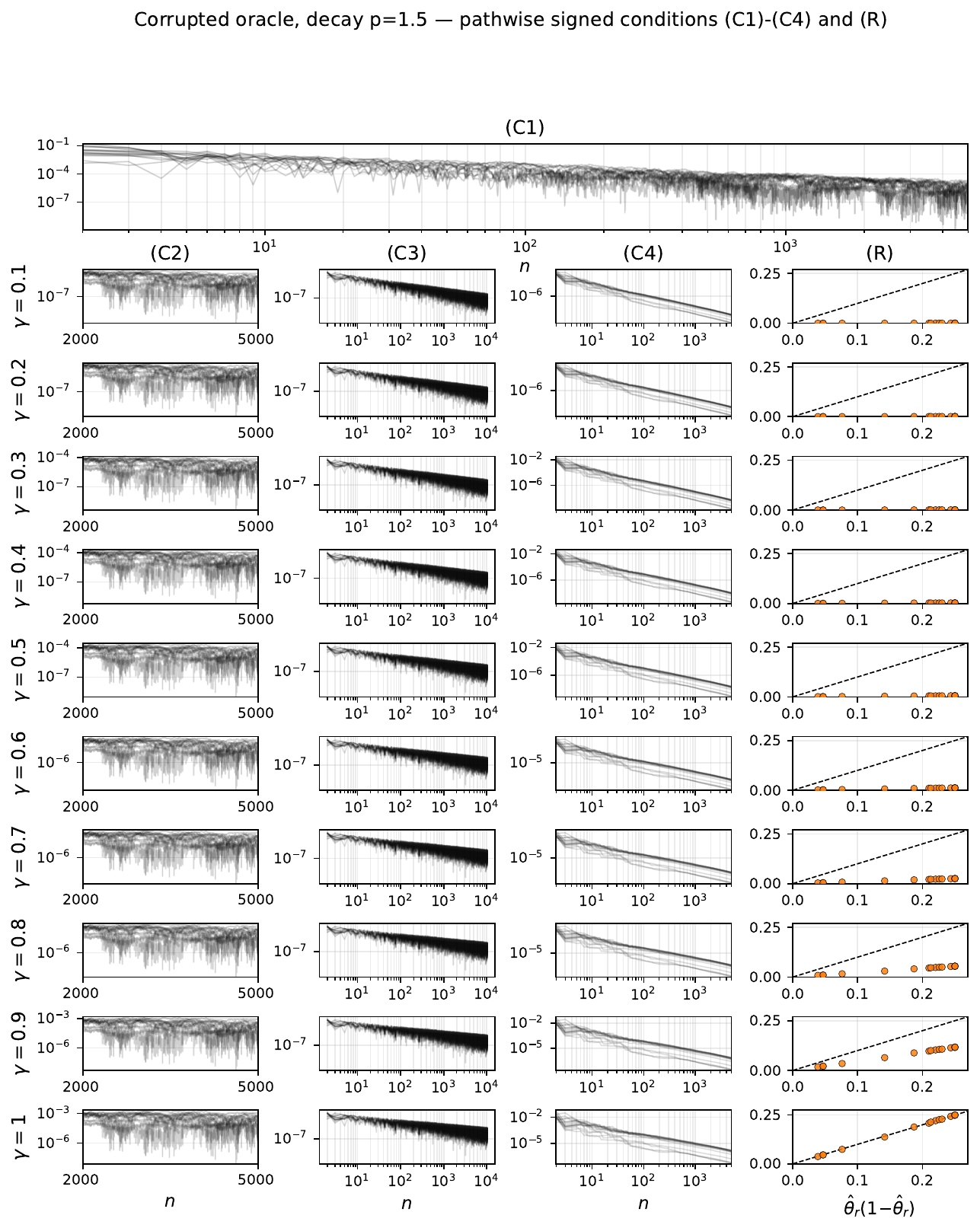}
    \caption{\textbf{Signed conditions, corrupted oracle (decay $\varepsilon n^{-p}$, $\varepsilon=0.5$, $p=1.5$).} Same panel layout as Figure~\ref{fig:bb-600-signed}, under this predictive rule's own induced law. $\lvert b_n\rvert=O(n^{-1.5})$ is fast enough that every rate-weighted panel decays at every tested $\gamma$.}
    \label{fig:bb-corrupt-decay-p15-signed}
  \end{subfigure}\\[1.5ex]
  \begin{subfigure}[t]{\linewidth}
    \centering
    \includegraphics[width=\linewidth]{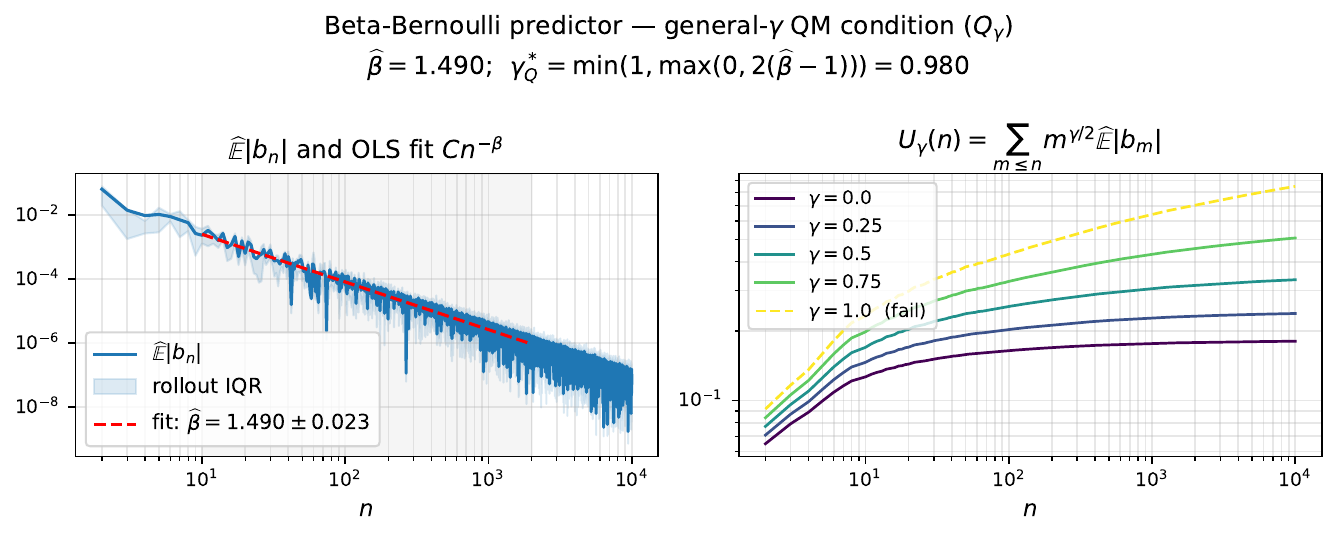}
    \caption{\textbf{$(Q1){+}(Q2)$, corrupted oracle (decay $\varepsilon n^{-p}$, $\varepsilon=0.5$, $p=1.5$).} Same panel layout as Figure~\ref{fig:bb-600-qm}. Fitted $\widehat\beta=1.490\pm 0.023$ (predicted $p=1.5$); $\gamma_Q^*\approx 0.98$. The $\gamma=1$ boundary is just missed.}
    \label{fig:bb-corrupt-decay-p15-qm}
  \end{subfigure}
  \caption{\textbf{Corrupted oracle, decay $\varepsilon n^{-p}$ with $\varepsilon=0.5$, $p=1.5$; rollouts sampled from this corrupted oracle's own predictive rule.}}
  \label{fig:bb-corrupt-decay-p15}
\end{figure}

\begin{figure}[p]
  \centering
  \begin{subfigure}[t]{\linewidth}
    \centering
    \includegraphics[width=\linewidth,height=0.70\textheight,keepaspectratio]{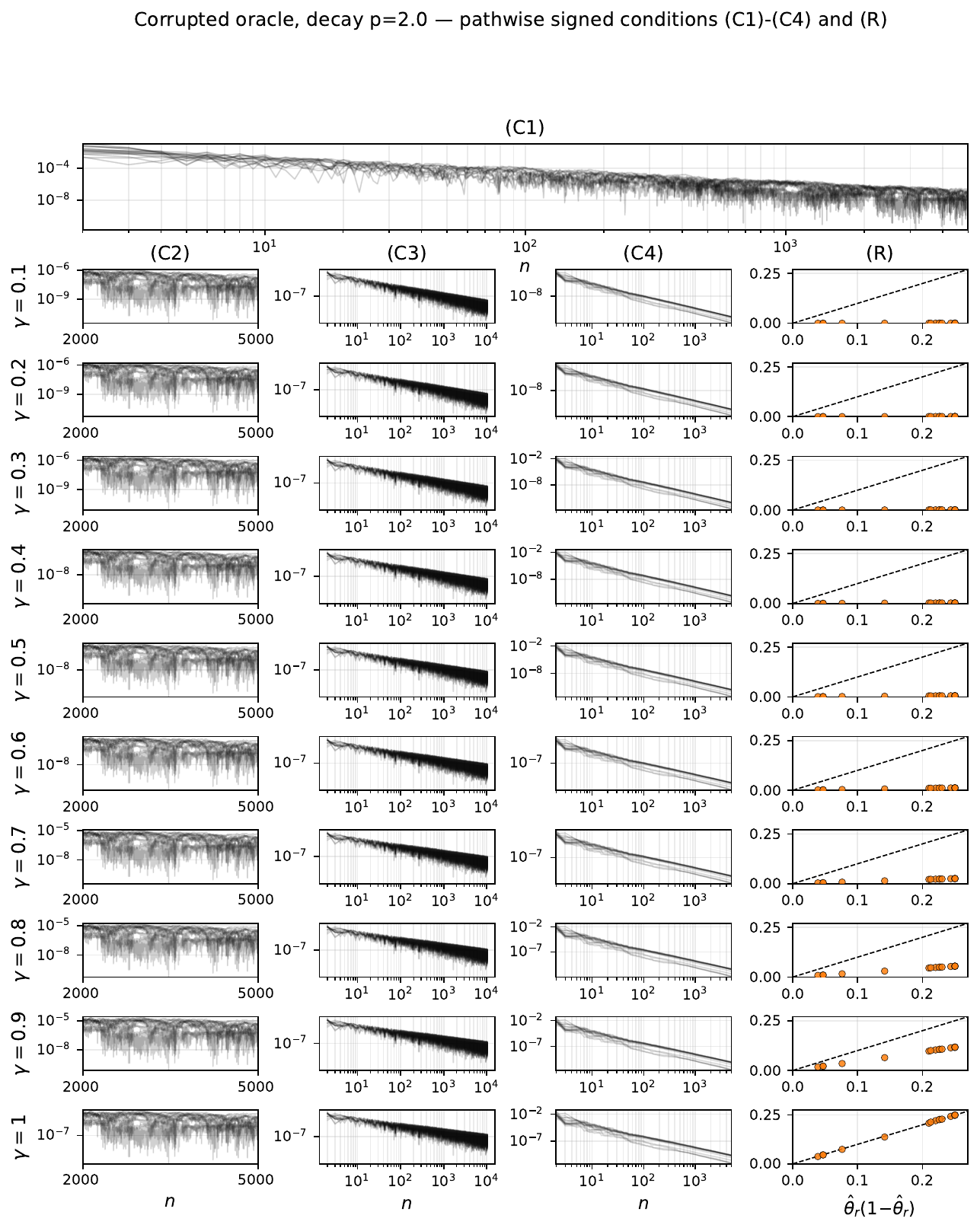}
    \caption{\textbf{Signed conditions, corrupted oracle (decay $\varepsilon n^{-p}$, $\varepsilon=0.5$, $p=2.0$).} Same panel layout as Figure~\ref{fig:bb-600-signed}, under this predictive rule's own induced law. $\lvert b_n\rvert=O(n^{-2})$ is fast enough that every rate-weighted panel decays cleanly at every tested $\gamma$.}
    \label{fig:bb-corrupt-decay-p20-signed}
  \end{subfigure}\\[1.5ex]
  \begin{subfigure}[t]{\linewidth}
    \centering
    \includegraphics[width=\linewidth]{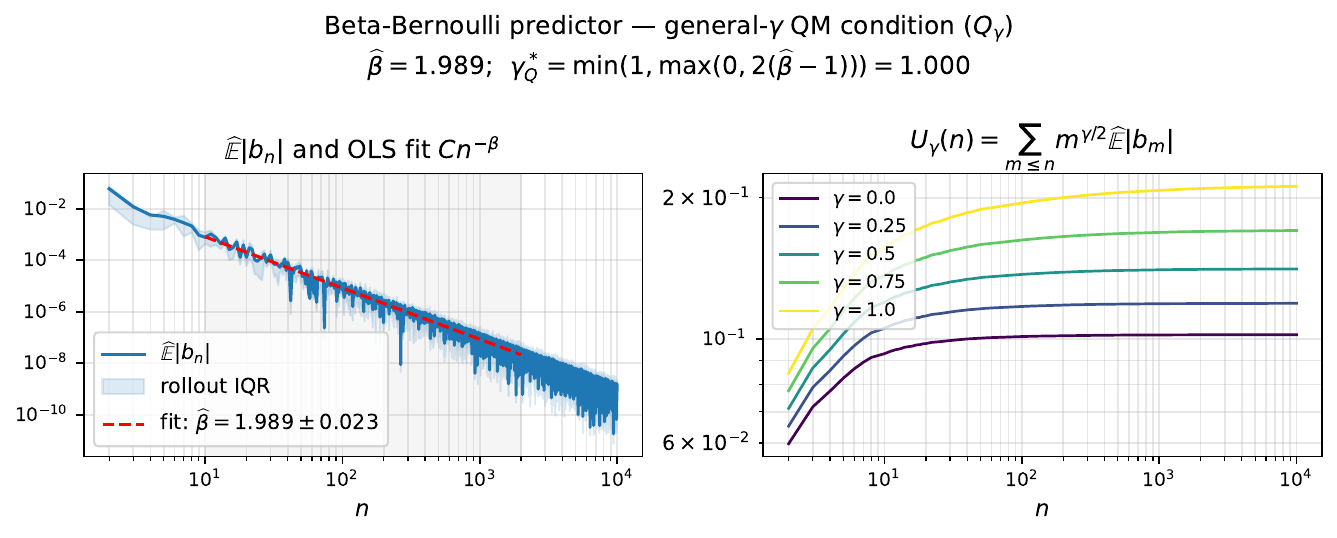}
    \caption{\textbf{$(Q1){+}(Q2)$, corrupted oracle (decay $\varepsilon n^{-p}$, $\varepsilon=0.5$, $p=2.0$).} Same panel layout as Figure~\ref{fig:bb-600-qm}. Fitted $\widehat\beta=1.989\pm 0.023$ (predicted $p=2.0$); $\gamma_Q^*\approx 1.00$. All tested $U_\gamma$ curves are bounded (solid).}
    \label{fig:bb-corrupt-decay-p20-qm}
  \end{subfigure}
  \caption{\textbf{Corrupted oracle, decay $\varepsilon n^{-p}$ with $\varepsilon=0.5$, $p=2.0$; rollouts sampled from this corrupted oracle's own predictive rule.}}
  \label{fig:bb-corrupt-decay-p20}
\end{figure}

\FloatBarrier
\subsection{A \texorpdfstring{$\gamma$}{gamma}-free check of the Gaussian shape}
\label{app:variance-only}

The (C1)--(C4) + (R) grid is organised around the supported rate $\gamma$ at which the variance estimator $\Vb_n$ contracts. The asymptotic Gaussian shape itself, however, does not require knowledge of $\gamma$. Inspecting the proof of Theorem~\ref{th:berti_new}, the predictive residual decomposes into a martingale plus a cumulative-bias term,
\begin{equation*}
  \tilde F(x,t) - F_n(x,t) \;=\; \sum_{k\geq n+1}\Delta_k(x,t)
  \;=\; \underbrace{\sum_{k\geq n+1}\bigl(\Delta_k - \EE[\Delta_k\mid Z_{1:k-1}]\bigr)}_{\text{martingale}} + \underbrace{\sum_{k\geq n+1}\EE[\Delta_k\mid Z_{1:k-1}]}_{\text{cumulative bias}},
\end{equation*}
and the Gaussian approximation in Theorem~\ref{th:ascondmult} requires
\begin{align}
  \sum_{k\geq n}\EE[\Delta_k\mid Z_{1:k-1}] &\;=\; o\!\left(\!\Bigl(\sum_{k\geq n}\Delta_k^2\Bigr)^{1/2}\right) \quad \text{(bias dominated by martingale s.d.)}, \label{eq:vo-bias}\\
  \sum_{k\geq n}\EE[\Delta_k\mid Z_{1:k-1}]^2 &\;=\; o\!\left(\sum_{k\geq n}\Delta_k^2\right) \quad \text{(drift variance dominated by martingale variance)}, \label{eq:vo-var}
\end{align}
which, given condition~(iii) at the same $\gamma$, correspond to the tail-sum limit part of $(ia)$ and to $(ib)$ of Theorem~\ref{th:berti_new} (up to the single term $n^{\gamma/2}b_n$, which is controlled by the sup part; the sup part of $(ia)$, its $L^1$ domination, and the moment condition $(ii)$ are not captured by these displays). The factor $n^\gamma$ does not appear in either display: as in the proof, the $n^\gamma$ on $\Vb=\lim_n n^\gamma\sum_{k\geq n}\Delta_k\Delta_k^\top$ cancels against the $n^{-\gamma}$ in the Gaussian rescaling $\mathcal N(F_n,\Vb/n^\gamma)$, and re-enters only when $\Vb$ is replaced by its rate-weighted estimator~\eqref{eq:Vn_xA}; even there, a misspecified $\gamma$ shifts interval lengths only by a constant factor, leaving the Gaussian shape intact.

We therefore plot, on a single set of log--log axes per predictive rule and one curve per rollout,
\begin{equation*}
  \mathrm{(a)}\;\Bigl(\!\!\sum_{k\geq n} b_k^{(r)}\!\Bigr)^{\!2},\qquad
  \mathrm{(b)}\;\sum_{k\geq n}\bigl(b_k^{(r)}\bigr)^{2},\qquad
  \mathrm{(c)}\;\sum_{k\geq n}\bigl(\Delta_k^{(r)}\bigr)^{2},
\end{equation*}
where (a) is the squared form of \eqref{eq:vo-bias} (so it can be compared to (c) on the same axes) and (b) is exactly \eqref{eq:vo-var}. The two conditions read off as: (a) and (b) sit far below (c). Figures~\ref{fig:bb-variance-only-bfts} and~\ref{fig:bb-variance-only-corrupt} show the result.

\begin{figure}[h]
  \centering
  \includegraphics[width=\linewidth]{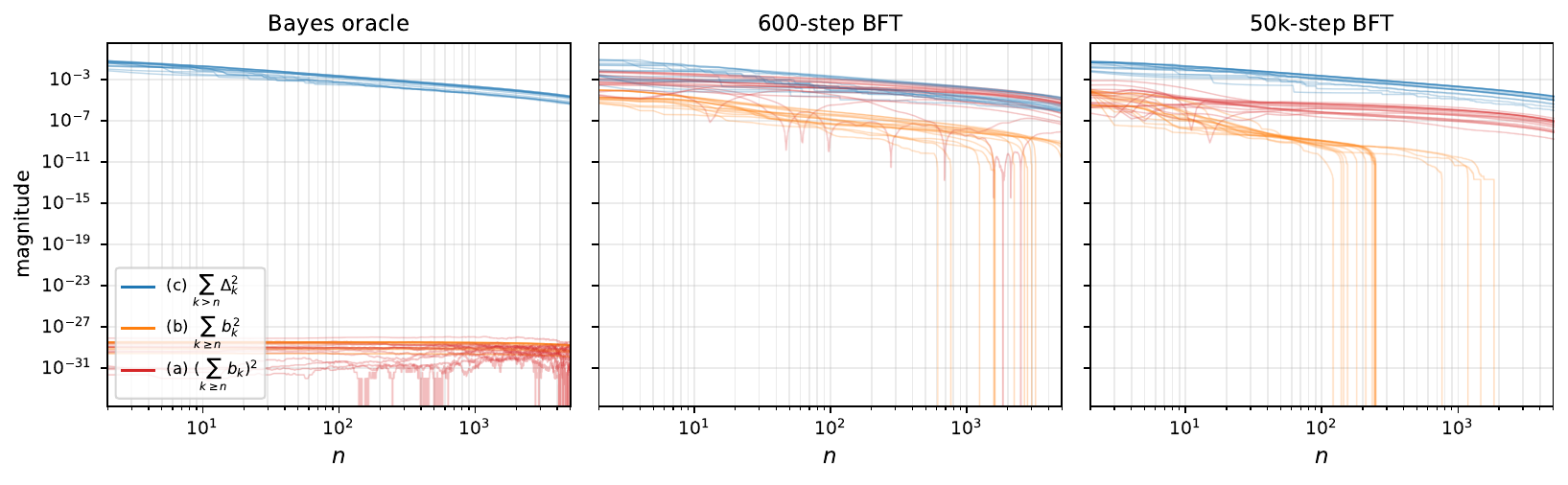}
  \caption{\textbf{$\gamma$-free Gaussian-shape diagnostic on the Bayes oracle and the two BFT checkpoints.} Per-rollout curves of (a)~$\bigl(\sum_{k\geq n} b_k^{(r)}\bigr)^2$ (squared cumulative bias, red), (b)~$\sum_{k\geq n}(b_k^{(r)})^2$ (drift's contribution to variance, orange), and (c)~$\sum_{k\geq n}(\Delta_k^{(r)})^2$ (martingale variance, blue) under each predictive rule's own induced law ($16$ rollouts, $N{-}1=10{,}000$). For an asymptotically Gaussian rule, (a) and (b) sit far below (c).}
  \label{fig:bb-variance-only-bfts}
\end{figure}

\begin{figure}[h]
  \centering
  \includegraphics[width=\linewidth]{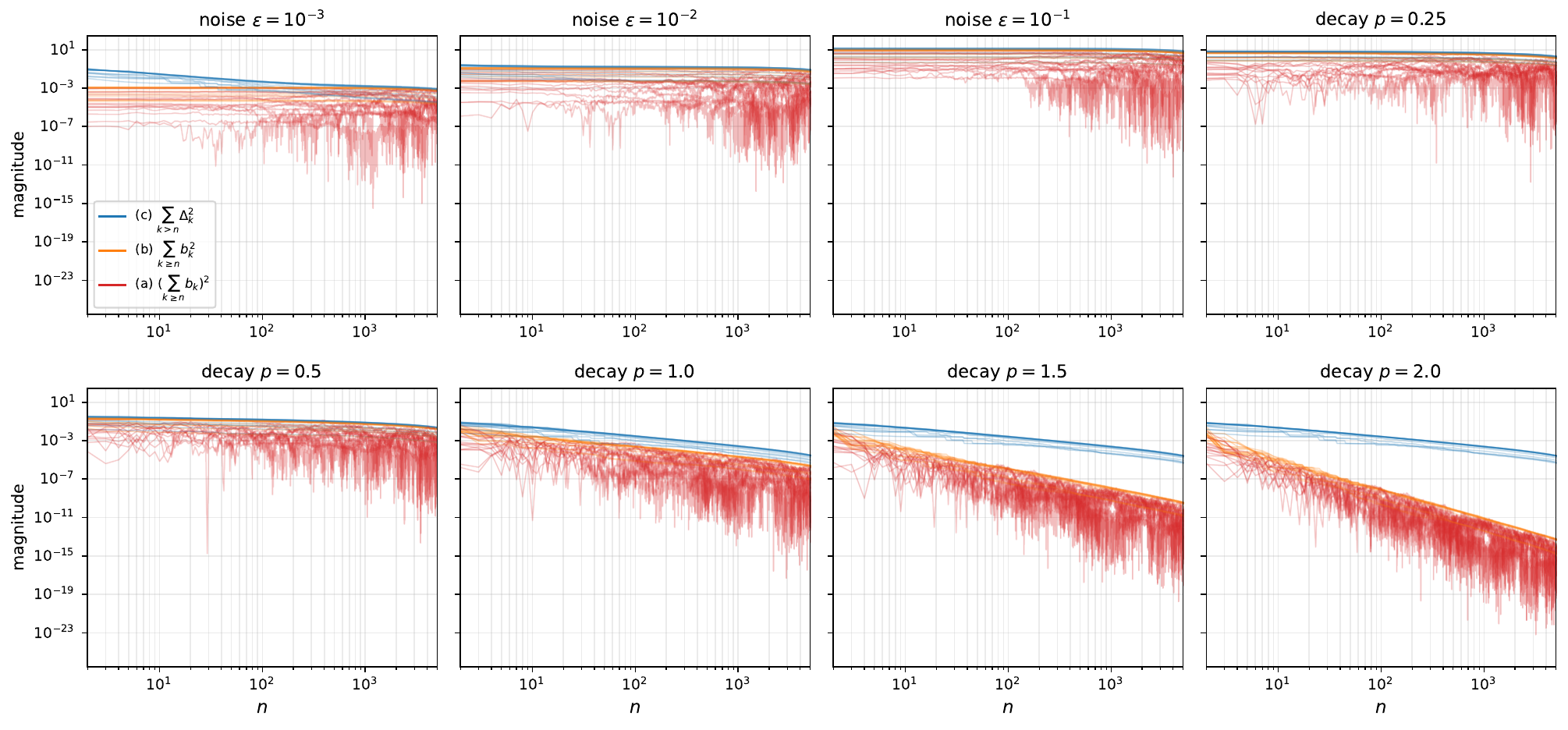}
  \caption{\textbf{$\gamma$-free Gaussian-shape diagnostic on the eight corrupted oracles.} Same layout and colour key as Figure~\ref{fig:bb-variance-only-bfts}; constant-noise corruptions (top row) and decay-$p$ corruptions (bottom row).}
  \label{fig:bb-variance-only-corrupt}
\end{figure}

\paragraph{Findings.}
On the Bayes oracle (Figure~\ref{fig:bb-variance-only-bfts}, left), (a) and (b) sit at the floating-point random-walk floor ($\sim 10^{-29}$) while (c) is $\sim 10^{-4}$ at $n=2000$, so (c) dominates by some $25$ orders of magnitude. The 600-step BFT (centre) has $(\sum b_k)^2/\sum\Delta_k^2 \approx 0.64$ at $n=2000$: the under-trained checkpoint's squared cumulative bias is comparable to its realised variance. This is consistent with the QM panel's $\widehat\beta\approx 1.16$ (Appendix~\ref{app:theorycheck-findings}), which says the drift has not yet decayed far below the noise on this checkpoint; the (C1)--(C4) grid does not surface this quantitatively. The 50k-step BFT (right) gives $(\sum b_k)^2/\sum\Delta_k^2 \approx 5\times 10^{-3}$ and $\sum b_k^2/\sum\Delta_k^2 \le 10^{-4}$ at $n=2000$; the bias and drift variance are well separated from the martingale variance.

On the corrupted oracles (Figure~\ref{fig:bb-variance-only-corrupt}), $\sum b_k^2$ approaches or exceeds $\sum\Delta_k^2$ on every constant-noise corruption ($\sum b_k^2/\sum\Delta_k^2 \approx 0.79$ at $\varepsilon=10^{-1}$) and on the slow-decay corruptions ($\approx 0.77$ at $p=0.5$), so the drift's contribution to the variance is not negligible and the Gaussian approximation is broken. Decay $p\ge 1.5$ recovers the oracle-like separation between (a), (b) and (c).

\paragraph{Relation to the $\gamma$-sweep.}
This diagnostic complements the (C1)--(C4) + (R) grid. The grid supports the \emph{rate} $\gamma$ at which $\Vb_n$ contracts, which is needed to construct the asymptotic credible bands of \eqref{eq:Vn_xA}; the diagnostic above supports the \emph{shape} of the limiting law without committing to a $\gamma$. Both are needed for valid bands.

\FloatBarrier
\section{Figure~\ref{fig:beta_bernoulli_intro} details}
\label{app:headline-figure}

Figure~\ref{fig:beta_bernoulli_intro} in the introduction illustrates our predictive CLT on the 50k-step Beta-Bernoulli BFT; architecture and meta-learning are described in Appendix~\ref{app:theorycheck-setup}.

\paragraph{Generating the figure (seed $7$).} The displayed sequence is drawn at a \emph{fixed} $\theta^\star=0.3$, with $y_1,\dots,y_{200}\stackrel{\mathrm{iid}}{\sim}\mathrm{Bernoulli}(\theta^\star)$. The Beta(1,1) prior used in meta-learning, updated under these observations, gives the analytic Beta posterior shown in panel~(a). The step-$0$ predictive probability is set analytically to $g_0=1/2$, since the BFT's forward pass requires at least one context token. For $k\geq 1$, $g_k=\PP_\phi(Y_{k+1}=1\mid y_{1:k})$ is computed by a single forward pass conditioned on $y_{1:k}$. Forward passes run in \texttt{float32}; the increments $\Delta_k=g_k-g_{k-1}$ and the variance estimator are computed in \texttt{float64}.

\paragraph{Panel~(a).} Two distributions over $\tilde\theta\in[0,1]$ at $n=200$:
\begin{itemize}[nosep,leftmargin=*]
  \item the analytic posterior $\tilde\theta\mid y_{1:n}\sim\mathrm{Beta}\bigl(\alpha+\sum_{i=1}^n y_i,\;\beta+n-\sum_{i=1}^n y_i\bigr)$ (orange);
  \item the predictive-CLT Gaussian approximation $\mathcal N(g_n,\,V_n/n)$ with $V_n=\tfrac{1}{n}\sum_{k=1}^n k^2\Delta_k^2$ (blue dashed).
\end{itemize}

\paragraph{Panel~(b).} \emph{Top:} the predictive trajectory $(g_k)_{k=0}^{200}$. \emph{Bottom:} the per-step contribution $k^2\Delta_k^2/n$ to $V_n$ as a stem plot, illustrating how $V_n$ accumulates from the volatility of the predictive rule.

\FloatBarrier
\section{Frequentist coverage experiment details and extra results}
\label{app:coverage}

\subsection{Implementation of predictive CLT}
Fix a grid $\xb=[x_1,\dots,x_m]^\top$. In binary classification we work with the
success-probability vector
\[
  \gb_n(\xb)
  :=
  \bigl(g_n(x_1),\dots,g_n(x_m)\bigr)^\top.
\]
In regression, for a fixed threshold $t\in\mathbb R$, we work with the
predictive CDF vector
\[
  \Fb_n(\xb,t)
  :=
  \bigl(F_n(x_1,t),\dots,F_n(x_m,t)\bigr)^\top.
\]

Let $\widehat\Sigma_n$ denote the chosen plug-in covariance on the grid (the $\gamma=1$ case of Appendix~\ref{app:credible_bands}, as assumed throughout):
\[
  \widehat\Sigma_n \in \left\{\frac{\Vb_n}{n},\ \frac{\Ub_n}{n}\right\},
\]
where the dependence on $(\xb,t)$ (regression) or on the classification setting
is suppressed for readability. We approximate the expectation in $\Ub_n$ using 1000 Monte Carlo samples of $Z_{n+1}=(X_{n+1},Y_{n+1})$, with $X_{n+1}$ drawn from the empirical measure of $x_{1:n}$ and $Y_{n+1}\mid X_{n+1},z_{1:n}$ drawn from TabPFN's predictive distribution $P_n(\cdot\mid X_{n+1},z_{1:n})$.

\subsection{Baselines}
\paragraph{Bootstrap.}
For each bootstrap sample $b = 1, \dots, B$, we sample $n$ data points uniformly
with replacement from $z_{1:n}$, refit \texttt{TabPFN} on this bootstrap
dataset, and evaluate the resulting predictive class-1 probability $g_n^{(b)}(x^\ast)$ or
CDF value $F_n^{(b)}(x^\ast, t)$ at each grid point $x^\ast \in \mathcal{X}$. We draw
$B=200$ bootstrap samples.

\paragraph{\citet{nagler25uncertainty} (NR).}
An alternative to the asymptotic approximation of Section~\ref{sec:PCLT} is PMC \citep{fortini20quasibayes}, described here in the unsupervised case for an event $A$. Starting from the observed data $z_{1:n}$, one simulates $z_{k+1} \sim P_k$ for $k = n, \dots, N-1$. The terminal predictive probability $P_N(A)$ is then one approximate draw from the posterior of $\tilde{P}(A) \mid z_{1:n}$; repeating this yields many posterior samples. The same idea underpins the martingale posterior framework of \citet{fong23martingale}, where it is termed \emph{predictive resampling}.

In supervised settings, PMC must simulate \emph{future covariates} from the predictive rule. TabPFN does not model covariates; it only provides $\PP(Y_{n+1} \in \cdot \mid X_{n+1}=x, z_{1:n})$. Modelling high-dimensional covariates separately would be burdensome, and the Bayesian-bootstrap workaround suggested by \citet{fong23martingale} is, at best, an approximation.

\citet{nagler25uncertainty} (NR) circumvent this difficulty by using TabPFN only to \textbf{initialise} the predictive distribution~$P_0$ at a fixed query covariate~$x^\ast$; the subsequent forward simulation proceeds via a Gaussian copula update rule constructed to satisfy the martingale property by design. TabPFN is not used in the forward simulation at all. The resulting UD is therefore for a hybrid rule (TabPFN initialisation + copula dynamics), and the hybrid rule's epistemic component conflates TabPFN's uncertainty with the copula's. Additionally, their method targets a single fixed query covariate~$x^\ast$ (e.g., quantiles $\tilde{F}^{-1}(\alpha \mid x^\ast)$).

\textbf{Implementation.} We use the Type 1 scheduler for the learning rate in the Gaussian copula update, with the bandwidth tuned by the prequential log score. We set the rollout length to $N = n + 1000$ and draw $B = 200$ samples from the martingale posterior.

\textbf{Computational cost.} For a single query~$x^\ast$, the NR procedure performs one TabPFN forward pass to initialise $P_0$ at cost $O(n^2)$ (encoder attention over the context), then draws $B$ rollouts, each extending the sequence from length $n$ to terminal length $N$ via copula updates at $O(1)$ per step. The per-query cost is therefore $O(n^2 + B(N-n))$, and $J$ query points cost $O(Jn^2 + JB(N-n))$ since TabPFN must be re-invoked at each $x^\ast$ and each rollout is query-specific. Our predictive-CLT estimator, by contrast, costs $O(n^3 + n^2 J + n J^2)$ total (Appendix~\ref{app:ud4tabpfn}; dominated by $n^3$ when $J \ll n$): each of the $n$ prefix forward passes simultaneously evaluates all $J$ test covariates inside the same attention call, and no Monte Carlo sampling is required. The two methods thus have different scaling in $J$ and $B$: NR avoids the cubic prefix cost but pays linearly in the number of query points and posterior samples.

\subsection{Credible bands for sampling-based baselines}
For the predictive CLT we use the bands of Appendix~\ref{app:credible_bands}. For the sampling-based baselines (bootstrap and NR), pointwise intervals are the empirical $\alpha/2$ and $1-\alpha/2$ quantiles of the $B$ posterior samples at each grid point, clipped to $[0,1]$. Simultaneous bands replace the Gaussian draws $W^{(\ell)}$ of that recipe with the deviations of the $B$ posterior samples from the band centre on the grid, where the centre is the predictive CDF computed on the observed data for the bootstrap and the mean of the $B$ posterior samples for NR; the non-studentised sup-norm $T^{(\ell)}=\max_j|W^{(\ell)}_j|$ then gives a uniform half-width $c_{1-\alpha}$ about the centre, clipped to $[0,1]$. For NR, the grid values within each posterior sample are generated with shared uniform variates (one uniform sequence per rollout, with each variate updating the predictive CDF over the whole evaluation grid), so each posterior sample is a joint draw of the predictive CDF over the grid and the sup-norm $T^{(\ell)}$ is well defined.

\subsection{Evaluation metrics}
\label{app:eval-metrics}
We assess the quality of the credible bands via their frequentist coverage
rates. The target of the credible bands is the CDF or mass function of the
true data-generating distribution. Specifically, for regression, this target is
$f_{0}(x^\ast) = \PP(Y \leq t \mid X = x^\ast)$ for a fixed $t$. For
classification, it is $f_{0}(x^\ast) = \PP(Y = 1 \mid X = x^\ast)$.

Let $C(x^{\ast}, z_{1:n})$ denote the credible set constructed from the dataset
$z_{1:n}$ at a test covariate $x^{\ast}$. The credible bands are constructed
over a grid $\mathcal{X}$ of $m = 100$ points drawn from a scrambled Sobol
sequence in $[-1, 1]^{10}$. For the Gaussian regression data-generating processes (DGPs), we evaluate at $t=0$, which ensures that the target values $f_0(x)=\PP(Y\le 0\mid X=x)$ span essentially all of $[0,1]$ as $x$ ranges over $\mathcal{X}$, so the evaluation is not concentrated in trivially covered tail regions of the CDF; for the Poisson DGP we evaluate at $t=2$ on the standardised scale of the DGP defined below.
The coverage calculation depends on the type of credible band. For
pointwise credible bands, we compute the coverage frequency at each grid point
$x^{\ast}$ and then average over all points in $\mathcal{X}$:
\begin{equation*}
  \frac{1}{|\mathcal{X}|} \sum_{x^{\ast} \in \mathcal{X}} \frac{1}{50} \sum_{r=1}^{50} \mathds{1}{\left\{f_{0}(x^{\ast}) \in C(x^{\ast}, z^{(r)}_{1:n}) \right\}},
\end{equation*}
where coverage is evaluated over $50$ independent draws of $z_{1:n}$ from the
true data-generating process. For simultaneous credible bands, we compute
the frequency with which the band covers the truth at all grid points in
$\mathcal{X}$:
\begin{equation*}
  \frac{1}{50} \sum_{r=1}^{50} \mathds{1}\left\{\textstyle \bigcap_{x^{\ast} \in \mathcal{X}} \left\{f_{0}(x^{\ast}) \in C(x^{\ast}, z^{(r)}_{1:n}) \right\} \right\}.
\end{equation*}

In addition, we report the average width of the credible band, averaged over
both repetitions and grid points:
\begin{equation*}
  \frac{1}{50} \sum^{50}_{r=1} \frac{1}{|\mathcal{X}|}\sum_{x^{\ast} \in \mathcal{X}} \lvert C(x^{\ast}, z^{(r)}_{1:n}) \rvert.
\end{equation*}

\subsection{Data-generating process}
\label{app:data-details}
This section describes the DGP of $Y$ given a
$d = 10$-dimensional covariate $X \in [-1,1]^{10}$.  Training covariates are
drawn from a scrambled Sobol sequence in $[-1,1]^{10}$.

All multivariate DGPs share a pair of normalised weight vectors
\begin{equation*}
  \mathbf{w}^{(0)} = \frac{1}{\sqrt{385}}(-1,2,-3,4,-5,6,-7,8,-9,10)^\top,
\end{equation*}
and
\begin{equation*}
  \mathbf{w}^{(1)} = \frac{1}{\sqrt{385}}(10,-1,2,-3,4,-5,6,-7,8,-9)^\top,
\end{equation*}
where $\sqrt{385} = \bigl(\sum_{j=1}^{10} j^2\bigr)^{1/2}$ and $\mathbf{w}^{(1)}$
is a cyclic shift of $\mathbf{w}^{(0)}$.

\paragraph{Linear (linear regression)} We sample from
\begin{equation*}
  Y = \sqrt{1.5} \mathbf{w}^{(0)\top} X + \varepsilon, \quad \varepsilon \sim \mathcal{N}\bigl(0, (\sqrt{0.5})^2\bigr).
\end{equation*}

\paragraph{Dependent (linear regression with dependent error)} We
sample from
\begin{equation*}
  Y \mid X \sim \mathcal{N} \left(\sqrt{0.75} \mathbf{w}^{(0)\top} X, \sigma^2(X) \right),
  \quad \sigma(x) = 0.75 + 0.25 \lvert \mathbf{w}^{(1)\top} x \rvert.
\end{equation*}

\paragraph{Poisson (Poisson regression, standardised)}
Let $\lambda(x) = \exp\left(0.5 \mathbf{w}^{(0)\top} x\right)$. Draw
$C \mid X \sim \mathrm{Poisson}(\lambda(X))$ and set
\begin{equation*}
  Y = \frac{C - \lambda(X)}{\sqrt{\lambda(X)}}.
\end{equation*}
\paragraph{Probit (probit mixture)}
Let
\begin{equation*}
    p(x) = 0.5 \Phi \left(1.4 \mathbf{w}^{(0)\top} x \right)
    + 0.5 \Phi \left(1.4 \mathbf{w}^{(1)\top} x \right).
\end{equation*}
The target follows
\begin{equation*}
  Y \mid X \sim
  \begin{cases}
    +1 & \text{with probability } p(X), \\
    -1 & \text{with probability } 1 - p(X).
  \end{cases}
\end{equation*}

\paragraph{Categorical (multinomial logistic regression)}
Let $s_1(x) = \mathbf{w}^{(0)\top} x$ and $s_2(x) = \mathbf{w}^{(1)\top} x$. Then,
\begin{equation*}
  Y \mid X \sim \mathrm{Categorical}\left( \mathrm{softmax}(1.2 s_1, -1.2 s_1, 1.2 s_2, -1.2 s_2) \right).
\end{equation*}

\subsection{Predictive rule}
We use \tabclf{} with Probit and Categorical DGPs, and
\tabreg{} for the remaining DGPs. Coverage is reported at $t = 0$
for Gaussian DGPs and $t = 2$ for Poisson; for Probit and Categorical, coverage
is reported for the class-1 probability $\PP(Y = 1 \mid x)$.
The \nest{} parameter is set to 16 for the coverage experiment.

\subsection{Coverage results at \texorpdfstring{$\alpha = 0.2$}{alpha = 0.2} and additional DGPs}
In addition to Table~\ref{tab:coverage}, we report the coverage for two
classification DGPs (Probit and Categorical) in
Table~\ref{tab:coverage-complete-05} (the numbers from Table~\ref{tab:coverage}
are repeated here for convenience). The Gaussian copula update of
\citet{nagler25uncertainty} is not directly applicable to discrete outcomes (Probit and Categorical), so
the corresponding cells are left blank. We also evaluate the methods at
$\alpha=0.2$ and report the results in Table~\ref{tab:coverage-complete-20}.

\begin{table}
  \centering
  \scriptsize
  \caption{Frequentist coverage and width of pointwise and simultaneous
    95\% credible bands for five DGPs at three sample sizes $n$. Four
    methods: our predictive CLT with $\Vb_n$ or $\Ub_n$, the bootstrap
    (Boot.), and \citet{nagler25uncertainty} (NR); NR is not directly
    applicable to Probit or Categorical.}
  \label{tab:coverage-complete-05}
\setlength{\tabcolsep}{3pt}
\begin{tabular}{llcccccccccccccccc}
  \toprule
  & & \multicolumn{2}{c}{$\mathbf{V}_n$ Point.} & \multicolumn{2}{c}{$\mathbf{U}_n$ Point.} & \multicolumn{2}{c}{Boot. Point.} & \multicolumn{2}{c}{NR Point.} & \multicolumn{2}{c}{$\mathbf{V}_n$ Simul.} & \multicolumn{2}{c}{$\mathbf{U}_n$ Simul.} & \multicolumn{2}{c}{Boot. Simul.} & \multicolumn{2}{c}{NR Simul.} \\
  \cmidrule(lr){3-4} \cmidrule(lr){5-6} \cmidrule(lr){7-8} \cmidrule(lr){9-10} \cmidrule(lr){11-12} \cmidrule(lr){13-14} \cmidrule(lr){15-16} \cmidrule(lr){17-18}
  DGP & $n$ & Rate & Width & Rate & Width & Rate & Width & Rate & Width & Rate & Width & Rate & Width & Rate & Width & Rate & Width \\
  \midrule
  \multirow{3}{*}{Linear}      & 200  & 0.97 & 0.39 & 0.98 & 0.47 & 0.93 & 0.42 & 0.98 & 0.85 & 0.88 & 0.67 & 0.98 & 0.81 & 1.00 & 0.82 & 1.00 & 0.98 \\
                               & 500  & 0.99 & 0.30 & 1.00 & 0.62 & 0.93 & 0.34 & 0.99 & 0.70 & 1.00 & 0.53 & 1.00 & 1.04 & 1.00 & 0.76 & 1.00 & 0.98 \\
                               & 1000 & 1.00 & 0.25 & 1.00 & 0.53 & 0.91 & 0.30 & 1.00 & 0.55 & 1.00 & 0.44 & 1.00 & 0.89 & 1.00 & 0.73 & 1.00 & 0.98 \\
  \midrule
  \multirow{3}{*}{Dependent}   & 200  & 0.90 & 0.38 & 0.96 & 0.51 & 0.94 & 0.47 & 0.97 & 0.89 & 0.68 & 0.65 & 0.90 & 0.88 & 1.00 & 0.87 & 0.98 & 0.96 \\
                               & 500  & 0.97 & 0.33 & 1.00 & 0.73 & 0.93 & 0.39 & 0.99 & 0.68 & 0.92 & 0.59 & 1.00 & 1.24 & 1.00 & 0.83 & 1.00 & 0.97 \\
                               & 1000 & 0.99 & 0.29 & 1.00 & 0.64 & 0.92 & 0.35 & 1.00 & 0.50 & 1.00 & 0.50 & 1.00 & 1.08 & 1.00 & 0.80 & 1.00 & 0.95 \\
  \midrule
  \multirow{3}{*}{Poisson}     & 200  & 0.87 & 0.07 & 0.90 & 0.08 & 0.90 & 0.13 & 0.30 & 0.10 & 0.48 & 0.11 & 0.68 & 0.14 & 1.00 & 0.45 & 0.10 & 0.10 \\
                               & 500  & 0.90 & 0.07 & 0.99 & 0.18 & 0.86 & 0.11 & 0.12 & 0.01 & 0.62 & 0.13 & 0.94 & 0.30 & 1.00 & 0.39 & 0.00 & 0.02 \\
                               & 1000 & 0.98 & 0.10 & 1.00 & 0.25 & 0.85 & 0.10 & 0.07 & 0.00 & 0.96 & 0.18 & 1.00 & 0.43 & 1.00 & 0.38 & 0.00 & 0.01 \\
  \midrule
  \multirow{3}{*}{Probit}      & 200  & 0.95 & 0.30 & 0.94 & 0.35 & 0.94 & 0.51 & & & 0.90 & 0.51 & 0.88 & 0.59 & 1.00 & 0.88 & & \\
                               & 500  & 0.95 & 0.25 & 0.99 & 0.39 & 0.93 & 0.47 & & & 0.88 & 0.44 & 0.98 & 0.67 & 1.00 & 0.88 & & \\
                               & 1000 & 0.98 & 0.22 & 0.99 & 0.34 & 0.92 & 0.44 & & & 1.00 & 0.39 & 0.98 & 0.59 & 1.00 & 0.88 & & \\
  \midrule
  \multirow{3}{*}{Categorical} & 200  & 0.87 & 0.30 & 0.95 & 0.40 & 0.84 & 0.32 & & & 0.62 & 0.52 & 0.80 & 0.69 & 1.00 & 0.71 & & \\
                               & 500  & 0.97 & 0.27 & 1.00 & 0.47 & 0.88 & 0.29 & & & 0.94 & 0.47 & 1.00 & 0.81 & 1.00 & 0.70 & & \\
                               & 1000 & 0.99 & 0.25 & 1.00 & 0.41 & 0.90 & 0.28 & & & 1.00 & 0.44 & 1.00 & 0.72 & 1.00 & 0.70 & & \\
  \bottomrule
\end{tabular}
\end{table}

\begin{table}
  \centering
  \scriptsize
  \caption{Frequentist coverage and width of pointwise and simultaneous
    80\% credible bands for five DGPs at three sample sizes $n$. Four
    methods: our predictive CLT with $\Vb_n$ or $\Ub_n$, the bootstrap
    (Boot.), and \citet{nagler25uncertainty} (NR); NR is not directly
    applicable to Probit or Categorical.}
  \label{tab:coverage-complete-20}
\setlength{\tabcolsep}{3pt}
\begin{tabular}{llcccccccccccccccc}
  \toprule
  & & \multicolumn{2}{c}{$\mathbf{V}_n$ Point.} & \multicolumn{2}{c}{$\mathbf{U}_n$ Point.} & \multicolumn{2}{c}{Boot. Point.} & \multicolumn{2}{c}{NR Point.} & \multicolumn{2}{c}{$\mathbf{V}_n$ Simul.} & \multicolumn{2}{c}{$\mathbf{U}_n$ Simul.} & \multicolumn{2}{c}{Boot. Simul.} & \multicolumn{2}{c}{NR Simul.} \\
  \cmidrule(lr){3-4} \cmidrule(lr){5-6} \cmidrule(lr){7-8} \cmidrule(lr){9-10} \cmidrule(lr){11-12} \cmidrule(lr){13-14} \cmidrule(lr){15-16} \cmidrule(lr){17-18}
  DGP & $n$ & Rate & Width & Rate & Width & Rate & Width & Rate & Width & Rate & Width & Rate & Width & Rate & Width & Rate & Width \\
  \midrule
  \multirow{3}{*}{Linear}      & 200  & 0.85 & 0.25 & 0.91 & 0.31 & 0.77 & 0.29 & 0.91 & 0.70 & 0.78 & 0.58 & 0.90 & 0.69 & 1.00 & 0.77 & 1.00 & 0.96 \\
                               & 500  & 0.93 & 0.20 & 0.99 & 0.40 & 0.77 & 0.23 & 0.97 & 0.55 & 0.96 & 0.45 & 1.00 & 0.88 & 1.00 & 0.70 & 1.00 & 0.93 \\
                               & 1000 & 0.96 & 0.16 & 1.00 & 0.35 & 0.74 & 0.20 & 0.99 & 0.42 & 1.00 & 0.38 & 1.00 & 0.76 & 1.00 & 0.67 & 1.00 & 0.88 \\
  \midrule
  \multirow{3}{*}{Dependent}   & 200  & 0.72 & 0.25 & 0.83 & 0.34 & 0.79 & 0.32 & 0.96 & 0.85 & 0.48 & 0.55 & 0.84 & 0.75 & 1.00 & 0.82 & 0.92 & 0.94 \\
                               & 500  & 0.88 & 0.22 & 0.99 & 0.48 & 0.77 & 0.27 & 0.97 & 0.57 & 0.92 & 0.50 & 1.00 & 1.05 & 1.00 & 0.77 & 0.98 & 0.92 \\
                               & 1000 & 0.93 & 0.19 & 1.00 & 0.42 & 0.76 & 0.24 & 0.98 & 0.40 & 0.98 & 0.44 & 1.00 & 0.92 & 1.00 & 0.74 & 1.00 & 0.87 \\
  \midrule
  \multirow{3}{*}{Poisson}     & 200  & 0.69 & 0.05 & 0.75 & 0.05 & 0.74 & 0.09 & 0.15 & 0.01 & 0.20 & 0.09 & 0.44 & 0.11 & 1.00 & 0.36 & 0.04 & 0.03 \\
                               & 500  & 0.70 & 0.05 & 0.94 & 0.11 & 0.70 & 0.08 & 0.08 & 0.00 & 0.52 & 0.11 & 0.86 & 0.25 & 1.00 & 0.32 & 0.00 & 0.01 \\
                               & 1000 & 0.84 & 0.07 & 0.99 & 0.16 & 0.66 & 0.07 & 0.04 & 0.00 & 0.94 & 0.16 & 1.00 & 0.37 & 1.00 & 0.31 & 0.00 & 0.01 \\
  \midrule
  \multirow{3}{*}{Probit}      & 200  & 0.81 & 0.20 & 0.83 & 0.23 & 0.80 & 0.36 & & & 0.66 & 0.44 & 0.64 & 0.49 & 1.00 & 0.84 & & \\
                               & 500  & 0.84 & 0.16 & 0.93 & 0.26 & 0.77 & 0.33 & & & 0.74 & 0.38 & 0.92 & 0.57 & 1.00 & 0.83 & & \\
                               & 1000 & 0.91 & 0.15 & 0.97 & 0.22 & 0.77 & 0.30 & & & 0.90 & 0.34 & 0.96 & 0.51 & 1.00 & 0.83 & & \\
  \midrule
  \multirow{3}{*}{Categorical} & 200  & 0.69 & 0.19 & 0.82 & 0.26 & 0.65 & 0.22 & & & 0.34 & 0.44 & 0.66 & 0.59 & 1.00 & 0.65 & & \\
                               & 500  & 0.86 & 0.18 & 0.97 & 0.31 & 0.70 & 0.20 & & & 0.84 & 0.41 & 0.96 & 0.70 & 1.00 & 0.64 & & \\
                               & 1000 & 0.94 & 0.16 & 0.99 & 0.27 & 0.73 & 0.19 & & & 0.98 & 0.38 & 1.00 & 0.62 & 1.00 & 0.64 & & \\
  \bottomrule
\end{tabular}
\end{table}

\FloatBarrier
\section{Gap experiments}
\label{app:gap}
We illustrate credible bands constructed using the predictive CLT in
Figures~\ref{fig:gap-gaussian-linear}--\ref{fig:gap-categorical-linear}. We draw
half of the sample of $x$ from $\mathrm{Uniform}(-8, -2)$ and the other half
from $\mathrm{Uniform}(2, 8)$. The targets $y$ are drawn from the following
data-generating process:

\paragraph{Linear (simple linear regression)}
\[
  Y = 0.2X + \epsilon, \quad \epsilon \sim \mathcal{N}(0, 1).
\]

\paragraph{Polynomial (polynomial regression)}
\[
  Y = 1 - 0.03X^2 + \epsilon, \quad \epsilon \sim \mathcal{N}(0, 1).
\]

\paragraph{Dependent (linear regression with dependent error)}
\[
  Y \mid X \sim \mathcal{N}(0.5X + 1, \sigma^2(X)) \quad \text{where} \quad \sigma(x) = 0.5 + 0.5|x|.
\]

\paragraph{Sine (sine wave with Gaussian noise)}
\[
  Y = 0.5\sin(X/2) + \epsilon, \quad \epsilon \sim \mathcal{N}(0, 0.5^2).
\]

\paragraph{Poisson (Poisson regression)}
\[
  Y \mid X \sim \mathrm{Poisson}(\lambda(X)) \quad \text{where} \quad \lambda(x) = 0.05(x^2 - 80) + 5 = 0.05x^2 + 1.
\]

\paragraph{Probit (probit regression)}
\[
  Y \mid X \sim \mathrm{Bernoulli}(p), \quad p = 0.6\Phi\left(\frac{x - 8}{4}\right) + 0.4\Phi\left(\frac{x + 8}{4}\right).
\]

\paragraph{Categorical (multinomial logistic regression)}
\[
  Y \mid X \sim \mathrm{Categorical}(p_0, p_1, p_2, p_3), \quad p_j = \frac{\exp(z_j(x))}{\sum_{i=0}^3 \exp(z_i(x))}, \quad \text{for } j \in \{0, 1, 2, 3\},
\]
where
\begin{equation*}
  z_0(x) = -\frac{(x+5)^2}{10}, \quad z_1(x) = -\frac{x^2}{30},
  z_2(x) = -\frac{(x-7)^2}{5}, \quad z_3(x) = -\frac{(x-4)^2}{8}.
\end{equation*}

We use \tabclf{} with Probit and Categorical DGPs, and
\tabreg{} for the remaining DGPs. We set $t = 0$ for the Gaussian
DGPs and $t = 2$ for the Poisson DGP. For the Probit and Categorical DGPs we
plot the class-1 probability $\PP(Y = 1 \mid x)$. Across the figures, the
credible bands widen visibly inside the gap on all DGPs except Probit, where
the response is muted.

\begin{figure}[p]
  \centering
  \includegraphics[width=\linewidth]{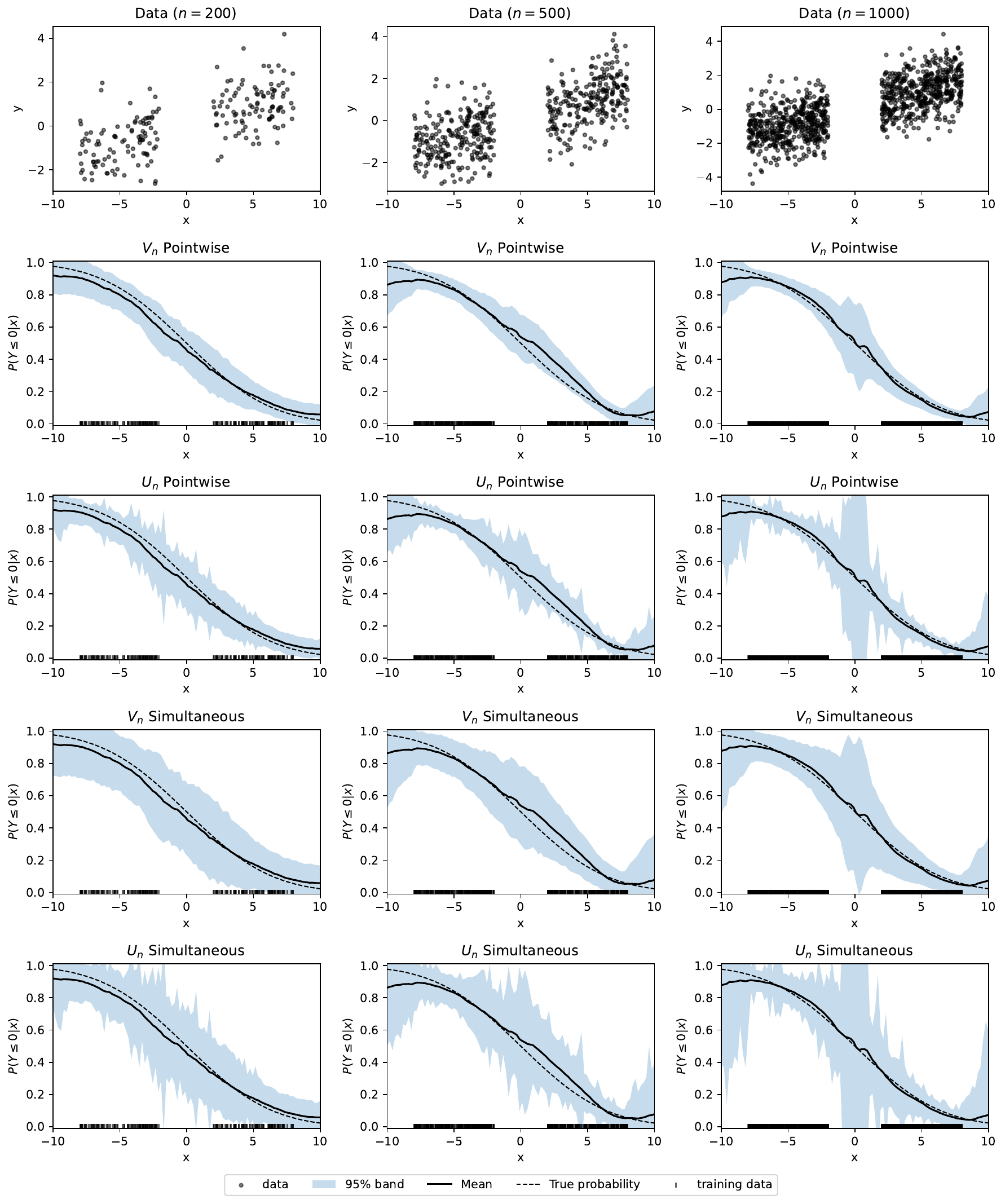}
  \caption{Credible bands with the Linear DGP under different combinations of
    construction methods (pointwise and simultaneous) and covariance estimators
    ($\Vb_n$ and $\Ub_n$), with a gap in observations between $-2$ and $2$. Each
    column corresponds to a different sample size $n$. The first row shows the
    sampled dataset, and the subsequent rows show the corresponding credible
    bands. The observed $x$ values are marked in the credible band plots. }
    \label{fig:gap-gaussian-linear}
\end{figure}

\begin{figure}[p]
  \centering
  \includegraphics[width=\linewidth]{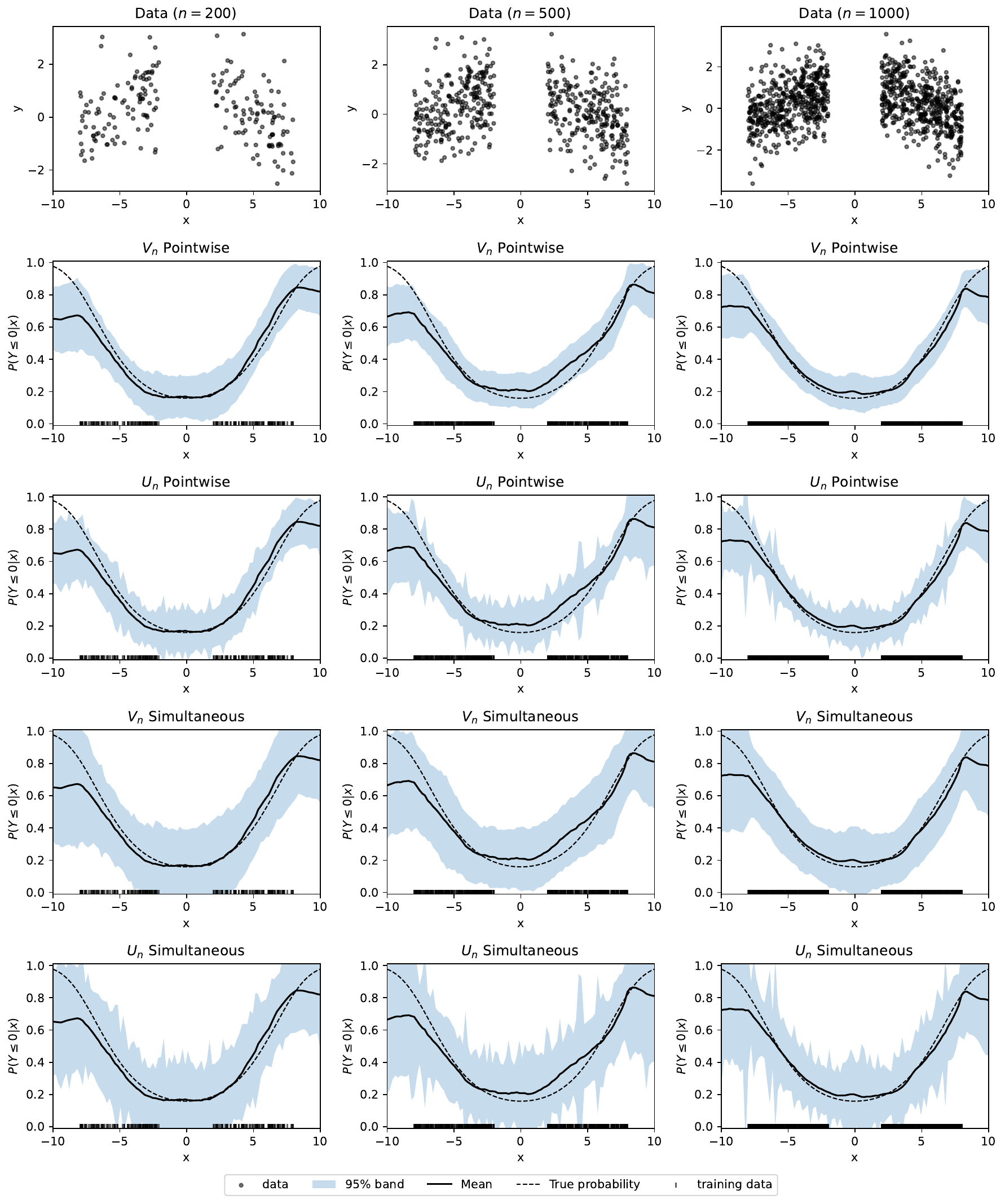}
  \caption{Credible bands with the Polynomial DGP under different combinations of
    construction methods (pointwise and simultaneous) and covariance estimators
    ($\Vb_n$ and $\Ub_n$), with a gap in observations between $-2$ and $2$. Each
    column corresponds to a different sample size $n$. The first row shows the
    sampled dataset, and the subsequent rows show the corresponding credible
    bands. The observed $x$ values are marked in the credible band plots. }
    \label{fig:gap-gaussian-polynomial}
\end{figure}

\begin{figure}[p]
  \centering
  \includegraphics[width=\linewidth]{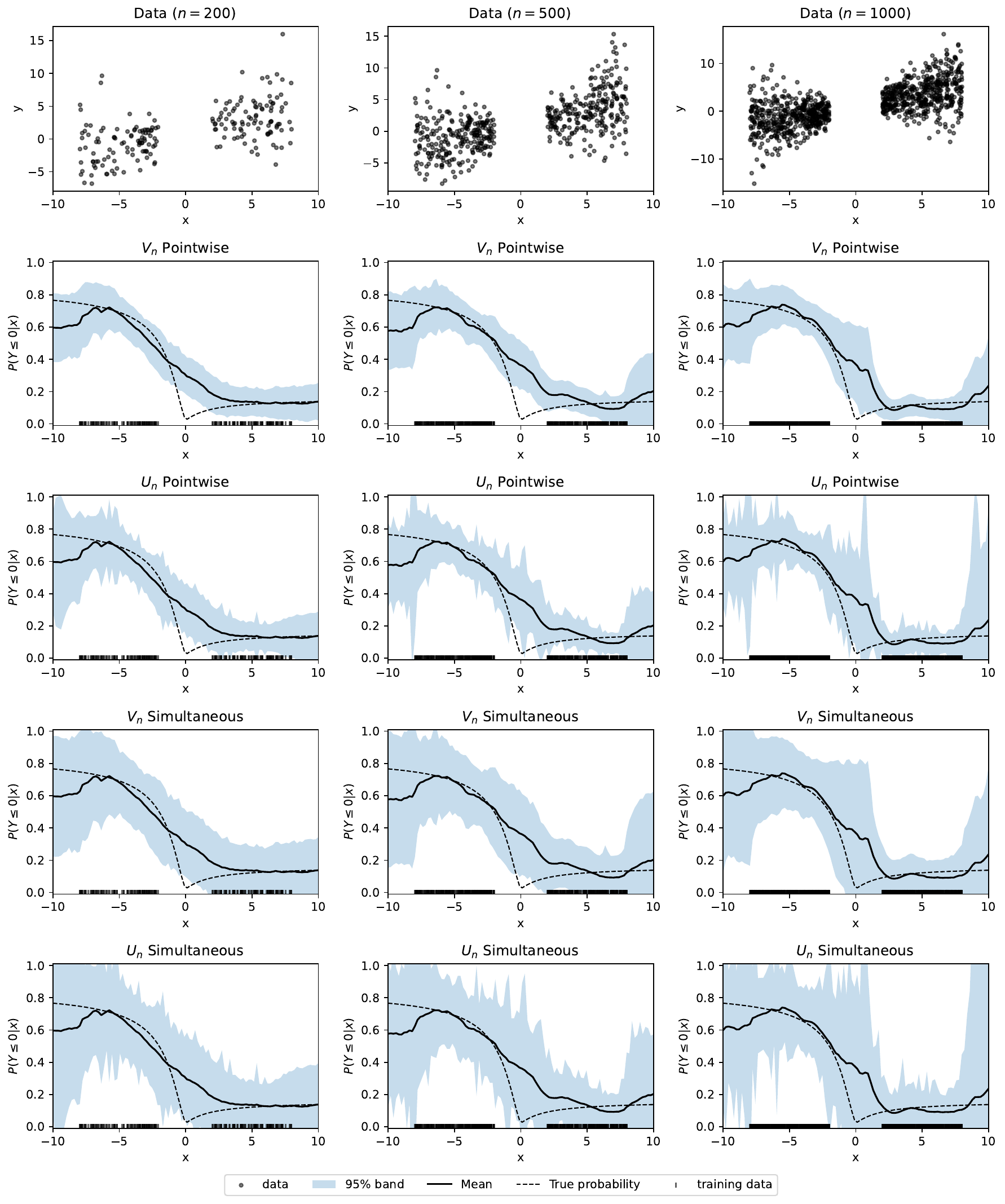}
  \caption{Credible bands with the Dependent DGP under different combinations of
    construction methods (pointwise and simultaneous) and covariance estimators
    ($\Vb_n$ and $\Ub_n$), with a gap in observations between $-2$ and $2$. Each
    column corresponds to a different sample size $n$. The first row shows the
    sampled dataset, and the subsequent rows show the corresponding credible
    bands. The observed $x$ values are marked in the credible band plots. }
    \label{fig:gap-gaussian-linear-dependent-error}
\end{figure}

\begin{figure}[p]
  \centering
  \includegraphics[width=\linewidth]{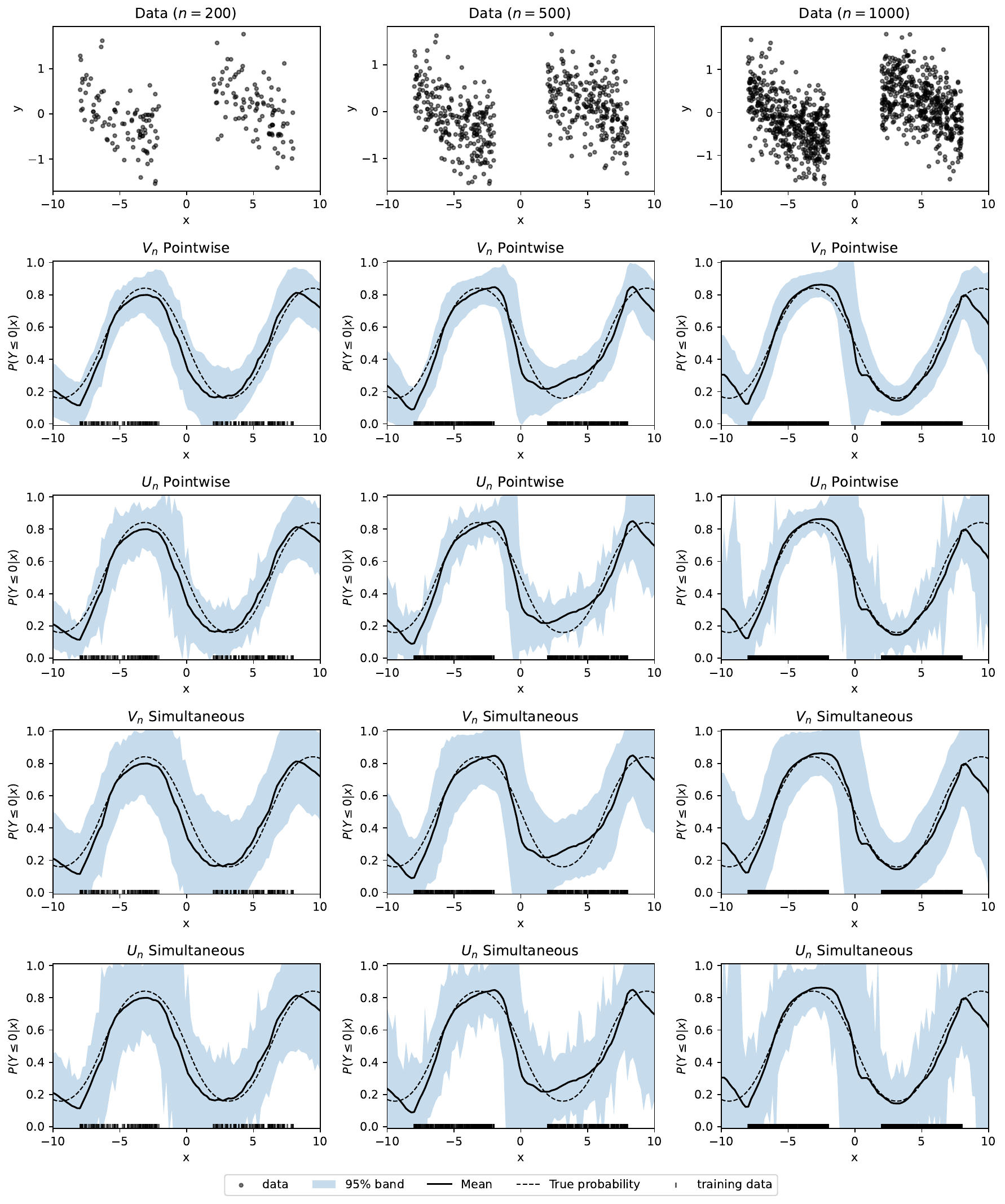}
  \caption{Credible bands with the Sine DGP under different combinations of
    construction methods (pointwise and simultaneous) and covariance estimators
    ($\Vb_n$ and $\Ub_n$), with a gap in observations between $-2$ and $2$. Each
    column corresponds to a different sample size $n$. The first row shows the
    sampled dataset, and the subsequent rows show the corresponding credible
    bands. The observed $x$ values are marked in the credible band plots. }
\end{figure}

\begin{figure}[p]
  \centering
  \includegraphics[width=\linewidth]{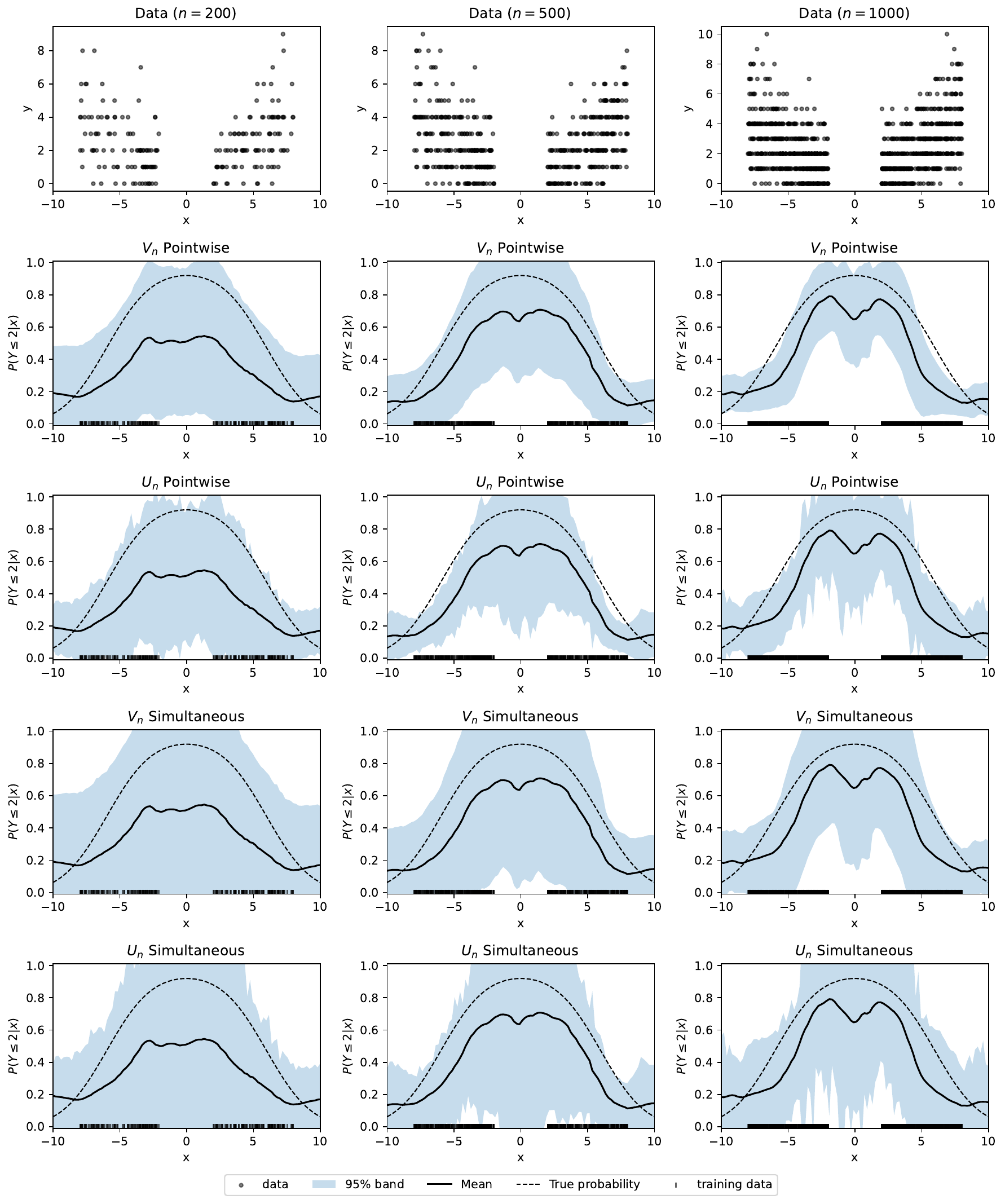}
  \caption{Credible bands with the Poisson DGP under different combinations of
    construction methods (pointwise and simultaneous) and covariance estimators
    ($\Vb_n$ and $\Ub_n$), with a gap in observations between $-2$ and $2$. Each
    column corresponds to a different sample size $n$. The first row shows the
    sampled dataset, and the subsequent rows show the corresponding credible
    bands. The observed $x$ values are marked in the credible band plots. }
\end{figure}

\begin{figure}[p]
  \centering
  \includegraphics[width=\linewidth]{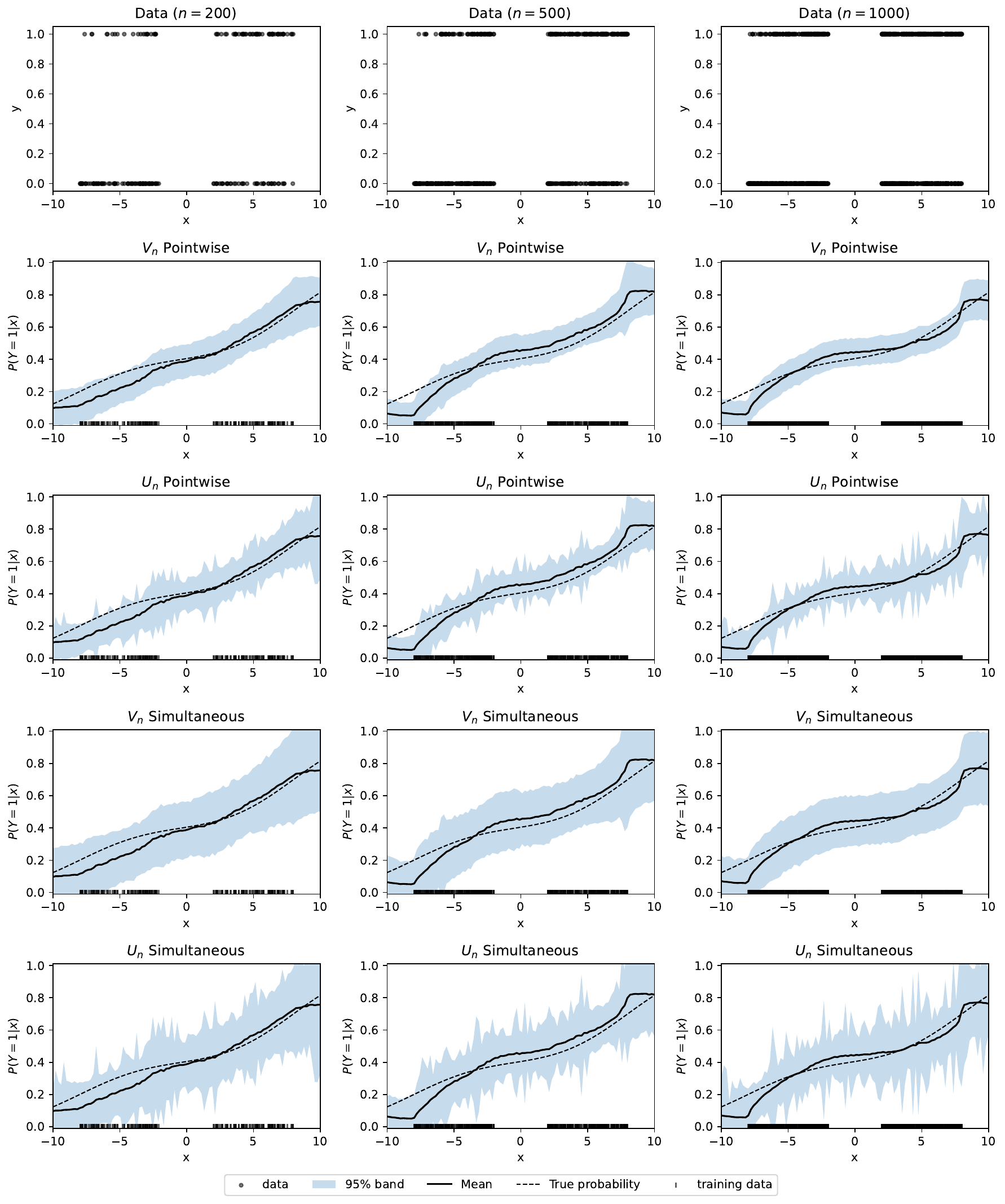}
  \caption{Credible bands with the Probit DGP under different combinations of
    construction methods (pointwise and simultaneous) and covariance estimators
    ($\Vb_n$ and $\Ub_n$), with a gap in observations between $-2$ and $2$. Each
    column corresponds to a different sample size $n$. The first row shows the
    sampled dataset, and the subsequent rows show the corresponding credible
    bands. The observed $x$ values are marked in the credible band plots. }
\end{figure}

\begin{figure}[p]
  \centering
  \includegraphics[width=\linewidth]{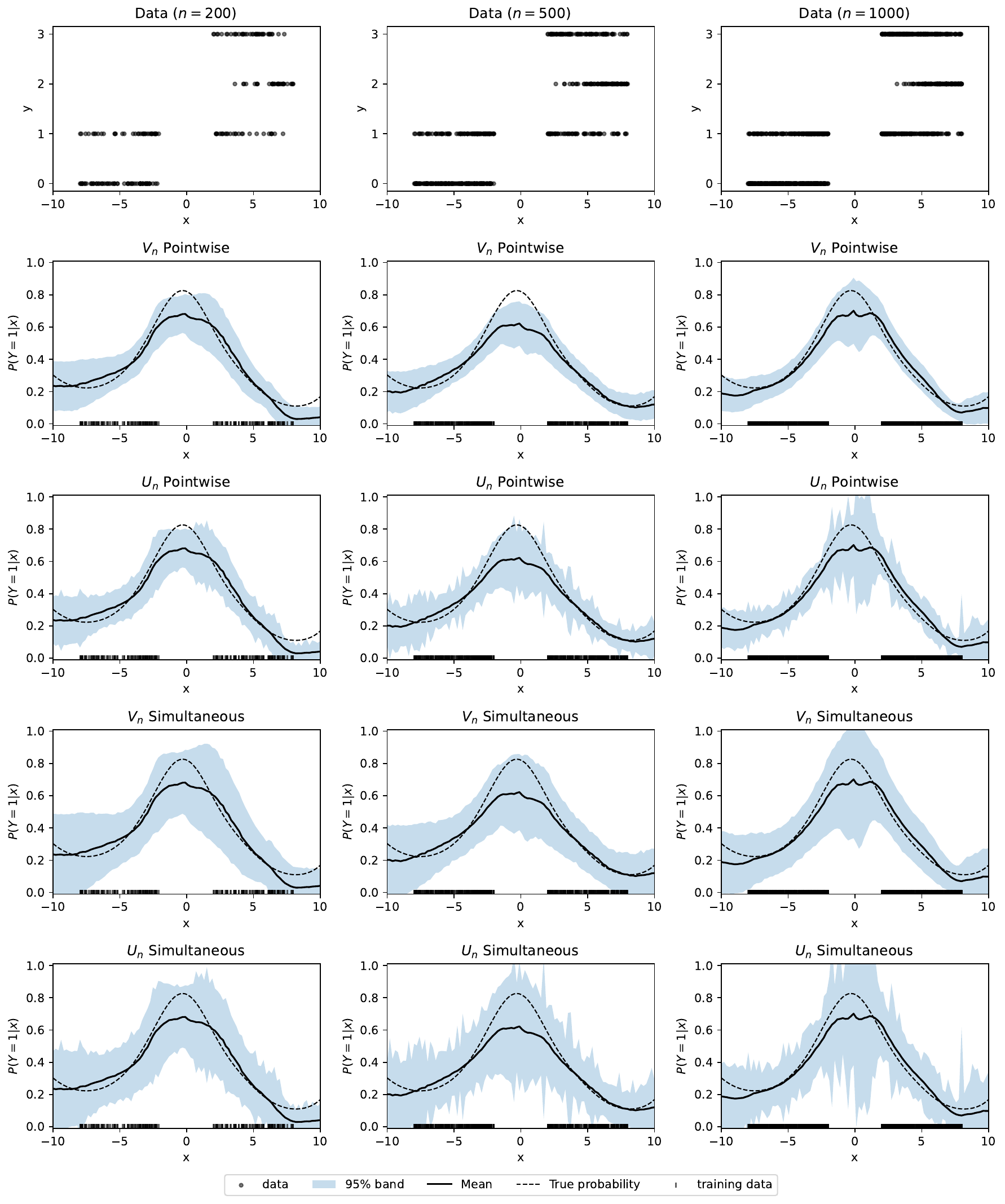}
  \caption{Credible bands with the Categorical DGP under different combinations
    of construction methods (pointwise and simultaneous) and covariance
    estimators ($\Vb_n$ and $\Ub_n$), with a gap in observations between $-2$
    and $2$. Each column corresponds to a different sample size $n$. The first
    row shows the sampled dataset, and the subsequent rows show the
    corresponding credible bands. The observed $x$ values are marked in the
    credible band plots. }
    \label{fig:gap-categorical-linear}
\end{figure}

\FloatBarrier
\section{Entropic UD}
\label{app:entropic_ud}

\subsection{Derivations}
\label{app:entropy-decomp}

Fix a test point $x^\ast$. Our goal is to approximate the aleatoric uncertainty component:
\[
  U_a(x^\ast, z_{1:n}) = \EE\bigl[h(\tilde g(x^\ast))\mid z_{1:n}\bigr],
\]
where $h(p) = -p\log p - (1-p)\log(1-p)$ is the binary entropy function. The Gaussian approximation provided by the predictive CLT, $\tilde{g}(x^\ast) \mid z_{1:n} \approx \mathcal{N}(g_n(x^\ast), v_n(x^\ast)/n)$, where $v_n(x^\ast)$ denotes the diagonal entry of $\Vb_n$ in \eqref{eq:Vnmult} (at $\gamma=1$) corresponding to the test point $x^\ast$, does not respect the support constraint $\tilde{g}(x^\ast) \in [0,1]$ and is therefore unsuitable for directly computing the expectation of the entropy.

To respect the support constraint we approximate the posterior of $\tilde g(x^\ast)\mid z_{1:n}$ by a Beta distribution moment-matched to the CLT mean and variance; its parameters $\alpha_n(x^\ast)$ and $\beta_n(x^\ast)$ are constructed in the subsubsection below. Let $\psi$ denote the digamma function. Writing $a=\alpha_n(x^\ast)$ and $b=\beta_n(x^\ast)$ for brevity, our estimator of aleatoric (entropic) uncertainty at $x^\ast$ is
\begin{equation}
  \widehat U_a(x^\ast, z_{1:n})
  := -\frac{a}{a+b}\,\psi(a+1)
     -\frac{b}{a+b}\,\psi(b+1)
     +\psi(a+b+1).
  \label{eq:hatUa}
\end{equation}

\subsubsection{Moment-matched Beta approximation}
Let $\sigma^2_n(x^\ast)$ denote the plug-in variance $v_n(x^\ast)/n$, clipped to be strictly less than $g_n(x^\ast)\bigl(1-g_n(x^\ast)\bigr)$; the strict clipping guarantees $T_n(x^\ast)>0$ below. We determine the parameters $\alpha_n = \alpha_n(x^\ast)$ and $\beta_n = \beta_n(x^\ast)$ of the approximating Beta distribution $G \sim \mathrm{Beta}(\alpha_n, \beta_n)$ by matching its first two moments to the CLT mean and the clipped variance:
\begin{align*}
  \EE[G]   & = \frac{\alpha_n}{\alpha_n + \beta_n} = g_n(x^\ast),                                             \\
  \mathrm{Var}(G) & = \frac{\alpha_n \beta_n}{(\alpha_n + \beta_n)^2(\alpha_n + \beta_n + 1)}=\sigma^2_n(x^\ast).
\end{align*}
Solving this system for $\alpha_n$ and $\beta_n$ yields:
\begin{align*}
  \alpha_n & = g_n(x^\ast) T_n(x^\ast),     \\
  \beta_n  & = (1-g_n(x^\ast)) T_n(x^\ast),
\end{align*}
where
\[
  T_n(x^\ast) = \frac{g_n(x^\ast)\bigl(1-g_n(x^\ast)\bigr)}{\sigma^2_n(x^\ast)} - 1.
\]
For $g_n(x^\ast)\in(0,1)$ and $T_n(x^\ast)>0$, this Beta distribution has mean $g_n(x^\ast)$ and variance $\sigma^2_n(x^\ast)$.

For $G \sim \mathrm{Beta}(\alpha_n, \beta_n)$, the identity $\EE[G\,f(G)] = \frac{\alpha_n}{\alpha_n+\beta_n}\,\EE_{\mathrm{Beta}(\alpha_n+1,\beta_n)}[f(G)]$ together with $\EE_{\mathrm{Beta}(a,b)}[\log G] = \psi(a)-\psi(a+b)$ gives
\[
  \EE[G \log G] = \frac{\alpha_n}{\alpha_n+\beta_n}\bigl(\psi(\alpha_n+1) - \psi(\alpha_n+\beta_n+1)\bigr),
\]
and analogously $\EE[(1-G)\log(1-G)] = \tfrac{\beta_n}{\alpha_n+\beta_n}\bigl(\psi(\beta_n+1) - \psi(\alpha_n+\beta_n+1)\bigr)$. Substituting into $\EE[h(G)] = -\EE[G\log G] - \EE[(1-G)\log(1-G)]$ yields the closed-form estimator \eqref{eq:hatUa}.

\paragraph{Asymptotic justification.}
The CLT (Theorem~\ref{th:ascondmult}) supplies an asymptotic Gaussian for the conditional law of $\tilde g(x^\ast) \mid z_{1:n}$, but the limiting predictive is only guaranteed to be asymptotically Gaussian; at finite $n$ both the true posterior and our Beta approximation depart from the Gaussian. The Beta matches the CLT mean and variance and respects the support $[0,1]$, but its higher moments are fixed by the Beta family. Alternative moment-matched families on $[0,1]$ (e.g., a logit-Gaussian) would yield slightly different entropy estimates in finite samples. As $v_n(x^\ast)/n \to 0$, every such approximation collapses to the point mass at $\lim_n g_n(x^\ast)$, so the bias of $\widehat U_a$ relative to the true expected entropy vanishes as $n \to \infty$; we do not characterise its rate.

\subsubsection{Comparison with Delta method}
A simpler alternative would be a second-order Delta method expansion around the mean $g_n(x^\ast)$. Since $h''(p) = -1/(p(1-p))$, this yields:
\begin{equation}
  U_a(x^\ast, z_{1:n}) \approx h(g_n(x^\ast)) - \frac{v_n(x^\ast)/n}{2 g_n(x^\ast)(1-g_n(x^\ast))}.
\end{equation}
However, this approximation is numerically unstable near the boundaries (where $g_n(x^\ast) \approx 0$ or $1$), as the curvature $h''$ explodes. Furthermore, this approximation is not guaranteed to be nonnegative, potentially yielding negative estimates of the aleatoric uncertainty $U_a(x^\ast, z_{1:n})$ (equivalently, epistemic estimates exceeding the total entropy). The moment-matched Beta approach avoids these pathologies by implicitly respecting the compact support of $\tilde{g}(x^\ast)$.

\subsubsection{Extension to multiclass classification}
For the multiclass setting with $K$ classes, let the predictive distribution be
denoted by the probability vector
$g_n(x^\ast) = [g_{n,1}(x^\ast), \dots, g_{n,K}(x^\ast)]^\top$, where
$\sum_{k=1}^K g_{n,k}(x^\ast) = 1$. The total predictive uncertainty is given by
the Shannon entropy:
\begin{equation}
  \mathbb H(Y^\ast \mid x^\ast, z_{1:n}) = -\sum_{k=1}^K g_{n,k}(x^\ast) \log g_{n,k}(x^\ast).
\end{equation}
We decompose this into aleatoric and epistemic components. 
To estimate the aleatoric uncertainty
$U_a(x^\ast, z_{1:n})$, we generalise the Beta approximation to a Dirichlet
approximation $\tilde g(x^\ast) \sim \mathrm{Dir}(\alpha_{n}(x^\ast))$, where
$\alpha_n(x^\ast) = [\alpha_{n,1}(x^\ast), \dots, \alpha_{n,K}(x^\ast)]^\top$.

We determine the parameters via moment matching. Let $\sigma^2_{n,k}(x^\ast)$
denote the variance for class $k$ provided by the predictive CLT, analogous to
the binary case. We define the term $\alpha_{n,0}(x^{\ast})$ by matching the
total variance of the Dirichlet distribution to the sum of the predictive
variances:
\begin{equation}
  \sum_{k=1}^K \frac{g_{n,k}(x^\ast) (1 - g_{n,k}(x^\ast))}{\alpha_{n,0}(x^{\ast}) + 1} = \sum_{k=1}^K \sigma^2_{n,k}(x^\ast).
\end{equation}
Rearranging the terms yields
\begin{equation}
  \alpha_{n,0}(x^\ast) := \frac{1 - \|g_n(x^\ast)\|_2^2}{\sum_{k=1}^K \sigma^2_{n,k}(x^\ast)} - 1.
\end{equation}
The sum of the predictive variances $\sum_{k=1}^K \sigma^2_{n,k}(x^\ast)$ is
clipped to be strictly less than $1 - \|g_n(x^\ast)\|_2^2$ to ensure
$\alpha_{n,0}(x^\ast) > 0$. The Dirichlet parameters are then set as
$\alpha_{n,k}(x^\ast) := g_{n,k}(x^\ast) \, \alpha_{n,0}(x^\ast)$.

The aleatoric uncertainty is estimated as the expected entropy of this Dirichlet
distribution. Using the properties of $\psi$, we obtain the
estimator:
\begin{equation}
  \widehat U_a(x^\ast, z_{1:n})
  := \psi(\alpha_{n,0}(x^\ast) + 1) - \sum_{k=1}^K g_{n,k}(x^\ast) \, \psi(\alpha_{n,k}(x^\ast) + 1).
\end{equation}
Epistemic uncertainty is then computed as the residual
$U_e(x^\ast, z_{1:n}) = \mathbb H(Y^\ast \mid x^\ast, z_{1:n}) - \widehat U_a(x^\ast, z_{1:n})$.

\subsection{Experiments}
\label{app:entropic_ud_experiments}
We take the synthetic experimental setups described in \citet{jayasekera25variational}, Appendix G.3 (``Synthetic Toy Experiments''), including the data-generating processes and evaluation grids for logistic regression, two moons, and spirals.

\subsubsection{Logistic regression}\label{app:logreg}
We consider a one-dimensional logistic-regression data-generating process. For each dataset size $n$, we sample covariates
$x_i \sim \mathcal{N}(1.5, 3.0^2)$ and then draw labels
\[
  y_i \sim \mathrm{Bernoulli}(p_i), \qquad
  p_i = \sigma(\beta x_i + \beta_0),
\]
where $\sigma(\cdot)$ is the logistic sigmoid, $\beta=0.25$, and $\beta_0=-0.5$.
We run this experiment for $n \in \{15,50,75,150\}$ with a fixed random seed.

For visualisation, we evaluate uncertainty on a one-dimensional test grid
$x^{\ast} \in [-15.0,\,15.0]$ with step size $0.2$. The results are in Figure~\ref{fig:ud-logreg-xstar}.


\begin{figure}[h]
  \centering
  \includegraphics[width=\linewidth]{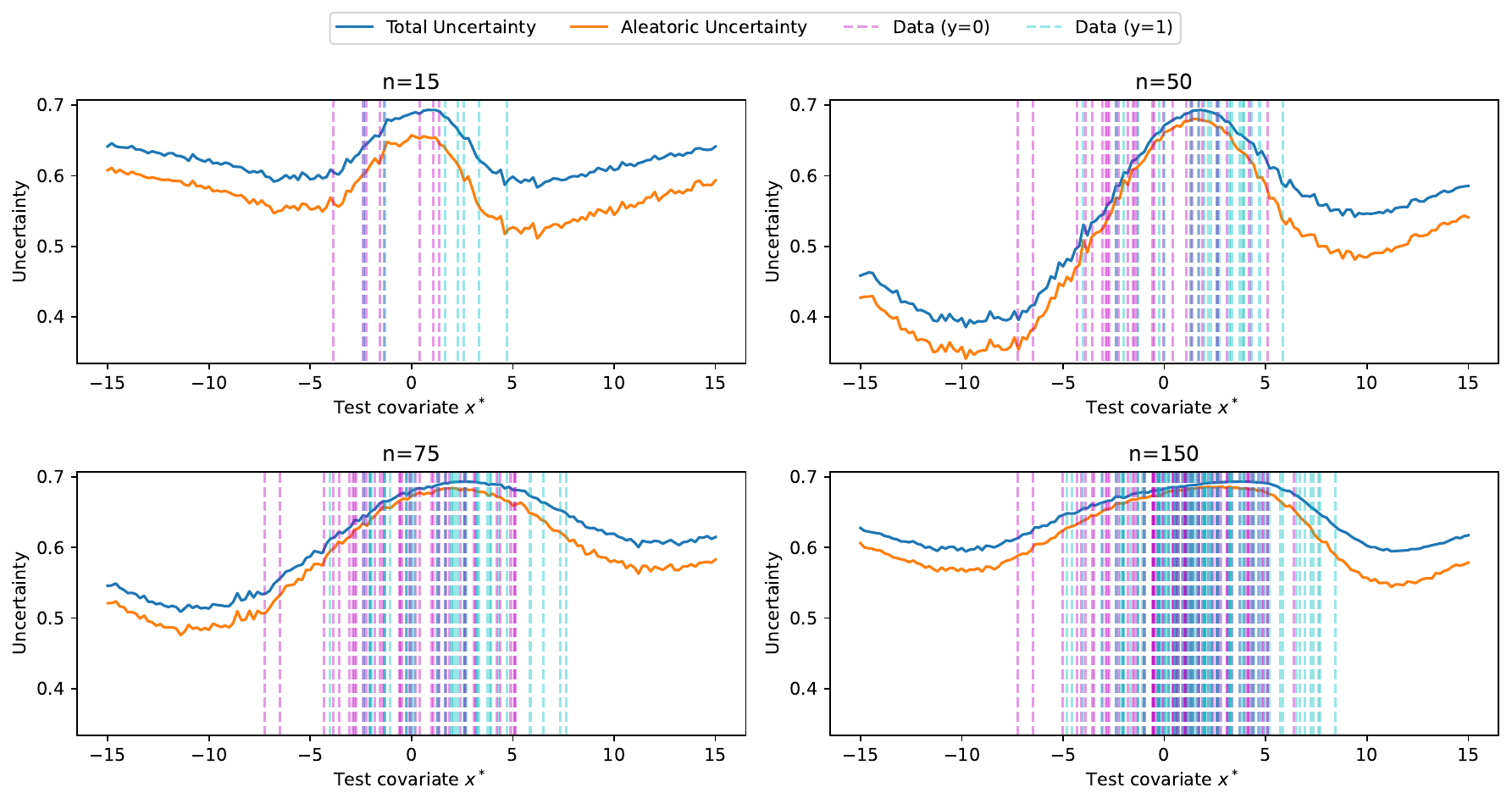}
  \caption{\textbf{Logistic regression (1D).} Entropic UD for TabPFN at
    dataset sizes $n\in\{15,50,75,150\}$.}
  \label{fig:ud-logreg-xstar}
\end{figure}

\paragraph{Size curves.}
To study scaling with context length, we evaluate uncertainties at test points
$x^{\ast} \in \{-15,-10,-5,0,5,10,15\}$ across context lengths
$n \in \{75, 80, 85, \ldots, 200\}$. For each context length $n$, we average results
across $50$ independently sampled datasets. The proportional version of
the aleatoric and epistemic uncertainty is shown in the main text
(Figure~\ref{fig:ud-logreg-context-length}); the raw values are presented
in Figure~\ref{fig:ud-logreg-context-length-raw}.

\begin{figure}[h]
  \centering
  \includegraphics[width=\linewidth]{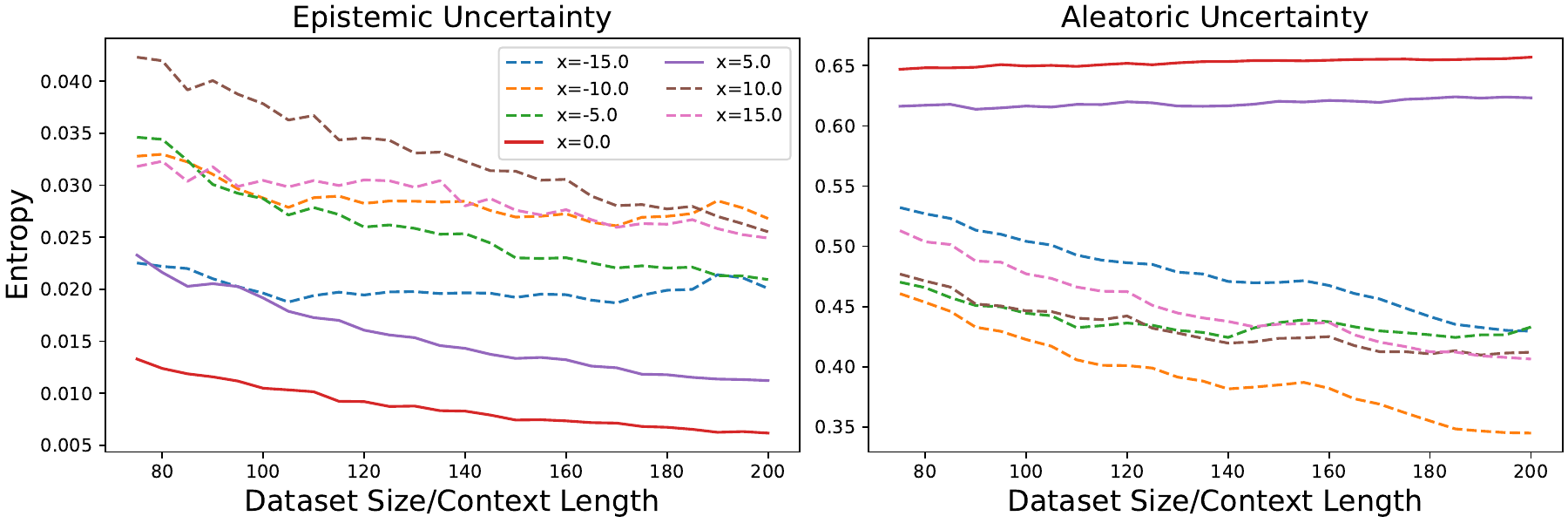}
  \caption{\textbf{Logistic regression (1D).} Raw aleatoric and epistemic
    uncertainty for TabPFN with different test covariates $x^\ast$ against
    varying context length. Solid and dotted lines indicate in-distribution
    and out-of-distribution $x^\ast$, respectively.}
  \label{fig:ud-logreg-context-length-raw}
\end{figure}

\subsubsection{Two moons}\label{app:moons}
We use the \texttt{scikit-learn} \texttt{make\_moons} generator with additive Gaussian noise.
We consider two noise regimes:
(i) \emph{Moons 1} with noise $\sigma=0.1$ and
(ii) \emph{Moons 2} with noise $\sigma=0.4$.
For each regime we run two dataset sizes, $n=30$ and $n=100$, using a fixed seed.
We evaluate uncertainty on two-dimensional uniform grids with 100 points in each
coordinate:
\begin{itemize}
  \item \textbf{Moons 1:} $x_1^{\ast} \in [-1.5,\,2.6]$ and $x_2^{\ast} \in [-1.5,\,2.6]$.
  \item \textbf{Moons 2:} $x_1^{\ast} \in [-3.0,\,3.6]$ and $x_2^{\ast} \in [-2.5,\,3.1]$.
\end{itemize}
We report heatmaps of total, aleatoric, and epistemic uncertainty over each grid, overlaying training samples by class, see Figure~\ref{fig:appendix-moons-grid}.

\subsubsection{Three-class spirals}\label{app:spirals}
We generate a three-class spirals dataset in $\mathbb{R}^2$ with $C=3$ arms. For class $c \in \{0,1,2\}$, we sample
$t \sim \mathrm{Uniform}(0,1)$, set radius $r = 4.0\,t$, and define the angle
\[
  \theta = 2\pi \cdot 2 \cdot t + \frac{2\pi c}{C}.
\]
Points are generated as
\[
  x_1 = r\cos\theta + \varepsilon_1,\qquad x_2 = r\sin\theta + \varepsilon_2,
\]
with i.i.d.\ Gaussian noise $\varepsilon_1,\varepsilon_2 \sim \mathcal{N}(0,\sigma^2)$ using $\sigma=0.1$.
We sample $n=200$ total points (approximately balanced across classes), randomly permute the dataset, and fix the random seed.

We evaluate uncertainty on a two-dimensional uniform grid over
$x_1^{\ast} \in [-4.0,\,4.0]$ and $x_2^{\ast} \in [-4.0,\,4.0]$ with 100 points
in each coordinate, and report heatmaps of total, aleatoric, and epistemic
uncertainty over this region, see Figure~\ref{fig:2moons-spirals}.


\begin{figure}[h]
  \centering
  \includegraphics[width=\linewidth]{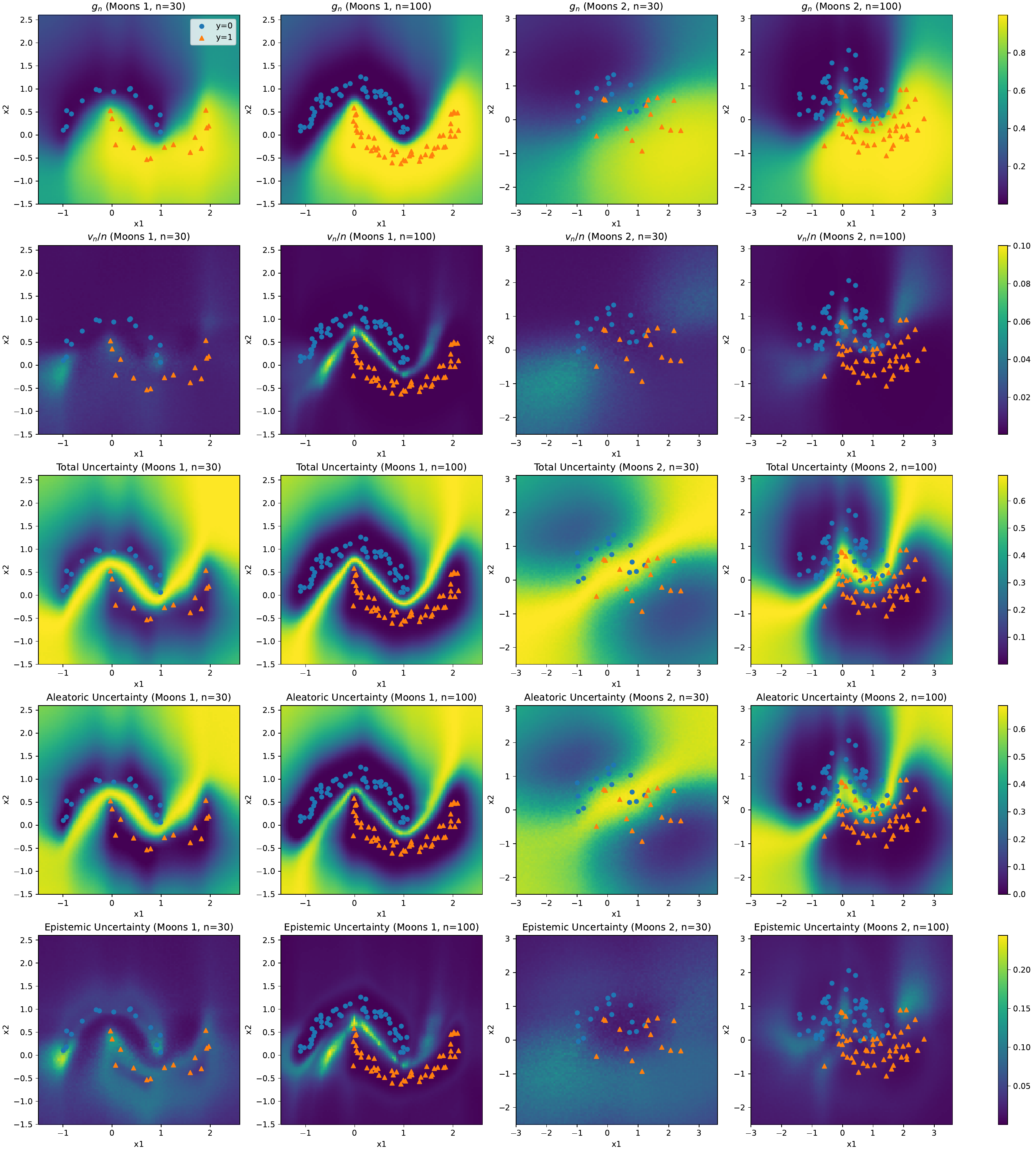} \\
  \caption{\textbf{Two moons.} Rows (top to bottom):
    $g_n(x)$, $v_n(x)/n$, total uncertainty, aleatoric uncertainty, epistemic
    uncertainty. Columns (left to right): Moons 1 ($n=30$), Moons 1 ($n=100$),
    Moons 2 ($n=30$), Moons 2 ($n=100$).}
  \label{fig:appendix-moons-grid}
\end{figure}


%% file: references.bib
@article{berti04limit,
  title = {Limit Theorems for a Class of Identically Distributed Random Variables},
  author = {Berti, Patrizia and Pratelli, Luca and Rigo, Pietro},
  year = 2004,
  journal = {The Annals of Probability},
  volume = {32},
  number = {3},
  pages = {2029--2052},
  publisher = {Institute of Mathematical Statistics}
}

@article{definetti37prevision,
  title = {La pr{\'e}vision~: ses lois logiques, ses sources subjectives},
  author = {de Finetti, Bruno},
  year = 1937,
  journal = {Annales de l'Institut Henri Poincar{\'e}},
  volume = {7},
  number = {1},
  pages = {1--68}
}

@book{albert20probability,
  title = {Probability and {{Bayesian Modeling}}},
  author = {Albert, Jim and Hu, Jingchen},
  year = 2020,
  series = {Texts in {{Statistical Science}}},
  edition = {first},
  publisher = {{Chapman and Hall/CRC}}
}

@unpublished{battiston25bayesian,
  title = {Bayesian Predictive Inference beyond Martingales},
  author = {Battiston, Marco and Cappello, Lorenzo},
  year = 2025,
  number = {arXiv:2507.21874},
  eprint = {2507.21874},
  primaryclass = {math},
  publisher = {arXiv},
  archiveprefix = {arXiv},
  note = {arXiv:2507.21874}
}

@unpublished{fortini26principled,
  title = {A principled framework for uncertainty decomposition in {{TabPFN}}},
  author = {Fortini, Sandra and Ng, Kenyon and Petrone, Sonia and Rousseau, Judith and Wei, Susan},
  year = 2026,
  number = {arXiv:2602.04596v1},
  eprint = {2602.04596v1},
  primaryclass = {stat.ML},
  publisher = {arXiv},
  archiveprefix = {arXiv},
  note = {arXiv:2602.04596v1; earlier version of the present work}
}

@article{crimaldi2009,
  title = {An Almost Sure Conditional Convergence Result and an Application to a Generalized {P}\'olya Urn},
  author = {Crimaldi, Irene},
  year = {2009},
journal={International Mathematical Forum}, 
volume= 4, 
number={23},
pages={1139-1156}
}

@article{crimaldi2016,
  title = {Fluctuation theorems for synchronization of interacting {P}\'olya's urns},
  author = {Crimaldi, Irene and Dai Pra, Paolo and Minelli, Ida Germana},
  journal = {Stochastic Processes and their Applications},
  volume = {126},
  number = {3},
  pages = {930--947},
  year = {2016},
  issn = {0304-4149},
  doi = {10.1016/j.spa.2015.10.005},
  url = {https://www.sciencedirect.com/science/article/pii/S0304414915002537}
}

@article{fong23martingale,
  title = {Martingale Posterior Distributions},
  author = {Fong, Edwin and Holmes, Chris and Walker, Stephen G.},
  year = 2023,
  journal = {Journal of the Royal Statistical Society: Series B (Statistical Methodology)},
  volume = {85},
  number = {5},
  pages = {1357--1391}
}

@article{fong26asymptotics,
  title = {Asymptotics for a class of parametric martingale posteriors},
  author = {Fong, Edwin and Yiu, Andrew},
  year = 2026,
  journal = {Biometrika},
  volume = {113},
  number = {2},
  doi = {10.1093/biomet/asag007}
}

@article{fong25bayesian,
  title = {Bayesian Quantile Estimation and Regression with Martingale Posteriors},
  author = {Fong, Edwin and Yiu, Andrew},
  year = 2025,
  journal = {Journal of the Royal Statistical Society: Series B (Statistical Methodology)},
  doi = {10.1093/jrsssb/qkaf080},
  note = {Advance article}
}

@article{fortini18notion,
  title = {On a Notion of Partially Conditionally Identically Distributed Sequences},
  author = {Fortini, Sandra and Petrone, Sonia and Sporysheva, Polina},
  year = 2018,
  journal = {Stochastic Processes and their Applications},
  volume = {128},
  number = {3},
  pages = {819--846}
}

@article{fortini20quasibayes,
  title = {Quasi-{{Bayes}} Properties of a Procedure for Sequential Learning in Mixture Models},
  author = {Fortini, Sandra and Petrone, Sonia},
  year = 2020,
  journal = {Journal of the Royal Statistical Society: Series B (Statistical Methodology)},
  volume = {82},
  number = {4},
  pages = {1087--1114}
}

@article{fortini23predictionbased,
  title = {Prediction-Based Uncertainty Quantification for Exchangeable Sequences},
  author = {Fortini, Sandra and Petrone, Sonia},
  year = 2023,
  journal = {Philosophical Transactions of the Royal Society A: Mathematical, Physical and Engineering Sciences},
  volume = {381},
  number = {2247},
  pages = {20220142},
  doi = {10.1098/rsta.2022.0142}
}

@article{fortini25exchangeability,
  title = {Exchangeability, Prediction and Predictive Modeling in {{Bayesian}} Statistics},
  author = {Fortini, Sandra and Petrone, Sonia},
  year = 2025,
  journal = {Statistical Science},
  volume = {40},
  number = {1},
  pages = {40--67},
  publisher = {Institute of Mathematical Statistics}
}

@inproceedings{genewein23memorybased,
  title = {Memory-Based Meta-Learning on Non-Stationary Distributions},
  booktitle = {Proceedings of the {{International Conference}} on {{Machine Learning}}},
  author = {Genewein, Tim and Del{\'e}tang, Gr{\'e}goire and Ruoss, Anian and Wenliang, Li Kevin and Catt, Elliot and Dutordoir, Vincent and {Grau-Moya}, Jordi and Orseau, Laurent and Hutter, Marcus and Veness, Joel},
  year = 2023,
  pages = {11173--11195}
}

@inproceedings{grau-moya24learning,
  title = {Learning Universal Predictors},
  booktitle = {Proceedings of the {{International Conference}} on {{Machine Learning}}},
  author = {{Grau-Moya}, Jordi and Genewein, Tim and Hutter, Marcus and Orseau, Laurent and Del{\'e}tang, Gr{\'e}goire and Catt, Elliot and Ruoss, Anian and Wenliang, Li Kevin and Mattern, Christopher and Aitchison, Matthew and Veness, Joel},
  year = 2024,
  pages = {16178--16205}
}

@book{hamada08bayesian,
  title = {Bayesian {{Reliability}}},
  author = {Hamada, Michael S. and Wilson, Alyson G. and Reese, C. Shane and Martz, Harry F.},
  year = 2008,
  series = {Springer {{Series}} in {{Statistics}}},
  publisher = {Springer}
}

@article{hollmann25accurate,
  title = {Accurate Predictions on Small Data with a Tabular Foundation Model},
  author = {Hollmann, Noah and M{\"u}ller, Samuel and Purucker, Lennart and Krishnakumar, Arjun and K{\"o}rfer, Max and Hoo, Shi Bin and Schirrmeister, Robin Tibor and Hutter, Frank},
  year = 2025,
  journal = {Nature},
  volume = {637},
  number = {8045},
  pages = {319--326},
  publisher = {Nature Publishing Group}
}

@inproceedings{hollmann23tabpfn,
  title = {{{TabPFN}}: A Transformer That Solves Small Tabular Classification Problems in a Second},
  author = {Hollmann, Noah and M{\"u}ller, Samuel and Eggensperger, Katharina and Hutter, Frank},
  year = 2023,
  booktitle = {International Conference on Learning Representations (ICLR)},
  url = {https://openreview.net/forum?id=cp5PvcI6w8_}
}

@inproceedings{jayasekera25variational,
  title = {Variational Uncertainty Decomposition for In-Context Learning},
  booktitle = {Advances in {{Neural Information Processing Systems}}},
  author = {Jayasekera, I. Shavindra and Si, Jacob and Valdettaro, Filippo and Chen, Wenlong and Faisal, Aldo A. and Li, Yingzhen},
  year = 2025
}

@book{kallenberg21foundations,
  title = {Foundations of {{Modern Probability}}},
  author = {Kallenberg, Olav},
  year = 2021,
  series = {Probability {{Theory}} and {{Stochastic Modelling}}},
  volume = {99},
  publisher = {Springer International Publishing}
}

@inproceedings{mikulik20metatrained,
  title = {Meta-Trained Agents Implement {{Bayes-optimal}} Agents},
  booktitle = {Advances in {{Neural Information Processing Systems}}},
  author = {Mikulik, Vladimir and Del{\'e}tang, Gr{\'e}goire and McGrath, Tom and Genewein, Tim and Martic, Miljan and Legg, Shane and Ortega, Pedro A.},
  year = 2020,
  volume = {33},
  pages = {18691--18703}
}

@unpublished{nagler25uncertainty,
  title = {Uncertainty Quantification for Prior-Data Fitted Networks Using Martingale Posteriors},
  author = {Nagler, Thomas and R{\"u}gamer, David},
  year = 2025,
  number = {arXiv:2505.11325},
  eprint = {2505.11325},
  primaryclass = {stat},
  publisher = {arXiv},
  archiveprefix = {arXiv},
  note = {arXiv:2505.11325}
}

@unpublished{ng25tabmgp,
  title = {{{TabMGP}}: {{Martingale}} Posterior with {{TabPFN}}},
  shorttitle = {{{TabMGP}}},
  author = {Ng, Kenyon and Fong, Edwin and Frazier, David T. and Knoblauch, Jeremias and Wei, Susan},
  year = 2025,
  number = {arXiv:2510.25154},
  eprint = {2510.25154},
  primaryclass = {stat},
  publisher = {arXiv},
  archiveprefix = {arXiv},
  note = {arXiv:2510.25154}
}

@article{olea19simultaneous,
  title = {Simultaneous Confidence Bands: {{Theory}}, Implementation, and an Application to {{SVARs}}},
  shorttitle = {Simultaneous Confidence Bands},
  author = {Olea, Jos{\'e} Luis Montiel and Plagborg-M{\o}ller, Mikkel},
  year = 2019,
  journal = {Journal of Applied Econometrics},
  volume = {34},
  number = {1},
  pages = {1--17},
  publisher = {John Wiley \& Sons, Ltd.}
}

@unpublished{ortega19metalearning,
  title = {Meta-learning of sequential strategies},
  author = {Ortega, Pedro A. and Wang, Jane X. and Rowland, Mark and Genewein, Tim and {Kurth-Nelson}, Zeb and Pascanu, Razvan and Heess, Nicolas and Veness, Joel and Pritzel, Alex and Sprechmann, Pablo and Jayakumar, Siddhant M. and McGrath, Tom and Miller, Kevin and Azar, Mohammad and Osband, Ian and Rabinowitz, Neil and Gy{\"o}rgy, Andr{\'a}s and Chiappa, Silvia and Osindero, Simon and Teh, Yee Whye and van Hasselt, Hado and de Freitas, Nando and Botvinick, Matthew and Legg, Shane},
  year = 2019,
  number = {arXiv:1905.03030},
  eprint = {1905.03030},
  primaryclass = {cs},
  publisher = {arXiv},
  archiveprefix = {arXiv},
  note = {arXiv:1905.03030}
}

@misc{priorlabs25tabpfn,
  title = {{{TabPFN}} Documentation: {{Intended}} Use},
  author = {{PriorLabs}},
  year = 2025,
  howpublished = {https://priorlabs.ai/getting\_started/intended\_use/\#computational-and-time-requirements}
}

@inproceedings{muller22pfns,
  title = {Transformers Can Do {{Bayesian}} Inference},
  booktitle = {International Conference on Learning Representations (ICLR)},
  author = {M{\"u}ller, Samuel and Hollmann, Noah and {Arango}, Sebastian Pineda and Grabocka, Josif and Hutter, Frank},
  year = 2022
}

@inproceedings{garg22what,
  title = {What Can Transformers Learn In-Context? {{A}} Case Study of Simple Function Classes},
  booktitle = {Advances in Neural Information Processing Systems},
  author = {Garg, Shivam and Tsipras, Dimitris and Liang, Percy and Valiant, Gregory},
  year = 2022
}

@inproceedings{raventos23pretraining,
  title = {Pretraining Task Diversity and the Emergence of Non-{{Bayesian}} In-Context Learning for Regression},
  booktitle = {Advances in Neural Information Processing Systems},
  author = {Ravent{\'o}s, Allan and Paul, Mansheej and Chen, Feng and Ganguli, Surya},
  year = 2023
}

@inproceedings{wurgaft25strategies,
  title = {In-Context Learning Strategies Emerge Rationally},
  author = {Wurgaft, Daniel and Lubana, Ekdeep Singh and Park, Core Francisco and Tanaka, Hidenori and Reddy, Gautam and Goodman, Noah D.},
  year = 2025,
  booktitle = {Advances in Neural Information Processing Systems}
}

@inproceedings{park25competition,
  title = {Competition Dynamics Shape Algorithmic Phases of In-Context Learning},
  author = {Park, Core Francisco and Lubana, Ekdeep Singh and Pres, Itamar and Tanaka, Hidenori},
  year = 2025,
  booktitle = {International Conference on Learning Representations (ICLR)},
  note = {Spotlight},
  eprint = {2412.01003},
  archivePrefix = {arXiv}
}

@article{zhang24trained,
  title = {Trained Transformers Learn Linear Models In-Context},
  author = {Zhang, Ruiqi and Frei, Spencer and Bartlett, Peter L.},
  journal = {Journal of Machine Learning Research},
  year = 2024,
  volume = {25},
  number = {49},
  pages = {1--55}
}

@inproceedings{wu24pretraining,
  title = {How Many Pretraining Tasks Are Needed for In-Context Learning of Linear Regression?},
  author = {Wu, Jingfeng and Zou, Difan and Chen, Zixiang and Braverman, Vladimir and Gu, Quanquan and Bartlett, Peter L.},
  booktitle = {International Conference on Learning Representations (ICLR)},
  year = 2024,
  note = {Spotlight}
}

@inproceedings{muller23pfns4bo,
  title = {{{PFNs4BO}}: In-context learning for {{Bayesian}} optimization},
  booktitle = {Proceedings of the {{International Conference}} on {{Machine Learning}}},
  author = {M{\"u}ller, Samuel and Feurer, Matthias and Hollmann, Noah and Hutter, Frank},
  year = 2023
}

@inproceedings{yu26gitbo,
  title = {{{GIT-BO}}: High-dimensional {{Bayesian}} optimization with tabular foundation models},
  author = {Yu, Rosen Ting-Ying and Picard, Cyril and Ahmed, Faez},
  year = 2026,
  booktitle = {International Conference on Learning Representations (ICLR)}
}
